\documentclass[journal]{IEEEtran}
\IEEEoverridecommandlockouts
\usepackage{cite}
\usepackage{amsmath, amsthm, amssymb,amsfonts}
\usepackage{extpfeil}
\usepackage{bm}
\usepackage{algorithmic}
\usepackage[ruled,linesnumbered]{algorithm2e}
\usepackage{graphicx}
\usepackage{textcomp}
\usepackage{subfigure}
\usepackage{booktabs}
\usepackage{colortbl}
\usepackage{xcolor}
\usepackage{soul}
\usepackage{hyperref}

\sethlcolor{yellow!30}
\newcommand{\findwidth}{%
  \ifboolexpr{test {\ifnumcomp{\value{page}}{=}{1}} or test {\@twocolumnfalse}}%
    {\textwidth}%
    {\linewidth}%
}

\def\BibTeX{{\rm B\kern-.05em{\sc i\kern-.025em b}\kern-.08em
    T\kern-.1667em\lower.7ex\hbox{E}\kern-.125emX}}

\newtheorem{assumption}{Assumption}
\newtheorem{theorem}{Theorem}
\newtheorem{lemma}[theorem]{Lemma}
\newtheorem{proposition}[theorem]{Proposition}
\newtheorem{corollary}[theorem]{Corollary}
\newtheorem{definition}[theorem]{Definition}
\newtheorem{remark}[theorem]{Remark}
\DeclareMathOperator{\mR}{\mathbb{R}}
\DeclareMathOperator{\mE}{\mathbb{E}}
\DeclareMathOperator{\mP}{\mathbb{P}}
\DeclareMathOperator{\diag}{diag}
\DeclareMathOperator{\cov}{cov}

\begin{document}
\title{U-OBCA: Uncertainty-Aware Optimization-Based Collision Avoidance via Wasserstein Distributionally Robust Chance Constraints
\thanks{
The authors are with School of Automation and Intelligent Sensing, Shanghai Jiao Tong University and Institute of Medical Robotics, Shanghai Jiao Tong University, Shanghai 200240, China. And they are also with the Key Laboratory of System Control and Information Processing, Ministry of Education of China, and Shanghai Engineering Research Center of Intelligent Control and Management, Shanghai 200240, China. 

Han Zhang is the corresponding author (email: zhanghan\_tc@sjtu.edu.cn).
Zehao Wang and Yuxuan Tang have the equal contribution.

This work is supported by the National Natural Science Foundation of China (Grant 62573287), Natural Science Foundation of Shanghai under Grant 25ZR1401208 and the Science and Technology Commission of Shanghai Municipality (Grant 20DZ2220400).
}
}

\author{
\IEEEauthorblockN{Zehao Wang, Yuxuan Tang, Han Zhang, Jingchuan Wang and Weidong Chen}
}

\maketitle

\begin{abstract}
Uncertainties arising from localization error, trajectory prediction errors of the moving obstacles and environmental disturbances pose significant challenges to robot's safe navigation. Existing uncertainty-aware planners often approximate polygon-shaped robots and obstacles using simple geometric primitives such as circles or ellipses. Though computationally convenient, these approximations substantially shrink the feasible space, leading to overly conservative trajectories and even planning failure in narrow environments. In addition, many such methods rely on specific assumptions about noise distributions, which may not hold in practice and thus limit their performance guarantees.
To address these limitations, we extend the Optimization-Based Collision Avoidance (OBCA) framework to an uncertainty-aware formulation, termed \emph{U-OBCA}. The proposed method explicitly accounts for the collision risk between polygon-shaped robots and obstacles by formulating OBCA-based chance constraints, and hence avoiding geometric simplifications and reducing unnecessary conservatism. These probabilistic constraints are further tightened into deterministic nonlinear constraints under mild distributional assumptions, which can be solved efficiently by standard numerical optimization solvers.
The proposed approach is validated through theoretical analysis, numerical simulations and real-world experiments. The results demonstrate that U-OBCA significantly mitigates the conservatism in trajectory planning and achieves higher navigation efficiency compared to existing baseline methods, particularly in narrow and cluttered environments.
\end{abstract}

\def\abstractname{Note to Practitioners}
\begin{abstract}
Mobile robots that operates in confined spaces—such as parking garages, warehouses, hospital corridors, or assisted-living facilities—often face a trade-off between safety and efficiency. To stay safe under the existence of localization error, trajectory prediction errors of the moving obstacles and environmental disturbances, many existing planning systems intentionally keep robots far away from obstacles. While this strategy reduces the collision risk, it can also make robots overly cautious, slow, or even unable to maneuver in tight spaces.
The method presented in this paper is intended for practitioners who need to operate robots safely without sacrificing maneuverability in narrow environments. Instead of simplifying the robot and surrounding obstacles into overly conservative shapes, the proposed approach directly uses their actual geometric outlines while still accounting for uncertainty in perception, trajectory prediction and disturbances. This allows the robot to move closer to obstacles in a controlled and quantifiable manner, rather than relying on large safety margins chosen by trial and error.
From an implementation perspective, the method can be integrated into existing optimization-based planners and solved using standard numerical optimization tools. It does not require detailed tuning of noise distributions or specialized probabilistic models, making it suitable for systems whose perception noise, moving obstacles trajectory prediction error and environmental disturbances vary with environment, hardware, or operating conditions.
Although the proposed method incurs higher computational cost than planners based on simplified geometry, it remains feasible for trajectory planning on modern computing hardware. In practice, it is particularly well suited for low-speed navigation, precision maneuvers, and tasks performed in highly constrained environments, where safety and efficient use of limited space are more critical than achieving maximum speed.

\end{abstract}


\setlength{\arraycolsep}{2pt}

\section{Introduction}
Ensuring safety is one of the central challenges in mobile robot and autonomous vehicle navigation in real-world environments. 
A major source of difficulty arises from uncertainties induced by perception noise, prediction error and environmental disturbances, which inevitably degrade the reliability of collision avoidance. 
This challenge becomes particularly critical in narrow environments, such as parking garages or confined corridors, where the robot must operate in close proximity to obstacles. 
In such scenarios, even small uncertainties may significantly increase the collision risk or fail to plan any feasible trajectories (a.k.a freezing robot problem).

A variety of approaches have been proposed to address this issue, including dynamic window-based methods \cite{molinos2019dynamic}, robust collision avoidance schemes \cite{RCA}, and velocity obstacle formulations \cite{lopez2020obstacle}. 
More recently, uncertainty-aware trajectory planning methods have incorporated probabilistic tools—such as chance constraints \cite{blackmore2011chance, da2019collision, zhu2019chance, Castillo2020, nakka2022trajectory, theurkauf2023chance, jasour2023convex, xu2024distributionally, summers2018distributionally}, value-at-risk \cite{lew2023risk}, and conditional value-at-risk \cite{hakobyan2021wasserstein, hakobyan2019risk, yin2023risk, ryu2024integrating}—to provide safety guarantees in terms of collision probability. 
However, most of these planners \cite{zhu2019chance, Castillo2020, theurkauf2023chance, lew2023risk, nakka2022trajectory, yin2023risk, ryu2024integrating} approximate the robot and obstacles using circles or ellipses, even though many practical systems, such as vehicles, were more accurately modeled as polygons. 
While such approximations lead to low-dimensional nonlinear programs that can be efficiently solved using gradient-based methods, they inevitably introduce conservatism. 
As illustrated in Fig.~\ref{fig: The shortage of approximating the polygon-shaped vehicles and obstacles as circles or ellipses.}, this conservatism may result in overly cautious trajectories or even infeasibility in narrow environments. 
Therefore, explicitly accounting for polygonal shapes is essential for reliable navigation in constrained spaces.

To provide rigorous safety guarantees for polygon-shaped robots and obstacles, \cite{blackmore2011chance, da2019collision} model chance-constrained collision avoidance as a disjunction of linear inequalities. 
This formulation introduces binary variables, yielding a Mixed-Integer Nonlinear Programming (MINLP) problem whose computational complexity grows exponentially with the number of obstacles. 
Moreover, these methods only consider the uncertainty in obstacle position, while the uncertainty in obstacle orientation—which can significantly affect the occupied space—is neglected.

\begin{figure}[!htbp]
\centering
\subfigure[Polygon-shaped vehicle and obstacle are collision-free, while their circle or ellipse approximations violate collision constraints.]{
\includegraphics[width=0.5\hsize]{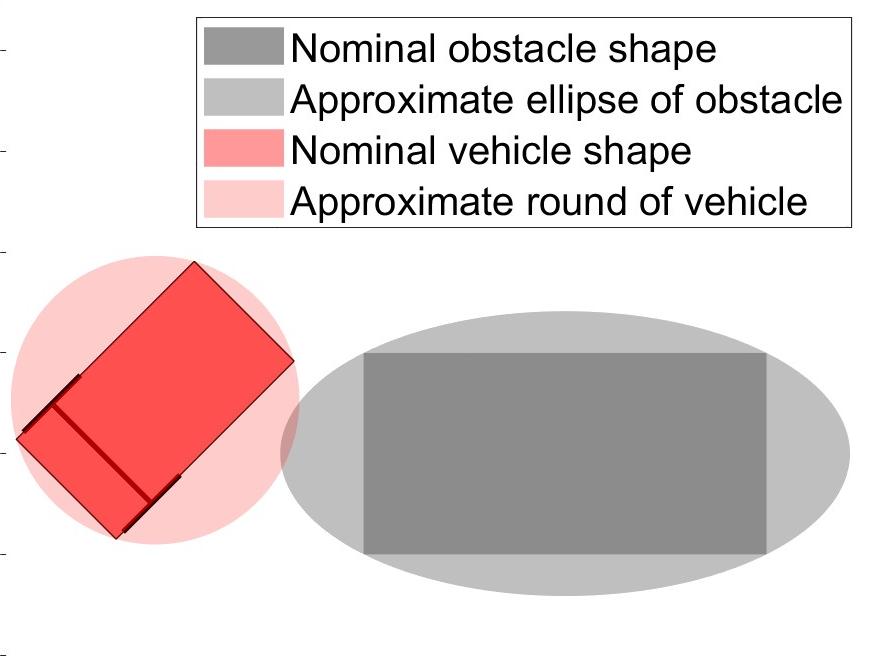}
\label{fig: The shortage of approximating the polygon-shaped vehicles and obstacles as circles or ellipses.}}
\subfigure[Obstacle shape with and without uncertainty in position and orientation.]{
\includegraphics[width=0.4\hsize]{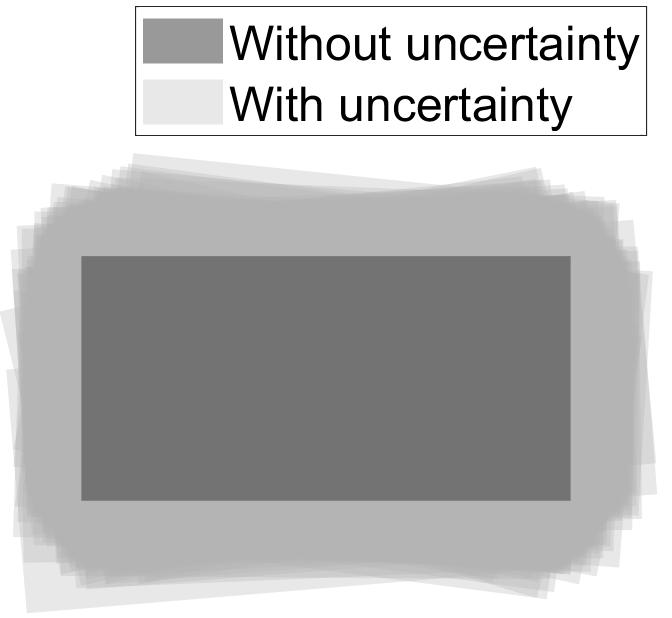}
\label{fig: The shape of the obstacle with/without the uncertainties in the position and the orientation angle.}}
\caption{Motivation: collision risk assessment between polygonal shapes under pose uncertainty.}
\label{fig: The motivation of this paper is to measure the collision risk between the polygons under the uncertainties.}
\end{figure}

On the other hand, several nonlinear constraints have been proposed for collision avoidance between polygons \cite{gerdts2012path, schulman2014motion, li2015unified, OBCA}. 
Although these methods guarantee safety in deterministic settings, they do not directly account for the uncertainty in obstacle pose. 
A common workaround is to enforce a minimum safety distance between the robot and obstacles; however, selecting such a distance is often heuristic, time-consuming, and leads to excessive conservatism. 
Furthermore, as shown in Fig.~\ref{fig: The shape of the obstacle with/without the uncertainties in the position and the orientation angle.}, the uncertainty in obstacle orientation induces highly irregular risk regions, making distance-based tuning particularly challenging. 
Consequently, the existing approaches fail to simultaneously achieve uncertainty-aware safety guarantees, preservation of feasible space, and computational efficiency for polygon-shaped robots and obstacles.

In this work, we propose an \emph{Uncertainty-aware Optimization-Based Collision Avoidance} (U-OBCA) framework. 
We introduce chance constraints to explicitly bound the collision probability between polygonal robots and obstacles. 
To enable efficient numerical optimization, these chance constraints are further tightened and reformulated as deterministic nonlinear constraints, which preserve feasible space while avoiding integer variables and maintaining low computational cost.
In particular, the main contribution of this paper is three-fold:
\begin{itemize}
    \item We generalize the Optimization-Based Collision Avoidance (OBCA) framework \cite{OBCA} to uncertainty-aware settings using chance constraints. 
    Unlike ellipse-based approximations, the proposed formulation directly handles polygonal shapes and significantly reduces conservatism, enabling safe navigation in narrow environments.
    \item We strengthen the chance constraints using Wasserstein distributionally robust formulations and further transform them into deterministic nonlinear constraints, making the consequent optimization problem suitable for being solved numerically.
    This reformulation does not require Gaussian assumptions on the uncertainty and relies only on mild assumptions on the distribution. 
    \item Through simulations and experiments, we demonstrate that the proposed method substantially reduces conservatism and improves the navigation efficiency compared with the state-of-the-art baselines.
\end{itemize}
\textit{Notations}: 
The $\ell_2$-norm on $\mR^n$ is denoted by $\|\cdot\|$. 
$\mathbf{I}_{n \times n} \in \mR^{n \times n}$ and $\mathbf{0}_{m \times n} \in \mR^{m \times n}$ denote the identity matrix and zero matrix, respectively. 
For integers $n_1$ and $n_2$, $n_1\!:\!n_2$ denotes the sequence $n_1, n_1+1, \ldots, n_2$. 
Bold italic symbols represent stochastic variables. 
$\mathcal{P}(\bm{\cdot})$ denotes probability, while $\mE[\bm{X}]$ and $\mathbb{P}_{\bm X}$ denote the expectation and distribution of the random variable $\bm{X}$, respectively.

\section{Problem Statement}
\label{sec: Problem Statement}
In this work, we consider the motion planning problem of a vehicle in a narrow environment with $N_r$ round-shaped and $N_p$ polygon-shaped obstacles under uncertainties caused by perception noise, prediction error and environmental disturbance. More specifically, we denote the index set of the round-shaped obstacles as $\mathcal{O}_r:=\{1:N_r\}$, and the index set of the polygon-shaped obstacles as $\mathcal{O}_p:=\{N_r+1:N_r+N_p\}$. Consequently, the index set of all obstacles is defined as $\mathcal{O}:=\{1:N_r+N_p\}$. The trajectory of the vehicle is discretized with
a sampling period $\Delta t = \frac{T}{N}$, where $T$ denotes the planning time horizon and $N$ denotes the sampling number. 

On the other hand, the nominal discrete-time kinematics of the vehicle takes the form
\begin{equation}
\label{eq: nominal kinematics}
\begin{aligned}
    &\bar{s}^v_{k+1} = \bar{s}^v_{k}+f(\bar{s}^v_{k}, u_k) \cdot 
 \Delta t, \quad k = 1: N, \\
\end{aligned}
\end{equation}
where $f(\cdot)$ is the nominal kinematic of the vehicle, $\bar{s}^v_{k} \in \mR^{n_v}$ is the vehicle's nominal state at the $k$-th time step, and $n_{v}$ is the dimension of the vehicle's state. 
Moreover, the vehicle's nominal state at time step $k$ must include the vehicle's position $(\bar{x}^v_{k},\bar{y}^v_{k})$ and orientation $\bar{\theta}^v_{k}$ in the \textbf{world frame}. 
Other state components, such as velocities, can be specified according to users' practical demands. 
$u_k \in \mR^{n_u}$ is the vehicle's control input at the $k$-th time step, where $n_u$ is the dimension of the control input. Similarly, we denote the nominal poses of the $j$-th obstacle at the $k$-th time step as the $[\bar{x}^{o_j}_{k},\bar{y}^{o_j}_{k}, \bar{\theta}^{o_j}_{k}]^T$, which includes the position and orientation information also in the \textbf{world frame}.

Based on the nominal poses, we further formulate the states under the uncertainties caused by perception and disturbance. In particular, 
let $\bm{w}^v_k, \bm{w}^{o_j}_k\in\mR^{3}$ be the stochastic disturbances and perception noise that affect the vehicle's pose at time step $k$.  
Consequently, we model the pose of the vehicle at the $k$-th time step as 
\begin{align*}
\begin{bmatrix}
\bm{x}^v_k \\
\bm{y}^v_k \\
\bm{\theta}^v_k
\end{bmatrix}\!=\!
\begin{bmatrix}
    \bar{x}^v_k\\
    \bar{y}^v_k\\
    \bar{\theta}^v_k
\end{bmatrix}
\!+\!
\underbrace{\begin{bmatrix}
    \bm w_{x,k}^v\\
    \bm w_{y,k}^v\\
    \bm w_{\theta,k}^v
\end{bmatrix}}_{\bm w_k^v},
\begin{bmatrix}
    \bm{x}^{o_j}_k\\
    \bm{y}^{o_j}_k\\
    \bm{\theta}^{o_j}_k
\end{bmatrix}\!=\!
\begin{bmatrix}
    \bar{x}_k^{o_j}\\
    \bar{y}_k^{o_j}\\
    \bar{\theta}_k^{o_j}
\end{bmatrix}\!+\!
\underbrace{\begin{bmatrix}
    \bm w_{x,k}^{o_j}\\
    \bm w_{y,k}^{o_j}\\
    \bm w_{\theta,k}^{o_j}
\end{bmatrix}}_{\bm w_k^{o_j}},
k\!=\!0:N.
\end{align*}
Notably, our goal is to plan the ``collision-free" future trajectory $\{\bar{s}_k^v\}_{k=1}^{N}$ (without losing the generality, let the current time step be $k=0$) based on the perception at the current time step $k=0$. Consequently, $\bm{w}_0^v$ and $\{\bm{w}_k^{o_j}\}_{k=0}^{N}$ in the above formulation shall be interpreted as the observation noise in pose perception and prediction, while $\{\bm{w}_k^v\}_{k=1}^{N}$ shall be interpreted as the environment disturbances.
Moreover, we make the following assumptions regarding the noises.
\begin{assumption}\label{ass: zero_mean_and_known_covariance}
    It is assumed that the expectations of the noises $\bm{w}^v_k$ and $\bm{w}^{o_j}_k$ are zeros and the covariance matrices of the noises are known a priori. Moreover, the noises in the positions and the orientations are assumed to be independent, and the noises of the vehicle $\bm{w}^v_k$ and the obstacles $\bm{w}^{o_j}_k$ are independent as well.
\end{assumption}
\begin{assumption}\label{ass: known_sin_cos}
    Let $\bm{w}^v_{\theta, k}$ and $\bm{w}^{o_j}_{\theta, k}$ denote the noises component that acts on the vehicle's and $j$-th obstacle's orientation at $k$-th time step, respectively. Their distributions are symmetric about zero. It is further assumed that we can off-line estimate or compute the expectations and covariances of the $\sin$ and $\cos$ functions of the noises, i.e.,
    $\big\{\sin(\bm{w}_{\theta, k}^v), \cos(\bm{w}_{\theta, k}^v), \sin(\bm{w}_{\theta, k}^{o_j}), \cos(\bm{w}_{\theta, k}^{o_j})\big\}$.
\end{assumption}
As we shall see later, based on the expectation and the covariance of these noises functions, we can directly compute the expectation and the covariance of the $\cos\bm{\theta}^v_k$, $\sin\bm{\theta}^v_k$, $ \cos\bm{\theta}^{o_j}_k$ and $\sin\bm{\theta}^{o_j}_k$.
\begin{assumption}\label{ass: bounded_wasserstein_distance}
    Given a specific measurable function $\bm{h}$, we assume the Wasserstein distance between the distribution $\mathbb{P}_{\bm h}$ and the Gaussian distribution $\mathbb{P}^R_{\bm h} = \mathcal{N}(\mE[\bm h], \Sigma_{\bm h})$ is smaller than a known constant bound $\theta_w$.
\end{assumption}
In particular, Wasserstein distance measures the ``distance" between two distributions. The definition of the Wasserstein distance will be explained in detail in Sec.~\ref{sec: Wasserstein Distributional Ambiguity}. 
\begin{assumption}\label{ass: deterministic_sizes_and_shape}
    The shapes and the sizes of the vehicle and the obstacles are deterministic and the uncertainties only exist in the poses of the vehicle and the obstacles.
\end{assumption}
Notably, the above assumptions are mild since: i) we do not require the noises to have any specific distributions (e.g. Gaussian distribution); ii) we can estimate the moments of these noises' functions through off-line identification experiments; and
iii) the assumption on the ``distance" between $\mathbb{P}_{\bm h}$ $\mathbb{P}^R_{\bm h}$ always holds as long as we choose a large enough $\theta_w$.

Next, we formulate the ``space" that is occupied by the vehicle and the obstacles. Suppose there are $N_r$ round-shaped obstacles and $N_p$ polygon-shaped obstacles. Let $\bm{t}^v_k = \begin{bmatrix} \bm{x}^v_k & \bm{y}^v_k\end{bmatrix}^T$ and $\bm{t}^{o_j}_k = \begin{bmatrix} \bm{x}^{o_j}_k & \bm{y}^{o_j}_k\end{bmatrix}^T$ denote
the positions of the vehicle and the $j$-th obstacle in the world coordinate frame at the $k$-th time step, respectively.
Then, denote $\bar{\mathbb{V}}$ and $\bar{\mathbb{O}}_j, j \in \mathcal{O}$ as the ``space'' that is occupied by the vehicle and $j$-th obstacle in their \textbf{ego frame}.
Furthermore, let $\mathbb{V}$ and $\mathbb{O}_j$ denote the ``space'' that is occupied by the vehicle and the $j$-th obstacle in the \textbf{world frame}. Our aim is to derive the explicit expression for $\mathbb{V}$ and $\mathbb{O}_j$.

To this end, for the vehicle and polygon-shaped obstacles, the occupied space in their ego frame take the form $\bar{\mathbb{V}}:= \{\bar{p} \in \mR^2: \bar{A}^v \bar{p} \preceq \bar{b}^v \}$ and $\bar{\mathbb{O}}_j:= \{\bar{p} \in \mR^2: \bar{A}^{o_j} \bar{p} \preceq \bar{b}^{o_j}\}, j \in \mathcal{O}_p$, where $\bar{A}^v \in \mR^{l_v\times 2}$, $\bar{b}^v \in \mR^{l_v}$, $\bar{A}^{o_j} \in \mR^{l_{o_j}\times 2}$, $\bar{b}^{o_j} \in \mR^{l_{o_j}}$ and $l_v,l_{o_j}$ are the number of the vehicle's and the $j$-th obstacle's edges.
Similarly, for round-shaped obstacles, the occupied space in their ego frame take the form $\bar{\mathbb{O}}_j:=\{\bar{p} \in \mR^2:  \|\bar{p}\| \leq r_{o_j}\}, j \in \mathcal{O}_r$, where $r_{o_j}$ is the radius of the $j$-th obstacle.

Consequently, for the round-shaped obstacles, the occupied space in world frame at the $k$-th time step shall take the form $\mathbb{O}^j_k:= \{p \in \mR^2:  \|p-\bm{t}^{o_j}_k\| \leq r_{o_j}\}, j \in \mathcal{O}_r$.
Moreover, for the vehicle and the polygon-shaped obstacles, it follows from coordinate transformations that their occupied space at the $k$-th time step in the world frame shall take the form $\mathbb{V}_k:=\{p \in \mR^2: \bm{A}^v_k p \preceq \bm{b}^v_k\}$ and $\mathbb{O}^j_k:= \{p \in \mR^2: \bm{A}^{o_j}_k p \preceq \bm{b}^{o_j}_k\}, j \in \mathcal{O}$, where
\begin{equation}
\label{eq: ego world frame transfromation}
\begin{aligned}
    \bm{A}^v_k \;= \;&\bar{A}^v \; R^{-1}(\bm{\theta}^v_k \;\; ),\\
    \bm{A}^{o_j}_k =\; &\bar{A}^{o_j}R^{-1}(\bm{\theta}^{o_j}_k),\\
    \bm{b}^v_k \;\; =\; &\bar{A}^v \; R^{-1}(\bm{\theta}^v_k \;\; )\bm{t}^v_k \: + \bar{b}^v,\\
    \bm{b}^{o_j}_k \; =\; &\bar{A}^{o_j}R^{-1}(\bm{\theta}^{o_j}_k)\bm{t}^{o_j}_k + \bar{b}^{o_j},\quad k = 1: N,
\end{aligned}
\end{equation}
and $R(\bm{\theta})$ is the rotation matrix that corresponds to rotation angle $\bm{\theta}$, namely, $R(\bm{\theta}) = \begin{bmatrix}
\cos\bm{\theta} & -\sin\bm{\theta}\\   
\sin\bm{\theta} & \;\;\;\cos\bm{\theta}\\  
\end{bmatrix}$.

Next, to guarantee the vehicle and the obstacles are collision-free, we consider the distance between the vehicle $\mathbb{V}_k$ and the $j$-th obstacle $\mathbb{O}^j_k$ at the $k$-th time step. As illustrated in Fig.~\ref{fig: The distance between the vehicle and the obstacle.}, the distance between the vehicle and the $j$-th obstacle at the $k$-th time step takes the form
\begin{equation}
\label{eq: distance between two polygons}
\operatorname{dist}(\mathbb{V}_k,\mathbb{O}^j_k) \!:=\!\min_t\{\|t\|\colon(\mathbb{V}_k+t)\cap\mathbb{O}^j_k\!\neq\!\emptyset\}.
\end{equation}
\begin{figure}[!htbp]
\centering
\subfigure[When the obstacle is round-shaped.]{\includegraphics[width=0.4\hsize]{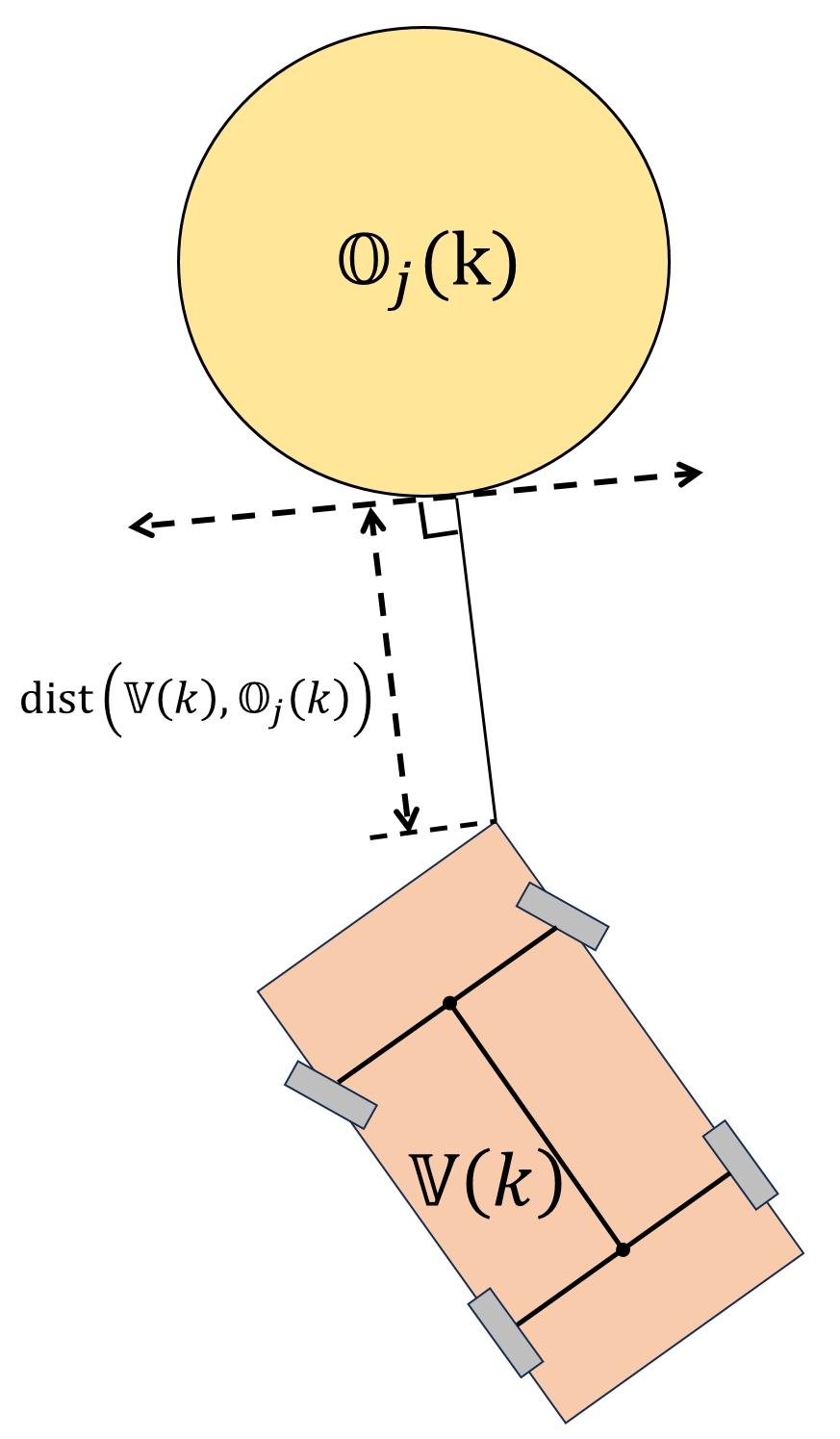}
\label{fig: When the obstacle is round-shaped.}}
\subfigure[When the obstacle is polygon-shaped.]
{\includegraphics[width=0.37\hsize]{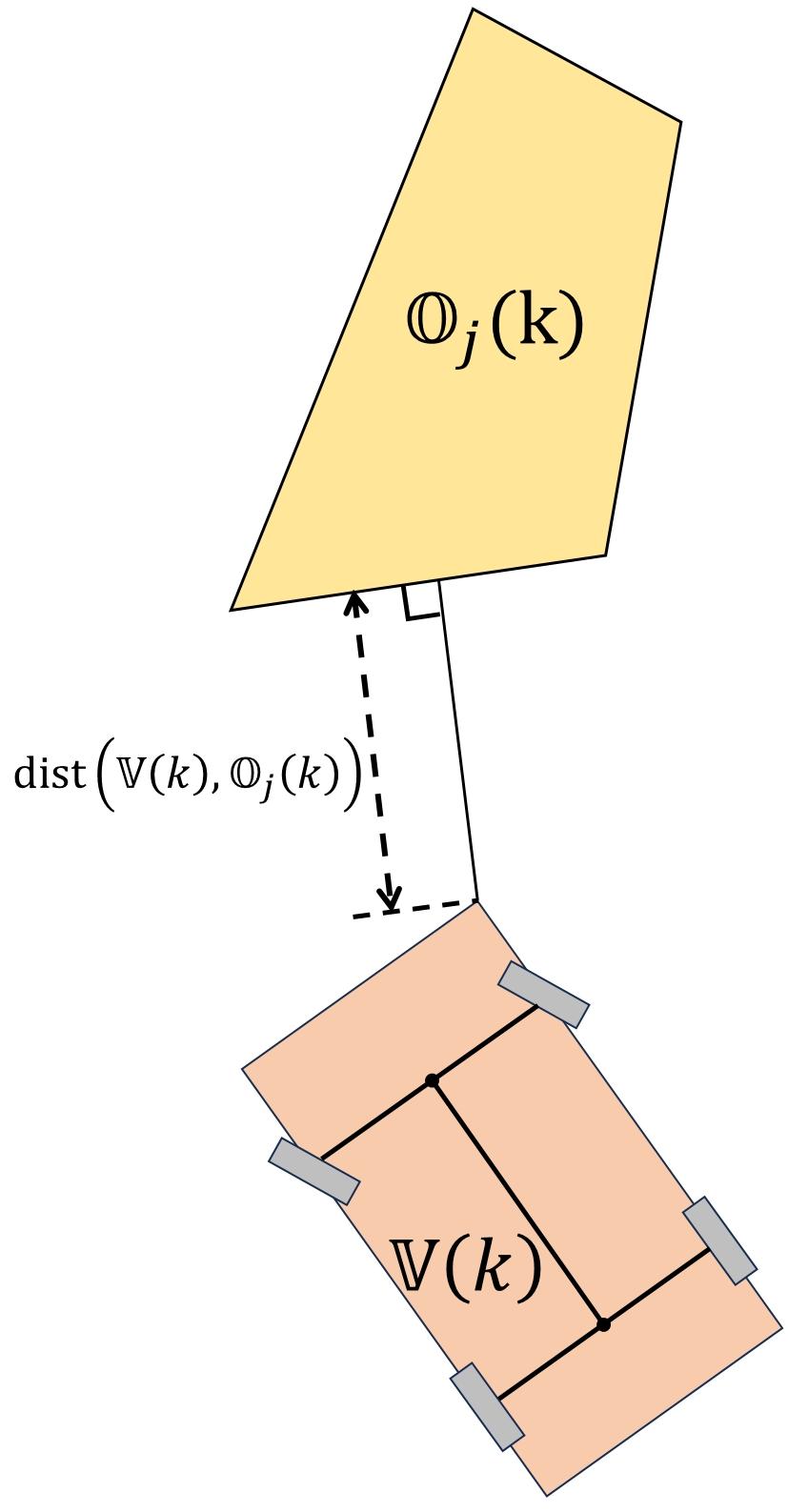}
\label{fig: When the obstacle is polygon-shaped.}}
\caption{The distance between the vehicle and the obstacle.}
\label{fig: The distance between the vehicle and the obstacle.}
\end{figure}
Therefore, we see generating a collision-free trajectory for the vehicle as always ensuring the distance \eqref{eq: distance between two polygons} greater than a minimum safety distance ${\rm d}_{\min}$, namely, 
$\operatorname{dist}(\mathbb{V}_k,\mathbb{O}^j_k)\geq{\rm d}_{\min}\ge \rm{d}_{\min}$, $k=1,\ldots,N-1$.
Nevertheless, under the uncertainties caused by $\bm{w}^v_k$ and $\bm{w}^{o_j}_k$, the occupied space $\mathbb{V}_k$ and $\mathbb{O}^j_k$ become also stochastic. To guarantee the safety under the uncertainties, we use the chance collision avoidance constraints to enforce the vehicle's trajectory collision-free up to a confidence level of $1-\alpha$, i.e.,
\begin{equation}
\label{eq: probabilistic collision avoidance constraint}
\begin{aligned}
&\mathcal{P}(\operatorname{dist}(\mathbb{V}_k,\mathbb{O}^j_k) \geq{\rm d}_{\min}) \geq 1 - \alpha, \\
&\forall k = 1: N, \quad \forall j \in \mathcal{O}. 
\end{aligned}
\end{equation}
And consequently, the overall trajectory planning problem shall take the form in probabilistic optimization as follows
\begin{subequations}
\label{eq: Overall Formulation}
\begin{align}
\min_{\substack{\bar{s}, u}}\;\;\;\; & J(\bar{s}, u),\\ 
\mbox{s.t.\;\;\;\;}& {\rm Nominal \ kinematic \ constraints} \ \eqref{eq: nominal kinematics}, \\
&{\rm \textbf{Collision avoidance constraints}}:\nonumber\\
&\mathcal{P}(\operatorname{dist}(\mathbb{V}_k,\mathbb{O}^j_k) \geq{\rm d}_{\min}) \geq 1 - \alpha, \nonumber\\
&\forall k = 1, \dots, N, \quad \forall j \in \mathcal{O}, \label{eq: chance_constraints}\\
&{\rm State \ and \ control \ constraints},
\end{align}
\end{subequations}
where $\bar{s}, u$ are short-hand notations for $\{\bar{s}^v_k\}_{k=0}^N$ and $\{u_k\}_{k=0}^{N-1}$. And $J(\bar{s}, u)$ is the user defined cost function. The state and control constraints include, but are not limited to, the boundary constraints, state and control bounds, etc. Note that we do not give the explicit expressions of the nominal kinematic constraints as well as the state and control constraints because they are not the focus of this work. In other words, our framework can be easily extended to a variety of robot systems by replacing the kinematics, the state and control constraints in \eqref{eq: Overall Formulation} with any user-defined form.

Notably, the key challenge in solving the optimization problem \eqref{eq: Overall Formulation} lies in evaluating the probability in the chance constraints \eqref{eq: probabilistic collision avoidance constraint}. This is mainly due to the fact that
\begin{itemize}
    \item The expression of the distance \eqref{eq: distance between two polygons} is intractable and;
    \item There is no explicit formulae to compute the chance constraints.
\end{itemize}
To address these difficulties, we need to reformulate \eqref{eq: probabilistic collision avoidance constraint} through dual variables and then approximate the chance constraints as nonlinear deterministic constraints, which are suitable for efficient computation. This will be explained in detail in the next section.

\section{The Preliminary Results}
In this section, we introduce the preliminary results to support further developments in Section~\ref{sec: Uncertainty-aware Optimization-Based Collision Avoidance}. More specifically, we first introduce the OBCA method \cite{OBCA} to reformulate the distance between the vehicle and the obstacles through the dual variables. Then, in order to deal with the uncertainties with the arbitrary distributions, we utilize the worst-case chance constraints with Wasserstein distributional ambiguity. Therefore, for the sake of self-containment, we state some preliminary theorems regarding the chance constraints and the Wasserstein distributional ambiguity set.
\subsection{Optimization-Based Collision Avoidance Without Uncertainties}
To proceed, we need to introduce the Optimization-Based Collision Avoidance (OBCA) \cite{OBCA} to reformulate \eqref{eq: distance between two polygons} into a tractable constraint. More precisely, when there are no uncertainties, the following theorem holds for the space occupied by the vehicle $\mathbb{V}_k$ and the $j$-th obstacle $\mathbb{O}^j_k$ at the $k$-th time step.
\begin{theorem}
\textit(OBCA \cite[Prop.~1, Prop.~3]{OBCA}):
On one hand, for the $j$-th round-shaped obstacle, the minimum distance constraint $\operatorname{dist}(\mathbb{V}_k,\mathbb{O}^j_k)\geq{\rm d}_{\min}$ is equivalent to
\begin{equation}
\label{eq: distance between polygon and point reformulation}
\begin{aligned}
&\exists\mu^j_k\succeq 0: \|A^{v,T}_k\mu^j_k\|\leq 1, \\
&(A^v_kt^{o_j}_k-b^v_k)^T\mu^j_k\geq{\rm d}_{\min} + r_{o_j},\forall k = 1:N, \quad \forall j \in \mathcal{O}_r,
\end{aligned}
\end{equation}
where $\mu^j_k \in \mR^{l_v}$ are dual variables. On the other hand, for the $j$-th polygon-shaped obstacle, the minimum distance constraints $\operatorname{dist}(\mathbb{V},\mathbb{O}_j)\geq{\rm d}_{\min}$ is equivalent to
\begin{equation}
\label{eq: distance between two polygons reformulation}
\begin{aligned}
&\exists\lambda^j_k,\mu^j_k\succeq 0: \|A^{o_j,T}_k\lambda^j_k\|\leq 1, \\
&-b^{v,T}_k\mu^j_k -b^{o_j, T}_k\lambda^j_k\geq{\rm d}_{\min},\\
&A^{v,T}_k\mu^j_k+A^{o_j,T}_k\lambda^j_k=\bm{0}_{2\times1},\forall k = 1:N, \quad \forall j \in \mathcal{O}_p,
\end{aligned}
\end{equation}
where $\lambda^j_k \in \mR^{l_{o_j}}, \mu^j_k \in \mR^{l_v}$ are dual variables.
\end{theorem}

\subsection{Bounds on Disjunctive Chance Constraints} 
Note that the chance constraints \eqref{eq: chance_constraints} has a conjunctive form. To further deal with such constraints later, we need the following theorem.
\begin{theorem}
\label{thm: Bounds on Disjunctive Chance Constraints}
(\textit{Bounds on Conjunctive Probability} \cite[Lem.~1]{ono2013}): 
For given events $E_i, i \in \{1:N_E\}$ and $\alpha \in[0,1]$, we have 
\begin{equation}
\label{eq: multiple prob}
\begin{aligned}
&\mathcal{P}\Big(\bigwedge_{i=1}^{N_E} E_i\Big) \geq 1 - \alpha \Leftarrow (\bigwedge_{i=1}^{N_E}\mathcal{P}(E_i) \geq 1 - \alpha_i) \\
&\wedge (0\leq \alpha_i \leq 1) \wedge (\sum_{i=1}^{N_E} \alpha_i \leq \alpha).
\end{aligned}
\end{equation}
\end{theorem}

\subsection{Worst-Case Chance Constraints with the Wasserstein Distributional Ambiguity}
\label{sec: Wasserstein Distributional Ambiguity}
Recall in Assumption~\ref{ass: zero_mean_and_known_covariance},~\ref{ass: known_sin_cos},~\ref{ass: bounded_wasserstein_distance} and~\ref{ass: deterministic_sizes_and_shape}, we do not assume any specific noises distributions. To handle these noises $\{\bm{w}_k^v\}$ and $\{\bm{w}_k^{o_j}\}$, we use the worst-case chance constraints with the Wasserstein distributional ambiguity. 
To this end, we first define the Wasserstein distance and the corresponding Wasserstein distributional ambiguity set, and then introduce the worst-case chance constraints based on the ambiguity set.
\subsubsection{Wasserstein distance and Wasserstein ambiguity set}
Intuitively speaking, the Wasserstein ambiguity set can be viewed as a ball in the space of distributions with respect to the Wasserstein distance. The exact definition of the Wasserstein ambiguity set is as follows.
\begin{definition}
\label{def: Wasserstein Distance}
(\textit{Wasserstein Distance}):
Equipped with a general norm $\|\cdot\|$ on $\mR^n$, the Wasserstein distance $d_w(\mathbb{P}_1,\mathbb{P}_2)$ between two distributions $\mathbb{P}_1$ and $\mathbb{P}_2$ on $\mR^n$ is defined as the minimal transportation cost of moving $\mathbb{P}_1$ to $\mathbb{P}_2$, under the premise that the cost of moving a unit mass from $\bm{m}_1$ to $\bm{m}_2$ with amounts $\|\bm{m}_1-\bm{m}_2\|$. Namely,
\begin{equation}
\label{eq: Wasserstein distance}
d_w(\mathbb{P}_1,\mathbb{P}_2) := \inf_{\mathbb{P} \in \mathbb{Q}(\mathbb{P}_1,\mathbb{P}_2)} \mathbb{E}_\mathbb{P}[\| \tilde{\bm{m}}_1 - \tilde{\bm{m}}_2\|],
\end{equation}
where $\tilde{\bm{m}}_1$ (respectively, $\Tilde{\bm{m}}_2$) follows the distribution $\mathbb{P}_1$ (respectively, $\mathbb{P}_2$) and $\mathbb{Q}(\mathbb{P}_1,\mathbb{P}_2)$ is
the set of all joint distributions on $\mR^n \times \mR^n$ with marginals $\mathbb{P}_1$ and $\mathbb{P}_2$. 
\end{definition}
\begin{definition}
(\textit{Wasserstein Ambiguity Set} \cite{blanchet2019quantifying}):
The Wasserstein ambiguity set $\mathcal{F}_w(\mathbb{P}_R,\theta_w)$ is defined as a ball of radius $\theta_w \geq 0$ with respect to the Wasserstein distance, around a prescribed reference distribution $\mathbb{P}_R$. Namely,
\begin{equation}
\label{eq: Wasserstein set}
\mathcal{F}_w(\mathbb{P}_R,\theta_w) := \{\mathbb{P}: d_w(\mathbb{P},\mathbb{P}_R) \leq \theta_w\},
\end{equation}  
where $\theta_w$ is called as Wasserstein radius.
\end{definition}

\subsubsection{Elliptical distributions}
We denote an elliptical distribution \cite{fang2018symmetric} as $\mathbb{P}_E(\mu, \Sigma, g)$ with a location parameter $\mu\in\mR^n$, a positive definite
matrix $\Sigma\in\mR^{n\times n}$, and a measurable and such that the normalization constant is finite function $g:\mR\mapsto\mR$, where its probability density function $f$ takes the form
\begin{equation*}
    f(\bm x) = k \cdot g(\frac{1}{2}(\bm{x}-\mu)^T\Sigma^{-1}(\bm{x}-\mu))
\end{equation*}
with a positive normalization scalar $k$.
Moreover, if the standard elliptical distribution $\mathbb{P}_0(0, 1, g)$ is generated by $g$, then its
probability density function of the standard elliptical distribution is $kg(z^2/2)$ and the corresponding cumulative distribution function is
\begin{equation}
\label{eq: cumulative distribution function}
    \Phi(a) = \int_{-\infty}^a kg(\frac{z^2}{2}) dz,
\end{equation}
where its inverse is denoted by $\Phi^{-1}(\cdot)$. 
Moreover, Gaussian distribution $\mathcal{N}(\mu, \Sigma, e^{-z})$ is a special case of the elliptical distribution.
In addition, we define the value-at-risk at confidence level $1-\alpha$ as
\begin{equation*}
    \mathbb{P}_{\bm x}\mbox{-}\operatorname{VaR}_{1-\alpha}(\bm{x}) = \inf_{a\in\mR}\{a|\mathcal{P}(\bm{x} \geq a) \leq \alpha\}.
\end{equation*}
Equipped with the above knowledge, the following theorem holds:
\begin{theorem}
\label{thm: WASSERSTEIN distributionally robust theorem}
(\textit{Dual Variables Reformulation}\cite[Thm.~4.8]{WASSERSTEIN}): 
Suppose the reference distribution $\mathbb{P}_R(\mu, \Sigma, g)$ is an $n$ dimensional elliptical distribution, where $\Sigma$ postive definite. Moreover, let the Wasserstein distance is defined through the Mahalanobis norm associated with $\Sigma$. Then for the given Wasserstein radius $\theta_w>0$ and the risk level $\alpha \in (0,0.5]$, the worst-case value-at-risk is given by
\begin{equation*}
\sup_{\mathbb{P}_{\bm{x}} \in \mathcal{F}_w(\mathbb{P}_R,\theta_w)} \mathbb{P}_{\bm{x}}\mbox{-}\operatorname{VaR}_{1-\alpha}(e_S^T\bm{x}) = e_S^T\mu + \eta^* \sqrt{e_S^T\Sigma e_S},
\end{equation*}
where $e_S$ is the indicator vector whose $i$-th entry takes value one if $i\in S\subseteq\{1,\ldots,n\}$.
And $\eta^*$ is the minimizer to the following optimization problem
\begin{equation}
\label{eq: eta*}
\begin{aligned}
    \min_{\eta} \; & \eta\\
    \mathrm{s.t.} \; &\eta \geq \Phi^{-1}(1-\alpha),\\ 
    & \eta\big(\Phi(\eta)\!-\!(1\!-\!\alpha)\big)\!-\!\int_{\frac12(\Phi^{-1}(1\!-\!\alpha))^2}^{\eta^2/2}k \!\cdot\! g(z)dz \!\geq\! \theta_w.\\
\end{aligned}
\end{equation}
\end{theorem}
\begin{corollary}
\label{corollary: chance constraint equivalent deterministic form}
Under the assumption of Theorem~\ref{thm: WASSERSTEIN distributionally robust theorem}, for the one dimensional stochastic variable $\bm{x} \in \mR$ with the expectation $\mE[\bm{x}] = \mu$ and variance $\Sigma$, let the reference distribution $\mathbb{P}_R(\mu, \Sigma, g)$ be a one dimensional elliptical distribution. Then, the worst-case chance constraint has an equivalent deterministic form as
\begin{equation}
\label{eq: worst-case chance constraint equivalent deterministic form}
\begin{aligned}
\inf_{\mathbb{P}_{\bm{x}} \in \mathcal{F}_w(\mathbb{P}_R,\theta_w)}&\mathcal{P}(\bm{x} \geq a) \geq 1-\alpha 
\Leftrightarrow \mu \geq a + \eta^*\Sigma^{\frac{1}{2}},
\end{aligned}
\end{equation}
where $\eta^*$ is the minimizer to \eqref{eq: eta*}.
\end{corollary}
\begin{proof}
The left hand side of \eqref{eq: worst-case chance constraint equivalent deterministic form} can be rewritten as 
\begin{equation}
\label{eq: Wasserstein lower bound proof}
\begin{aligned}
&\inf_{\mathbb{P}_{\bm{x}} \in \mathcal{F}_w(\mathbb{P}_R,\theta_w)}\mathcal{P}(\bm{x} - a \geq 0) \geq 1-\alpha \\
\Leftrightarrow &\sup_{\mathbb{P}_{\bm{x}} \in \mathcal{F}_w(\mathbb{P}_R,\theta_w)}\mathcal{P}(a - \bm{x} \geq 0) \leq\alpha \\
\Leftrightarrow &\sup_{\mathbb{P}_{\bm{x}} \in \mathcal{F}_w(\mathbb{P}_R,\theta_w)} \mathbb{P}_{\bm{x}}\mbox{-}\operatorname{VaR}_{1-\alpha}(a - \bm{x}) \leq  0
\end{aligned}
\end{equation}
Furthermore, let $S:\mR\mapsto\mR$ be the coordinate transformation $\tilde{\bm x} = S(\bm x):=a-\bm x$. Since $S(\cdot)$ is isometry and $S=S^{-1}$, then it hold for the Wasserstein distance that $d_w(\mP_{\bm x}\circ S^{-1},\mP_{R})=d_w(\mP_{\bm{x}},\mP_{R}\circ S)=d_w(\mP_{\bm{x}},\mP_{R}\circ S^{-1})$. 
Hence it holds for the image $S(\mathcal{F}_w(\mP_R,\theta_w)) = \mathcal{F}_w(\mP_R\circ S^{-1},\theta_w)$. Since $\mP_R(\mu,\Sigma,g)$ is one dimensional elliptical distribution, the density of $\mP_R\circ S^{-1}$ is $k\cdot g(\frac{1}{2}(\tilde{\bm x}-(a-\mu))^T\Sigma^{-1}(\tilde{\bm x}-(a-\mu))$, and therefore by Thm.~\ref{thm: WASSERSTEIN distributionally robust theorem}
\begin{align*}
    &\sup_{\mathbb{P}_{\bm{x}} \in \mathcal{F}_w(\mathbb{P}_R,\theta_w)} \mathbb{P}_{\bm x}\mbox{-}\operatorname{VaR}_{1-\alpha}(a - \bm{x}) \leq  0\\
    &\Leftrightarrow\sup_{\mP_{\tilde{\bm x}}\in \mathcal{F}_w(\mP_R\circ S^{-1},\theta_w)} \mathbb{P}_{\bm \tilde{\bm x}}\mbox{-}\operatorname{VaR}_{1-\alpha}(\tilde{\bm{x}}) \leq  0 \\
    &\Leftrightarrow \quad a-\mu+\eta^*\Sigma^{\frac{1}{2}}\le 0
    \Leftrightarrow \mu \geq a + \eta^*\Sigma^{\frac{1}{2}}.
\end{align*}
\end{proof}

\section{Uncertainty-aware Optimization-Based Collision Avoidance}
\label{sec: Uncertainty-aware Optimization-Based Collision Avoidance}
Notably, 
when the uncertainties in the poses of the vehicle and the obstacles are further considered in the trajectory planning process, the constraints \eqref{eq: distance between polygon and point reformulation}, \eqref{eq: distance between two polygons reformulation} are no longer sufficient to guarantee the safety. Therefore, we use a probability framework to quantify the collision risk through chance constraints. More specifically, given a maximum collision probability $\alpha$ threshold, we generalize the OBCA constraints \eqref{eq: distance between polygon and point reformulation}, \eqref{eq: distance between two polygons reformulation} into the following uncertainty-aware form for the time instants $k = 1:N$
\begin{subequations}
\label{eq:U_OBCA_constraints}
\begin{align}
&\mbox{\textbf{Round-shaped obstacle chance constraints: }}
\label{eq: Round chance constraint} \\
&\mathcal{P}\Big(\exists\mu^j_k\succeq 0: \|\bm{A}^{v,T}_k\mu^j_k\|\leq 1, \nonumber\\
&\quad(\bm{A}^v_k\bm{t}^{o_j}_k-\bm{b}^v_k)^T\mu^j_k\geq{\rm d}_{\min}+r_{o_j}\Big) \geq 1 - \alpha, \forall j \in \mathcal{O}_r, \nonumber\\
& \mbox{\textbf{Polygon-shaped obstacle chance constraints: }}\label{eq: Polygon chance constraint}\\
& \mathcal{P}\Big(\exists\lambda^j_k,\mu^j_k\succeq 0: \|\bm{A}^{o_j, T}_k\lambda^j_k\|\leq 1, \nonumber\\
&\quad-\bm{b}^{v,T}_k\mu^j_k -\bm{b}^{o_j,T}_k\lambda^j_k\geq{\rm d}_{\min},\nonumber\\
&\quad\bm{A}^{v,T}_k\mu^j_k \!+\! \bm{A}^{o_j, T}_k\lambda^j_k \!=\! \bm{0}_{2\times1} \Big) \geq 1 - \alpha, \forall j \in \mathcal{O}_p\nonumber
\end{align}
\end{subequations}
Notably, $\bm{A}^v_k$, $\bm A^{o_j}_k$, $\bm b^v_k$,  $\bm b^{o_j}_k$, $\bm t^{o_j}_k$ in \eqref{eq:U_OBCA_constraints} are all stochastic matrices and vectors, and we guarantee the vehicle's safety by constraining the probability of satisfying the OBCA constraints to be greater than a threshold $1-\alpha$.
Next, to make the optimization with the chance constraints \eqref{eq:U_OBCA_constraints} numerically solvable,
we reformulate the constraints \eqref{eq:U_OBCA_constraints} into deterministic constraints. 

\subsection{Round-Shaped Obstacles}
To proceed, we first focus on the collision-avoidance with round-shaped obstacles $\mathbb{O}_j, j \in \{1:N_r\}$. 
The difficulties in quantifying the probability in \eqref{eq: Round chance constraint} lies in the fact that: 1) the value of dual variables $\mu^j_k$ are changing together with the value of the stochastic variables; 2) the distributions of stochastic variables (including the noises) in \eqref{eq: Round chance constraint} are unknown. 
To address these difficulties, we first fix the values of the dual variables to obtain a sufficient condition of \eqref{eq: Round chance constraint}.
Then, we tighen the sufficient condition into worst-case chance constraints and transform the worst-case chance constraints into deterministic nonlinear constraints so that it can be solved by optimization solvers. More precisely, we have the following theorem.
\begin{theorem}
\label{theorem: round sufficient condition}  
Suppose Assumption~\ref{ass: zero_mean_and_known_covariance}-\ref{ass: deterministic_sizes_and_shape} holds.
Denote $\bm{p}^j_k := \bar{A}^v R^{-1}(\bm{\theta}^v_k)\delta \bm{t}^j_k +  \bar{b}^v$, where $\delta \bm{t}^j_k := \bm{t}^v_k - \bm{t}^{o_j}_k$, and suppose it holds for the distribution $\mP_{\bm \phi_{k}^j}$ of the random variable $\bm{\phi}_{k}^j:=-\bm{p}^{j,T}_k\mu^j_k$ that $\mathbb{P}_{\bm{\phi}_{k}^j} \in \mathcal{F}_w(\mathbb{P}_{\bm{\phi}_{k}^j}^R, \theta_{w})$, where the reference distribution $\mathbb{P}_{\bm{\phi}_{k}^j}^R := \mathcal{N}(-\mE[\bm{p}^{j,T}_k]\mu^j_k, \mu^{j,T}_k \Sigma_{\bm{p}^j_k}\mu^j_k)$. Then the deterministic constraint
\begin{equation}
\label{eq: round sufficient condition theorem}
\begin{aligned}
&\exists\mu^j_k\succeq 0: \|\bar{A}^{v,T}\mu^j_k\|\leq 1,\\
&\mE[\bm{p}^j_k]^T \mu^j_k \!+\! {\rm d}_{\min} \!+\! r_{o_j} \!+\!\eta^*\sqrt{\mu^{j,T}_k \Sigma_{\bm{p}^j_k}\mu^j_k} \leq 0,\\
\end{aligned}
\end{equation}
is a sufficient condition to \eqref{eq: Round chance constraint} for $\forall j \in \mathcal{O}_r$,
where $\eta^*$ is computed off-line in a similar way to \eqref{eq: eta*} and the computation of $\mE[\bm{p}^{j}_k]$ and $\Sigma_{\bm{p}^j_k}$ is described in Appendix~\ref{app: expectation and covariance in round condition}. 
\end{theorem}

\begin{proof}
Note that the rotation matrix $R$ is orthogonal, i.e. $R^{-1}(\bm{\theta}) = R^T(\bm{\theta})$ even if $\bm{\theta}$ is stochastic, we first rewrite the term $\|\bm{A}^{v,T}_k\mu^j_k\|$ in \eqref{eq: Round chance constraint} into the following deterministic form
\begin{equation*}
\label{eq: A_v mu equation reformulation}
\begin{aligned}
&\|\bm{A}^{v,T}_k\mu^j_k\|
=  \sqrt{\mu^{j,T}_k\bm{A}^v_k\bm{A}^{v,T}_k\mu^j_k} \\
=&  \sqrt{\mu^{j,T}_k\bar{A}^vR^{-1}(\bm{\theta}^v_k)R(\bm{\theta}^v_k)\bar{A}^{v,T}\mu^j_k}\\
=&  \sqrt{\mu^{j,T}_k\bar{A}^v\bar{A}^{v,T}\mu^j_k}
=  \|\bar{A}^{v,T}\mu^j_k\|.
\end{aligned}
\end{equation*}
And we can further rewrite the term $\bm{A}^v_k\bm{t}^{o_j}_k-\bm{b}^v_k$ in \eqref{eq: Round chance constraint} as
\begin{equation*}
\label{eq: A_v R t_o equation reformulation}
\begin{aligned}
&\bm{A}^v_k\bm{t}^{o_j}_k-\bm{b}^v_k=\bar{A}^vR^{-1}(\bm{\theta}^v_k)\bm{t}^{o_j}_k - (\bar{A}^vR^{-1}(\bm{\theta}^v_k)\bm{t}^v_k + \bar{b}^v) \\
&= -\bar{A}^vR^{-1}(\bm{\theta}^v_k)\delta \bm{t}^j_k - \bar{b}^v = -\bm{p}^j_k.
\end{aligned}
\end{equation*}
Therefore, \eqref{eq: Round chance constraint} is equivalent to 
\begin{equation}
\label{eq: round equivalence}
\begin{aligned}
& \mathcal{P}\Big(\exists\mu^j_k\succeq 0: \|\bar{A}^{v,T}\mu^j_k\|\leq 1, \\
&\quad -\bm{p}^{j,T}_k\mu^j_k\geq{\rm d}_{\min} + r_{o_j}\Big) \geq 1 - \alpha,\quad \forall j\in\mathcal{O}_r.
\end{aligned}
\end{equation}
Next, we exchange the order of the existential quantifier and the probability in \eqref{eq: round equivalence}. This yields tightening of the constraints which is more conservative: if there exists a fixed $\mu_j^k$ satisfying the probabilistic constraint, then the original probabilistic statement with the existential quantifier inside is also satisfied. Namely, it holds that
\begin{align}
    \exists\mu^j_k\!\succeq\! 0:\! \|\bar{A}^{v,T}\mu^j_k\|\leq 1,\mathcal{P}(-\bm{p}^{j,T}_k\mu^j_k\!\geq\!{\rm d}_{\min}\! +\! r_{o_j})\! \geq\! 1\! -\! \alpha \Rightarrow \eqref{eq: round equivalence}
\end{align}
Moreover, we further tighten the constraint and apply the Corollary~\ref{corollary: chance constraint equivalent deterministic form}, which yields
\begin{equation*}
\label{eq: round sufficient condition}
\begin{aligned}
\eqref{eq: round equivalence}
\Leftarrow\; &\exists\mu^j_k\succeq 0: \|\bar{A}^{v,T}\mu^j_k\|\leq 1,\\
&\mathcal{P}(-\bm{p}^{j,T}_k\mu^j_k\geq{\rm d}_{\min} + r_{o_j}) \geq 1 - \alpha,\\
\Leftarrow\; &\exists\mu^j_k\succeq 0: \|\bar{A}^{v,T}\mu^j_k\|\leq 1,\\
&\inf_{\substack{\mathbb{P}_{\bm{\phi}_{k}^j}
\in \mathcal{F}_w(\mathbb{P}_{\bm{\phi}_{j,k}}^R, \theta_{w})}}\mathcal{P}(\bm{\phi}^j_k\geq{\rm d}_{\min} + r_{o_j}) \geq 1 - \alpha,\\
&\xLeftrightarrow{\text{Corollary }\ref{corollary: chance constraint equivalent deterministic form}} \eqref{eq: round sufficient condition theorem},
\end{aligned}
\end{equation*}
where 
the second sufficient condition holds as we consider the worst-case condition in $\mathbb{P}_{\bm{\phi}_{j,k}} \in \mathcal{F}_w(\mathbb{P}_{\bm{\phi}_{j,k}}^R, \theta_{w})$. Thus the statement is proved. 
\end{proof}

\subsection{Polygon-Shaped Obstacles}
In Theorem~\ref{theorem: round sufficient condition}, we have tighten the chance constraints for round-shaped obstacles into nonlinear deterministic constraints by fixing the dual variables and using the Wasserstein distributionally robust constraints. Next, we will deal with the chance constraints for polygon-shaped obstacles in a similar way.

Note that the difference between \eqref{eq: Round chance constraint} and \eqref{eq: Polygon chance constraint} is that there is an extra equality constraint in \eqref{eq: Polygon chance constraint}. For such equality constraint, we can not take the dual variables $\lambda^j_k$ and $\mu^j_k$ out of the probability measure $\mathcal{P}$ in the same way as \eqref{eq: round sufficient condition theorem}.
To address this issue, we eliminate the equality constraint in \eqref{eq: Polygon chance constraint} by finding the general solution of the equality constraint. Then, similarly to \eqref{eq: round sufficient condition theorem}, we find the sufficient condition for \eqref{eq: Polygon chance constraint} by fixing the value of dual variables, then tighten the sufficient condition into worst-case chance constraints and consequently transform the worst-case chance constraints into deterministic nonlinear constraints. 

To proceed, we present the following lemma, which indicates the general solution of the equality constraint in \eqref{eq: Polygon chance constraint} has a concise expression.
\begin{lemma}
\label{lemma: general solution lemma}
If the shape of the vehicle is rectangular, namely, $\bar{A}^v = \begin{bmatrix}
1 & 0 & -1 & 0\\ 0 & 1 & 0 & -1\\
\end{bmatrix}^T = \begin{bmatrix}
\mathbf{I}_{2 \times 2} \\
-\mathbf{I}_{2 \times 2} \\
\end{bmatrix}$, 
it holds from the equality constraint $\bm{A}^{v,T}_k\mu^j_k+\bm{A}^{o_j, T}_k\lambda^j_k=\bm{0}_{2\times1}$ that
\begin{equation}
\label{eq: general solution form}
\mu^j_k =  \xi^j_k - R^T_g(\delta \bm{\theta}^j_k)\bar{A}^{o_j, T}\lambda^j_k,
\end{equation}
where $ \xi^j_k := \begin{bmatrix}
\xi^{j,1}_k & \xi^{j,2}_k & \xi^{j,1}_k & \xi^{j,2}_k \\
\end{bmatrix}^T \in \mR^4 $, $R_g(\delta \bm{\theta}^j_k) := \begin{bmatrix}
R(\delta \bm{\theta}^j_k) &  \bm{0}_{2 \times 2}
\end{bmatrix} \in \mR^{2 \times 4}$ and $\delta \bm{\theta}^j_k := \bm{\theta}^v_k - \bm{\theta}^{o_j}_k$ is the difference between the orientation angle of the vehicle and the $j$-th obstacle. 
\end{lemma}
\begin{proof} 
    Note that $R^{-1}(\theta) = R^T(\theta) = R(-\theta)$, the equation $\bm{A}^{v,T}_k\mu^j_k+\bm{A}^{o_j, T}_k\lambda^j_k=\bm{0}_{2\times1}$ can be reformulated as
    \begin{equation}
    \label{rewrite equation}
    \begin{aligned}    \;&\bm{A}^{v,T}_k\mu^j_k+\bm{A}^{o_j, T}_k\lambda^j_k=\bm{0}_{2\times1}\\  
    \Leftrightarrow\; & R(\bm{\theta}^v_k)\bar{A}^{v,T}\mu^j_k+R(\bm{\theta}^{o_j}_k)\bar{A}^{o_j, T}\lambda^j_k=\bm{0}_{2\times1}  \\  
    \Leftrightarrow\; & \bar{A}^{v,T}\mu^j_k+R^{-1}(\bm{\theta}^v_k)R(\bm{\theta}^{o_j}_k)\bar{A}^{o_j, T}\lambda^j_k=\bm{0}_{2\times1}\\
    \Leftrightarrow\; & \bar{A}^{v,T}\mu^j_k+R(\bm{\theta}^{o_j}_k-\bm{\theta}^v_k)\bar{A}^{o_j, T}\lambda^j_k=\bm{0}_{2\times1}\\
    \Leftrightarrow\; & \bar{A}^{v,T}\mu^j_k = -R(-\delta \bm{\theta}^j_k)\bar{A}^{o_j, T}\lambda^j_k.\\
    \end{aligned}
    \end{equation}
Since the null space of $\bar{A}^v$ is spanned by $\begin{bmatrix} 1 & 0 & 1 & 0 \end{bmatrix}^T$ and $\begin{bmatrix} 0 & 1 & 0 & 1 \end{bmatrix}^T$. 
By substituting $\bar{A}^v$ into the last equation in \eqref{rewrite equation}, we can write $\mu_k^j$ as
\begin{equation*}
\begin{aligned}
\mu^j_k = &\xi^{j,1}_k \begin{bmatrix}
1 \\ 0 \\ 1 \\ 0\\
\end{bmatrix} \!+\! 
\xi^{j,2}_k \begin{bmatrix}
0 \\ 1 \\ 0 \\ 1\\
\end{bmatrix} \!-\! 
\begin{bmatrix}
R(\!-\!\delta \bm{\theta}^j_k) \\  \bm{0}_{2 \times 2}
\end{bmatrix}\bar{A}^{o_j, T}\lambda^j_k\\
= &\xi^j_k - R^T_g(\delta \bm{\theta}^j_k)\bar{A}^{o_j, T}\lambda^j_k.
\end{aligned}
\end{equation*}
Hence the lemma is proved.
\end{proof}

Through Lemma~\ref{lemma: general solution lemma}, we can eliminate the equality constraint in \eqref{eq: Polygon chance constraint} by plugging in \eqref{eq: general solution form}. Consequently, the following proposition holds.
\begin{proposition}
\label{proposition: fixing the value of dual variables}   
If the assumption in Lemma~\ref{lemma: general solution lemma} holds, given any $\alpha_1, \alpha_2, \alpha_3\in[0,1]$ such that $0 \leq \alpha_1 + \alpha_2 + \alpha_3 \leq \alpha$, the following
\begin{equation}
\label{eq: polygon sufficient condition 1}
\begin{aligned}
&\exists\lambda^j_k\succeq 0, \xi^{j,1}_k, \xi^{j,2}_k \geq 0: \|\bar{A}^{o_j, T}\lambda^j_k\|\leq 1,\\
&\mathcal{P}(\xi^{j,1}_k - \bm{q}^{j,1,T}_k\lambda^j_k \geq 0) \geq 1 - \alpha_1,\\
&\mathcal{P}(\xi^{j,2}_k - \bm{q}^{j,2,T}_k\lambda^j_k \geq 0) \geq 1 - \alpha_2,\\
&\mathcal{P}(\bm{r}^{j,T}_k\lambda^j_k -\bm{b}^{v,T}_k\xi^j_k \geq{\rm d}_{\min})\!\geq\! 1 - \alpha_3, \\
\end{aligned}
\end{equation}
is a sufficient condition to \eqref{eq: Polygon chance constraint} for $\forall j \in \mathcal{O}_p$,
where 
\begin{equation}
\label{p_j, q_j}
\begin{aligned}
 \bm{q}^{j,1}_k &:= \bar{A}^{o_j} \begin{bmatrix}
\cos\delta\bm{\theta}^j_k & \sin\delta\bm{\theta}^j_k \end{bmatrix}^T,\\
 \bm{q}^{j,2}_k &:= \bar{A}^{o_j} \begin{bmatrix}
-\sin\delta\bm{\theta}^j_k & \cos\delta\bm{\theta}^j_k \end{bmatrix}^T,\\
 \bm{r}^j_k &:= \bar{A}^{o_j} R_g(\delta \bm{\theta}^j_k)\bm{b}^v_k \!-\! \bm{b}^{o_j}_k.
\end{aligned}
\end{equation}
and $ \xi^j_k := \begin{bmatrix}
\xi^{j,1}_k & \xi^{j,2}_k & \xi^{j,1}_k & \xi^{j,2}_k \end{bmatrix}^T$.
\end{proposition}

\begin{proof}
Similar to Sec.~\ref{eq: A_v mu equation reformulation}, $\|\bm{A}^{o_j, T}_k\lambda^j_k\|$ can be reformulated as the following deterministic form:
\begin{equation}
\label{eq: A_o lambda equation reformulation}
\|\bm{A}^{o_j, T}_k\lambda^j_k\|= \|\bar{A}^{o_j,T}\lambda^j_k\|,    
\end{equation}
and note that by substituting~\eqref{eq: general solution form} into the term $-\bm{b}^{v,T}_k \mu^j_k - \bm{b}^{o_j,T}_k \lambda^j_k$ in the chance constraints, it holds that
\begin{equation}
    \label{eq: r^j_k}
    \begin{aligned}
    &-\bm{b}^{v,T}_k \mu^j_k - \bm{b}^{o_j,T}_k \lambda^j_k\\
    =& -\bm{b}^{v,T}_k\!\left(\xi^j_k - R_g^{T}(\delta \bm{\theta}^j_k)\bar{A}^{o_j,T}\lambda^j_k\right) - \bm{b}^{o_j,T}_k \lambda^j_k \\
    =& -\bm{b}^{v,T}_k \xi^j_k + \left(\bar{A}^{o_j} R_g(\delta \bm{\theta}^j_k)\bm{b}^v_k - \bm{b}^{o_j}_k\right)^{\!T}\lambda^j_k\\
    =& \quad \bm{r}^{j,T}_k\lambda^j_k -\bm{b}^{v,T}_k\xi^j_k.
    \end{aligned}
\end{equation}
Hence the following equivalence holds
\begin{equation}
\label{eq: polygon sufficient condition1 proof part1}
\begin{aligned}
\eqref{eq: Polygon chance constraint}\Leftrightarrow\; & \mathcal{P}(\exists\lambda^j_k,\mu^j_k\succeq 0: \|\bar{A}^{o_j, T}\lambda^j_k\|\leq 1,\\
&\quad -\bm{b}^{v,T}_k\mu^j_k -\bm{b}^{o_j,T}_k\lambda^j_k \geq {\rm d}_{\min},\\
&\quad \bm{A}^{v,T}_k\mu^j_k + \bm{A}^{o_j, T}_k\lambda^j_k = \bm{0}_{2\times1} ) \geq 1 - \alpha. \\
\Leftrightarrow\; & \mathcal{P}(\exists\lambda^j_k\succeq 0, \xi^{j,1}_k, \xi^{j,2}_k \in \mR: \|\bar{A}^{o_j, T}\lambda^j_k\|\leq 1,\\
&\quad \xi^j_k - R^T_g(\delta \bm{\theta}^j_k)\bar{A}^{o_j, T}\lambda^j_k \succeq 0,\\
&\quad \bm{r}^{j,T}_k\lambda^j_k -\bm{b}^{v,T}_k\xi^j_k \geq{\rm d}_{\min}) \geq 1 - \alpha,\\
\Leftrightarrow\; & \mathcal{P}(\exists\lambda^j_k\succeq 0, \xi^{j,1}_k, \xi^{j,2}_k \geq 0: \|\bar{A}^{o_j, T}\lambda^j_k\|\leq 1,\\
&\quad \xi^{j,1}_k - \bm{q}^{j,1,T}_k\lambda^j_k \geq 0, \xi^{j,2}_k - \bm{q}^{j,1,T}_k\lambda^j_k \geq 0,\\
&\quad \bm{r}^{j,T}_k\lambda^j_k -\bm{b}^{v,T}_k\xi^j_k \geq{\rm d}_{\min}) \geq 1 - \alpha,
\end{aligned}
\end{equation} 
where 
the third equivalence results from separating each rows of the inequality constraint.
Then, by extracting $\lambda^j_k$, $\xi^{j,1}_k$ and $\xi^{j,2}_k$ from the probability measure $\mathcal{P}(\cdot)$ similar to \eqref{eq: round equivalence},
and applying Theorem~\ref{thm: Bounds on Disjunctive Chance Constraints}, we get the following sufficient condition form for \eqref{eq: polygon sufficient condition1 proof part1} and hence for \eqref{eq: Polygon chance constraint}
\begin{equation}
\label{eq: polygon sufficient condition1 proof part2}
\begin{aligned}
\eqref{eq: Polygon chance constraint}\Leftarrow \; & \exists\lambda^j_k\succeq 0, \xi^{j,1}_k, \xi^{j,2}_k \geq 0: \|\bar{A}^{o_j, T}\lambda^j_k\|\leq 1,\\
&\mathcal{P}(\xi^{j,1}_k - \bm{q}^{j,1,T}_k\lambda^j_k \geq 0, \; \xi^{j,2}_k - \bm{q}^{j,2,T}_k\lambda^j_k \geq 0,\\
&\quad \bm{r}^{j,T}_k\lambda^j_k -\bm{b}^{v,T}_k\xi^j_k \geq{\rm d}_{\min}) \geq 1 - \alpha,\Leftarrow\eqref{eq: polygon sufficient condition 1}
\end{aligned}
\end{equation}    
\end{proof}
So far, we have eliminated the equality constraint and tighten the constraint \eqref{eq: Polygon chance constraint} with its sufficient condition \eqref{eq: polygon sufficient condition 1}. Nevertheless, as the distribution of stochastic variables $\bm{r}^j_k$, $\bm{q}^{j,1}_k$ and $\bm{q}^{j,2}_k$ in \eqref{eq: polygon sufficient condition 1} are unknown, the chance constraints in \eqref{eq: polygon sufficient condition 1} are still intractable. 
To address this issue, similar to Theorem~\ref{theorem: round sufficient condition}, we further tighten the sufficient condition into worst-case chance constraints through Wasserstein ambiguity set and then transform the worst-case chance constraints into deterministic nonlinear constraints.
More specifically, under the Assumptions in Section~\ref{sec: Problem Statement}, the following theorem holds.
\begin{theorem}
\label{theorem: polygon obstacle deterministic nonlinear constraints}
Suppose Assumption~\ref{ass: zero_mean_and_known_covariance}-\ref{ass: deterministic_sizes_and_shape} holds.
Denote $\bm{\psi}_{j,k,1}:=-\bm{q}^{j,1,T}_k\lambda^j_k$, $\bm{\psi}_{j,k,2}:=-\bm{q}^{j,2,T}_k\lambda^j_k$ and $\bm{\psi}_{j,k,3} := \bm{r}^{j,T}_k\lambda^j_k -\bm{b}^{v,T}_k\xi^j_k$, and suppose it holds that the distribution $\mathbb{P}_{\bm{\psi}_{j,k,i}} \in \mathcal{F}_w(\mathbb{P}_{\bm{\psi}_{j,k,i}}^R, \theta_{w_i}), i \in \{1,2,3\}$, where the reference distribution  $\mathbb{P}_{\bm{\psi}_{j,k,i}}^R = \mathcal{N}(\mE[\bm{\psi}_{j,k,i}], \Sigma_{\bm{\psi}_{j,k,i}}), i \in \{1,2,3\}$. Then
\begin{equation}
\label{sufficient condition2}
\begin{aligned}
& \exists\lambda^j_k\succeq 0, \xi^{j,1}_k, \xi^{j,2}_k \geq 0: \|\bar{A}^{o_j, T}\lambda^j_k\|\leq 1,\\
&\xi^{j,1}_k \!-\! \mE[\bm{q}^{j,1}_k]^T\lambda^j_k \!\geq\! \eta_1^*\sqrt{\lambda^{j,T}_k \Sigma_{\bm{q}^{j,1}_k}\lambda^j_k}, \\
&\xi^{j,2}_k \!-\! \mE[\bm{q}^{j,2}_k]^T\lambda^j_k \!\geq\! \eta_2^*\sqrt{\lambda^{j,T}_k \Sigma_{\bm{q}^{j,2}_k}\lambda^j_k},\\
&\mE[\bm{r}^j_k]^T\lambda^j_k -\mE[\bm{b}^v_k]^T\xi^j_k \geq {\rm d}_{\min} + \eta_3^*\Sigma_{\bm{\psi}{j,k,3}}^{\frac{1}{2}}, \\
\end{aligned}
\end{equation}
is a sufficient deterministic condition for \eqref{eq: polygon sufficient condition 1},
where $\eta_1^*$, $\eta_2^*$ and $\eta_3^*$ are computed in a similar way to \eqref{eq: eta*} and the expectation $\mE[\bm{q}^{j,1}_k]$, $\mE[\bm{q}^{j,2}_k]$, $\mE[\bm{r}^j_k]$, $\mE[\bm{b}^v_k]$, the covariance matrices $\Sigma_{\bm{q}^{j,1}_k}$, $\Sigma_{\bm{q}^{j,2}_k}$, $\Sigma_{\bm{\psi}_{j,k,3}}$ are computed in detail in Appendix~\ref{{app: expectation and covariance in polygon condition}}.  
\end{theorem}
\begin{proof}
Similar to the proof of Theorem~\ref{theorem: round sufficient condition}, it holds that 
\begin{equation*}
\label{sufficient condition2 proof}
\begin{aligned}
\eqref{eq: polygon sufficient condition 1}\Leftarrow\; & \exists\lambda^j_k\succeq 0, \xi^{j,1}_k, \xi^{j,2}_k \geq 0: \|\bar{A}^{o_j, T}\lambda^j_k\|\leq 1,\\
&\inf_{\mathbb{P}_{\bm{\psi}_{j,k,1}}}
\mathcal{P}(\xi^{j,1}_k \!-\! \bm{q}^{j,1,T}_k\lambda^j_k \geq 0) \geq 1 \!-\! \alpha_1,\\
&\inf_{\mathbb{P}_{\bm{\psi}_{j,k,2}}}\mathcal{P}(\xi^{j,2}_k \!-\! \bm{q}^{j,2,T}_k\lambda^j_k \geq 0) \geq 1 \!-\! \alpha_2,\\
&\inf_{\mathbb{P}_{\bm{\psi}_{j,k,3}}}\mathcal{P}(\bm{r}^{j,T}_k\lambda^j_k \!-\!\bm{b}^{v,T}_k\xi^j_k \geq{\rm d}_{\min})\!\geq\! 1 \!-\! \alpha_3, \\
&\xLeftrightarrow{\text{Corollary }\ref{corollary: chance constraint equivalent deterministic form}} \eqref{sufficient condition2}.
\end{aligned}
\end{equation*}
\end{proof}
Through Theorem~\ref{theorem: round sufficient condition} and Theorem~\ref{theorem: polygon obstacle deterministic nonlinear constraints}, we tighten the chance constraints \eqref{eq:U_OBCA_constraints} into deterministic ones, i.e., 
\begin{align}
&\mbox{\textbf{Round-shaped obstacle deterministic constraints: }}
\label{Round deterministic constraint}\\
&\mu^j_k\succeq 0, \|\bar{A}^{v,T}\mu^j_k\|\leq 1,\nonumber\\
&\mE[\bm{p}^j_k]^T \mu^j_k \!+\! {\rm d}_{\min} \!+\! r_{o_j} \!+\!\eta^*\sqrt{\mu^{j,T}_k \Sigma_{\bm{p}^j_k}\mu^j_k} \leq 0, \nonumber\\
&\forall k = 1, \dots, N, \quad\forall j \in \mathcal{O}_r, \nonumber\\
&\mbox{\textbf{Polygon-shaped obstacle deterministic constraints:} }
\label{Polygon deterministic constraint}\\
&\lambda^j_k\succeq 0, \xi_k^{j,1}\geq 0, \xi_k^{j,2}\geq 0, \|\bar{A}^{o_j, T}\lambda^j_k\|\leq 1,\nonumber\\
&\xi^{j,1}_k \!-\! \mE[\bm{q}^{j,1}_k]^T\lambda^j_k \!\geq\! \eta_1^*\sqrt{\lambda^{j,T}_k \Sigma_{\bm{q}^{j,1}_k}\lambda^j_k}, \nonumber\\
&\xi^{j,2}_k \!-\! \mE[\bm{q}^{j,2}_k]^T\lambda^j_k \!\geq\! \eta_2^*\sqrt{\lambda^{j,T}_k \Sigma_{\bm{q}^{j,2}_k}\lambda^j_k}, \nonumber\\
&\mE[\bm{r}^j_k]^T\lambda^j_k -\mE[\bm{b}^v_k]^T\xi^j_k \geq {\rm d}_{\min} + \eta_3^*\Sigma_{\bm{\psi}{j,k,3}}^{\frac{1}{2}}\nonumber\\
&\forall k = 1: N, \quad\forall j \in \mathcal{O}_p. \nonumber
\end{align}
And consequently, we transform the probabilistic motion planning problem \eqref{eq: Overall Formulation} into the  following deterministic nonlinear optimization problem
\begin{equation}
\label{eq: deterministic nonlinear optimization problem}
\begin{aligned}
\min_{\substack{\bar{s}, u, \lambda, \mu, \xi}}\;\; &J(\bar{s}, u),\\ 
\mbox{s.t.\quad \quad} &{\rm Nominal \ kinematic \ constraints} \ \eqref{eq: nominal kinematics}, \\
\qquad&{\rm Deterministic \ safety \ constraints} \ \eqref{Round deterministic constraint}, \eqref{Polygon deterministic constraint},\\
\qquad&{\rm State \ and \ control \ constraints},\\ 
\end{aligned}
\end{equation}
where $\lambda, \mu, \xi$ are short hand notations for all of the dual variables. Through such reformulation, \eqref{eq: deterministic nonlinear optimization problem} can be solved efficiently by gradient-based nonlinear solvers.
\begin{remark}
In practice, we estimate the covariance matrix $\Sigma_{\bm{p}^j_k}$, $\Sigma_{\bm{q}^{j,1}_k}$, $\Sigma_{\bm{q}^{j,2}_k}$ and $\Sigma_{\bm{q}^{j,3}_k}$ in Theorem~\ref{theorem: round sufficient condition} and Theorem~\ref{theorem: polygon obstacle deterministic nonlinear constraints} by replacing the variables $\bar{x}^v_k, \bar{y}^v_k, \bar{\theta}^v_k, k =1:N$ and $\bar{x}^{o_j}_k, \bar{y}^{o_j}_k, \bar{\theta}^{o_j}_k, j \in \{1,\ldots,N_r+N_p\}, k =1:N$ in Appendix~\ref{app: expectation and covariance in round condition} and Appendix~\ref{{app: expectation and covariance in polygon condition}} with the initial guess of these decision variables. This would further accelerate the numerical computation as the covariance matrix can be computed through the initial guess before the numerical solving process. With such approximation, the OCP \eqref{eq: deterministic nonlinear optimization problem} can be efficiently solved. However, the trade-off of the computational advantage is that there exists deviations between the estimated uncertainties and true uncertainties, yet we claim the deviation is tolerable as long as we can guarantee the assumptions in Theorem~\ref{theorem: round sufficient condition} and Theorem~\ref{theorem: polygon obstacle deterministic nonlinear constraints} are satisfied by choosing large enough $\theta_w$ in \eqref{eq: round sufficient condition theorem} and $\theta_{w_1}$, $\theta_{w_2}$, $\theta_{w_3}$ in \eqref{sufficient condition2}.
\end{remark}
\section{Simulation Results}

We implement the proposed algorithm in MATLAB R2022b and execute it on a laptop equipped with an Intel Core i5-12500H CPU and 16\,GB RAM. 
Two scenarios—the parallel parking scenario and the narrow corridor navigation scenario—are considered to evaluate the performance of the proposed method in narrow environments and to compare it against benchmark approaches.

\subsection{Global Planning: Parallel Parking Trajectory Planning for a Four Wheel Steering Vehicle}\label{sec:parallel_parking_sim}
\subsubsection{The Simulation Settings}
\label{sec: The Simulation Settings}
As shown in Fig.~\ref{fig: The simulation results in the parallel parking scenario.}, in the parallel parking lot scenario, a $4.8$m $\times$ $1.962$m four-wheel-steering vehicle needs to park into a $8.4$m $\times$ $2.4$m parking space while avoiding collision with the static and moving obstacles. The uncertainties exist in the perception of the obstacles, motion of the moving obstacles and the tracking control errors caused by environmental disturbances. Moreover, to minimize the total parking time, we let the sampling period $\Delta t$ also be the decision variable in the optimization and choose the number of the time step as $N = 100$, the risk level as $\alpha_1 = \alpha_2 = 0.2\alpha$, $\alpha_3 = 0.6\alpha$, and Wasserstein radius $\theta_w = \theta_{w_1} = \theta_{w_2} = \theta_{w_3} = 0.001$. For the numerical solver, we choose the interior-point solver IPOPT \cite{IPOPT} and set the linear solver as ``MA27" \cite{MA27} in AMPL \cite{AMPL}.

In the parking scenario, the state and control inputs of the four-wheel-steering vehicle are defined as $\bm{s}^v = [\bm{x}_v, \bm{y}_v, \bm{\theta}_v, \varphi_r, \varphi_f, v]^T$ and $u = [w_r, w_f, a]^T$, where $\varphi_r, \varphi_f$ and $v$ are the rear steering angle, front steering angle and the velocity of the vehicle, respectively. $w_r, w_f$ and $a$ are the angular velocities of the steering angles as well as the acceleration of the vehicle.
Thus, the nominal vehicle kinematics are modeled as
\begin{equation*}
\label{eq: nominal vehicle kinematics parking scenario}
\begin{aligned}
    &\bar{x}^v_{k+1} = \bar{x}^v_k + v_k\cos\bar{\theta}^v_k \cdot \Delta t , \\
    &\bar{y}^v_{k+1} = \bar{y}^v_k + v_k\sin\bar{\theta}^v_k \cdot \Delta t , \\
    &\bar{\theta}^v_{k+1} = \bar{\theta}^v_k + v_k \frac{\tan\varphi_{f,k}\cos\varphi_{r,k} - \sin\varphi_{r,k}}{L} \cdot \Delta t,\\
    &\varphi_{r,k+1} = \varphi_{r,k} + w_{r,k}\cdot \Delta t , \\
    &\varphi_{f,k+1} = \varphi_{f,k} + w_{f,k}\cdot \Delta t , \\
    &v_{k+1} = v_k + a_k\cdot \Delta t ,\quad k = 0:N-1,\\
\end{aligned}
\end{equation*}
where $L=2.8m$ is the length between the axles. Moreover, we assume to have a low-level controller to track the planned trajectory, and the tracking errors caused by the environmental disturbances as well as the pose perception error shall have the distribution $[\bm{w}_{x,k}^v, \bm{w}_{y,k}^v, \bm{w}_{\theta, k}^v]^T \sim \mathcal{N}\Big(\bm{0}_{3\times1}, \diag(10^{-3}\mbox{ m}^2, 10^{-3}\mbox{ m}^2, 0.1\mbox{ deg}^2)\Big),k =1: N$. 

Furthermore, we assume there are four static polygon-shaped obstacles in the scenario. We simulate their observations noises as $[\bm{w}^{o_j}_{x}, \bm{w}^{o_j}_{y}, \bm{w}^{o_j}_{\theta}]^T \sim \mathcal{N}\Big(\bm{0}_{3\times1}, \diag(10^{-2}\mbox{ m}^2, 10^{-2}\mbox{ m}^2, 1\mbox{ deg}^2)\Big)$, $j =1:3$ and the moving polygon-shaped obstacle shall have a horizontal constant velocity $v_{\rm obs}=0.8\mbox{m/s}$. For the bottom obstacle $(j=4)$, we let it be a deterministic obstacle without any noises to make sure that the collision avoidance constraints could be satisfied at the goal position. 
To model the uncertainty growth in the predicted positions of the moving obstacles over the planning horizon, the trajectory prediction error is simulated via a discretized random-walk process. Namely,
$[\bm{w}^{o_j}_{x,k}, \bm{w}^{o_j}_{y,k}, \bm{w}^{o_j}_{\theta,k}]^T \sim \mathcal{N}\Big(\bm{0}_{3\times1}, \diag(k \times 10^{-3}\mbox{ m}^2, k \times 10^{-3}\mbox{ m}^2$, $k \times 0.1\mbox{ deg}^2)\Big)$, $ k =1:N, j = 5$. 

We penalize the parking time, terminal error and control effort in the optimization problem. In particular, we take the cost function as
\begin{equation*}
\label{eq: cost function corridor}
\begin{aligned}
J(\bar{s}, u, \Delta t)
\!=\! (N\!-\!1)\Delta t \!+ \!\|\bar{s}^v_N\!-\!\bar{s}_g\|_Q^2 \! + \!\sum_{k=0}^{N-1}\|u_k\|_R^2,
\end{aligned}
\end{equation*}
where $\bar{s}_g$ is the goal state of the parking lot, $Q=\diag(10,10,10)$ and $R=\diag(0.01, 0.01)$ are weighting matrices.

\subsubsection{The Simulation Results}
In this simulation, we compare our method with the uncertainty-aware methods: Robust Collision Avoidance (RCA) \cite{RCA}, Linearized Chance Constraints (LCC) \cite{zhu2019chance}, Ellipsoid Chance Constraints (ECC) \cite{Castillo2020}. Since they approximate the obstacles and the vehicle with ellipses, they fail to find a feasible solution due to over-conservatism. Therefore, we only compare our method against the nominal OBCA method \cite{OBCA} to validate the ability of guaranteeing the safety under uncertainties. The results of OBCA and U-OBCA are shown in Fig.\ref{fig: Trajectory of OBCA.},~\ref{fig: Trajectory of U-OBCA (alpha=0.1).}. More specifically, as shown in Fig.\ref{fig: Trajectory of OBCA.}, OBCA chooses to get into the parking space before the moving obstacle passes by the parking space. The trajectory is more efficient when the pose of the obstacles are deterministic, however, when there exists uncertainties, the trajectory is too aggressive. In contrast, our method adapts to the collision probability uncertainties, resulting in the trajectory that avoids the obstacle first, and then parks into the parking spaces after the moving obstacle leaves. 

Moreover, to further quantitatively verify the effectiveness of U-OBCA, we evaluate its robustness against collision via a Monte-Carlo simulation with 1000 trials.
In each trial, for a given nominal trajectory that consists $N = 100$ discrete time steps, random perturbations are independently sampled and added to the nominal poses of both the vehicle and the obstacles according to the noises distributions described in Sec. ~\ref{sec: The Simulation Settings}. A collision check is performed at every time step by calculating the distance between the vehicle and obstacles under the perturbed poses. 
If a collision occurs at any time step within a trial, that trial is counted as a trajectory-level collision.
A trajectory is regarded as successful in a Monte-Carlo trial if no collision is detected at all time steps in that trial. The results of the collision test are shown in Table~\ref{table: Comparison of different methods in parking scenario}. From Table~\ref{table: Comparison of different methods in parking scenario}, under a confidence level of $\alpha = 0.01$, U-OBCA reduces the number of collisions by 99.0\% compared with the OBCA method and achieves a success rate of 97.1\%. Moreover, under a risk level of $\alpha = 0.1$, U-OBCA selects trajectories similar to the OBCA method, but still reduces the collision count by 94.4\% and attains a success rate of 87.0\%. In contrast, the OBCA method results in 3210 collisions and achieves a success rate of only 1.4\%.
However, the trade-off is that U-OBCA takes longer time to finish the task than OBCA. As a summary, U-OBCA guarantees the safety by reducing the number of collisions and improving the success rate significantly under the existence of the uncertainties.

\begin{figure}[!htbp]
\centering
\subfigure[Trajectory of OBCA.]{\includegraphics[width=0.45\hsize]{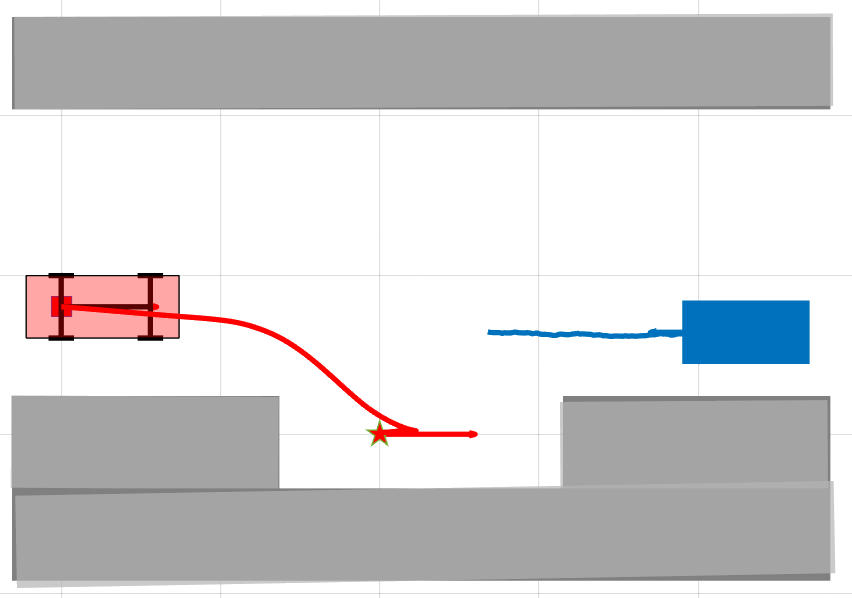}
\label{fig: Trajectory of OBCA.}}
\subfigure[Legend.]
{\includegraphics[width=0.45\hsize]{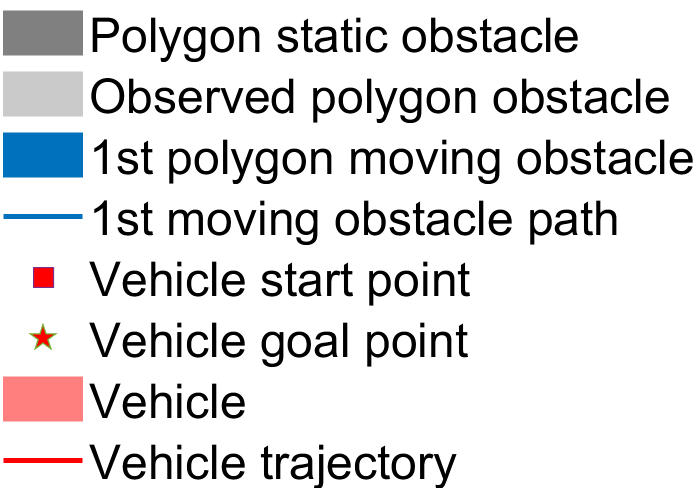}
\label{fig: Legend parking lot.}}
\subfigure[Trajectory of U-OBCA $(\alpha=0.1)$.]{\includegraphics[width=0.43\hsize]{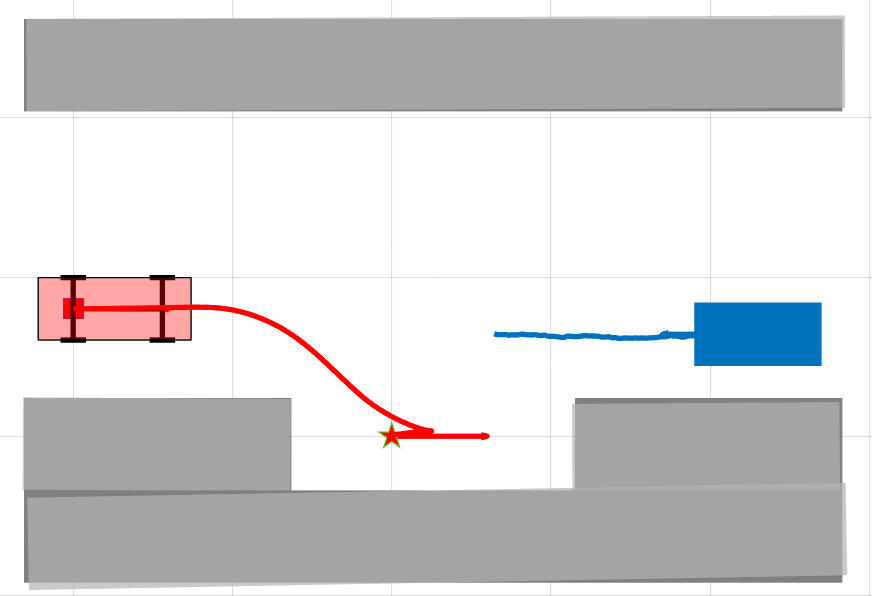}
\label{fig: Trajectory of U-OBCA (alpha=0.1).}}
\subfigure[Trajectory of U-OBCA $(\alpha=0.01)$.]{\includegraphics[width=0.43\hsize]{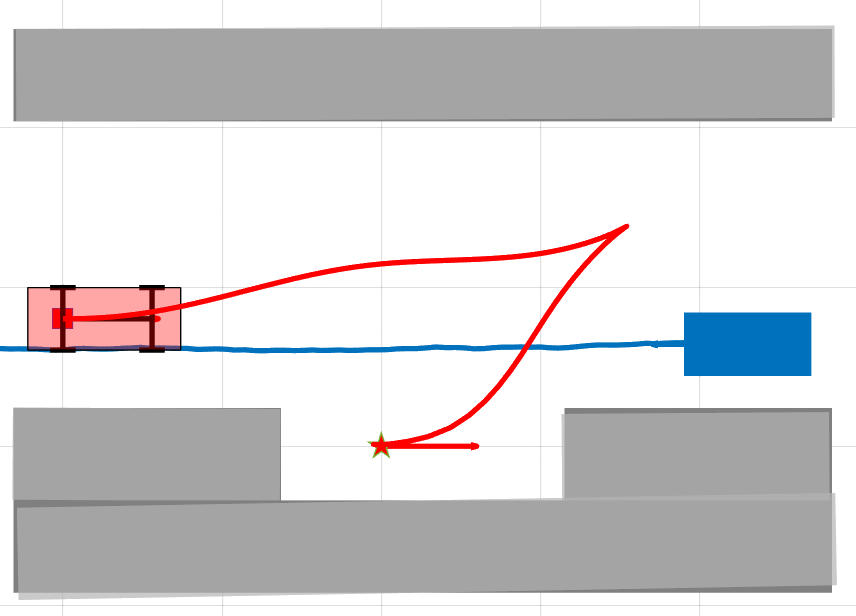}
\label{fig: Trajectory of U-OBCA (alpha=0.01).}}
\caption{The simulation results in the parallel parking scenario.}
\label{fig: The simulation results in the parallel parking scenario.}
\end{figure}

\begin{table}[!htbp]
\caption{Simulation results comparison of different methods in parking scenario}
\vspace{-10pt}
\label{table: Comparison of different methods in parking scenario}
\rowcolors{2}{white}{gray!30}
\begin{center}
\begin{tabular}{l|l|l|l}
\rowcolor{blue!10}
\textbf{Methods} & \textbf{Cost} & \textbf{Collision Num} &\textbf{Success Rate}\\
\hline
U-OBCA $(\alpha=0.01)$ & 37.18 & \textbf{32} & \textbf{97.1\%} \\
U-OBCA $(\alpha=0.1)$ & 10.54 & 178 & 87.0\% \\
OBCA \cite{OBCA} & 10.32 & 3210 & 1.4\% \\
\end{tabular}
\end{center}
\end{table}

\subsection{Local Planning: Navigation in a Narrow Corridor}\label{sec:sim_corridor_nav}

\begin{table*}[!htbp]
\caption{Simulation result comparison of the different methods in random corridor scenarios}
\vspace{-10pt}
\label{table: Comparison of different methods in random corridor scenarios}
\rowcolors{2}{white}{gray!30}
\begin{center}
\begin{tabular}{c|c|c|c|c|c}
\rowcolor{blue!10}
\textbf{Methods} & \textbf{Success Rate (\%)} & \textbf{Finishing Time (s)} & \textbf{Mean Cost Value} & \textbf{Min Distance (m)} & \textbf{Mean Computation Time (s)}\\
\hline
U-OBCA & \textbf{97} & 11.887$\pm$0.956 & 197.423$\pm$49.518 & 0.174$\pm$0.069 & 0.114$\pm$0.071 \\
OBCA \cite{OBCA} & 35 & 11.679$\pm$0.839 & 181.757$\pm$41.100 & 0.043$\pm$0.045 &  0.109$\pm$0.049 \\
RCA \cite{RCA} & 84 & 13.446$\pm$2.994 & 232.375$\pm$85.929 & 0.558$\pm$0.503 & \textbf{0.066$\pm$0.044} \\
LCC \cite{zhu2019chance} & 82 & 13.338$\pm$3.061 & 232.247$\pm$109.238 & 0.316$\pm$0.177 & 0.089$\pm$0.059 \\
ECC \cite{Castillo2020} & 82 & 13.616$\pm$3.212 & 246.505$\pm$110.087 & 0.561$\pm$0.251 & 0.085$\pm$0.063 \\
RC \cite{jasour2023convex} & 84 & 12.884$\pm$2.966 & 226.178$\pm$99.326 & 0.367$\pm$0.114 & 0.088$\pm$0.138 \\
DRCC \cite{ryu2024integrating} & 82 & 13.058$\pm$2.632 & 224.066$\pm$81.373 & 0.527$\pm$0.226 & 0.082$\pm$0.053
\end{tabular}
\end{center}
\end{table*}

\begin{figure*}[!htbp]
\centering
\subfigure[The trajectory of U-OBCA.]{\includegraphics[width=0.121\hsize]{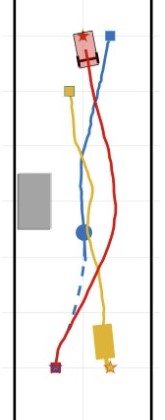}
\label{fig: Navigation trajectory of U-OBCA.}}
\subfigure[The trajectory of OBCA.]
{\includegraphics[width=0.12\hsize]{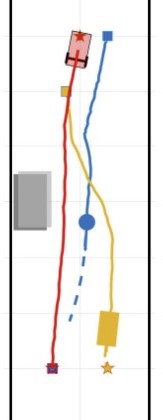}
\label{fig: Navigation trajectory of OBCA.}}
\subfigure[The trajectory of RCA.]{\includegraphics[width=0.12\hsize]{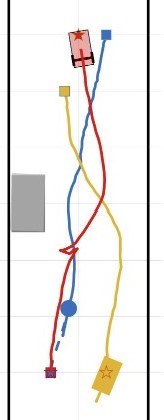}
\label{fig: Trajectory of RCA.}}
\subfigure[The trajectory of LCC.]{\includegraphics[width=0.12\hsize]{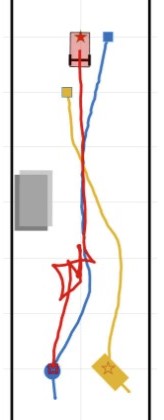}
\label{fig: Trajectory of LCC.}}
\subfigure[The trajectory of ECC.]{\includegraphics[width=0.12\hsize]{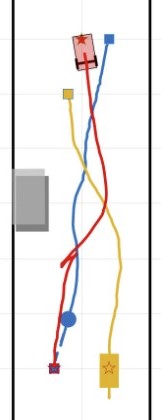}
\label{fig: Trajectory of ECC.}}
\subfigure[The trajectory of RC.]{\includegraphics[width=0.12\hsize]{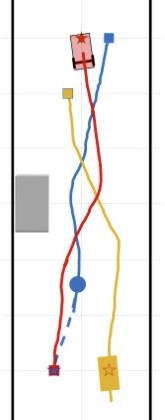}
\label{fig: Trajectory of RC.}}
\subfigure[The trajectory of DRCC.]
{\includegraphics[width=0.12\hsize]{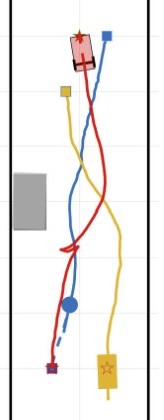}
\label{fig: Trajectory of DRCC.}}
\subfigure[The figure legend.]{\includegraphics[width=1\hsize]{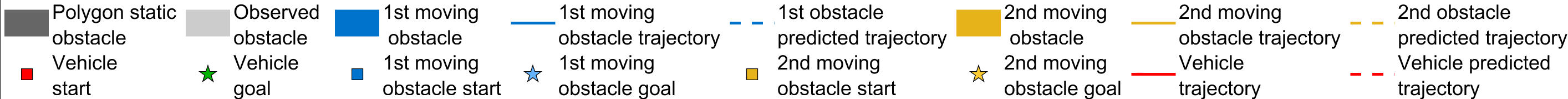}
\label{fig: Legend.}}
\caption{The simulation results in the narrow corridor scenario.}
\label{fig: The simulation results in the narrow corridor scenario.}
\end{figure*}

\subsubsection{The Simulation Settings}
As shown in Fig.~\ref{fig: The simulation results in the narrow corridor scenario.} , we further consider the navigation task for a $1.25\mathrm{m}$ $\times$ $0.7\mathrm{m}$ smart wheelchair in a $5 \mathrm{m}$-wide narrow corridor. In particular, the wheelchair needs to reach the goal while avoiding collision with static and moving obstacles. The uncertainties exist in the perception of the obstacles' and vehicle's pose, trajectory prediction of the moving obstacles and the tracking control errors caused by environmental disturbances. Moreover, we fix the sampling period $\Delta t = 0.25s$ and choose the number of the time steps as $N = 20$. We then plan a local trajectory in a receding horizon way. The risk level is set as $\alpha=0.01$, $\alpha_1 = \alpha_2 = 0.002$, $\alpha_3 = 0.006$ and the Wasserstein radius is set as $\theta_w = \theta_{w_1} = \theta_{w_2} = \theta_{w_3} = 0.001$. For the numerical solver, we choose the interior-point solver IPOPT \cite{IPOPT} and set the linear solver as ``mumps" \cite{mumps} in CasADi \cite{casadi}.

The state and control inputs of the  wheelchair are defined as $\bm{s}^v_k := [\bm{x}^v_k, \bm{y}^v_k, \bm{\theta}^v_k, v_k, \omega_k]^T$ and $u_k ;= [a^v_k, a^\omega_k]^T$, where $v_k$, $\omega_k$ are the velocity and angular velocity, $a^v_k, a^\omega_k$ are the acceleration and angular acceleration respectively.
Thus, the nominal vehicle kinematics take the form
\begin{equation}
\label{eq: nominal vehicle kinematics corridor scenario}
\begin{aligned}
    &\bar{x}^v_{k+1} = \bar{x}^v_k + v_k\cos\bar{\theta}^v_k \cdot \Delta t , \\
    &\bar{y}^v_{k+1} = \bar{y}^v_k + v_k\sin\bar{\theta}^v_k \cdot \Delta t , \\
    &\bar{\theta}^v_{k+1} = \bar{\theta}^v_k + w_k \cdot \Delta t,\\
    &v_{k+1} = v_k + a^v_k \cdot \Delta t,\\
    &w_{k+1} = w_k + a^\omega_k \cdot \Delta t,\; k = 0: N-1.
\end{aligned}
\end{equation}
Specifically, we set one static obstacle, one moving round-shaped obstacle and one moving polygon-shaped (rectangular) obstacle in the scenario.
The distributions of the random variables that corresponds to the pose perception noise of the smart wheelchair, the static and moving obstacles, as well as the environmental disturbances are identical to those in Sec.~\ref{sec:parallel_parking_sim}.
The nominal trajectories of the moving obstacle are generated by social force model \cite{helbing1995social}.

Moreover, we penalize the terminal error and control effort in the corridor navigation scenario. Namely, the cost function takes the form
\begin{equation*}
\label{eq: cost function parking lot}
\begin{aligned}
J(\bar{s}, u)=  \|\bar{s}^v_N-\bar{s}_g\|_{Q_N}^2 + \sum_{k=1}^{N-1}\|\bar{s}^v_k-\bar{s}_g\|_Q^2 + \sum_{k=0}^{N-1}\|u_k\|_R^2,
\end{aligned}
\end{equation*}
where $\bar{s}_g$ is the navigation goal state, $Q=\diag(0.1, 0.1, 1)$, $Q_N=100Q$ and $R=\diag(0.1, 0.1)$ are positive definite weighting matrices.

\subsubsection{The Simulation Results}
In this simulation, we compare U-OBCA with the nominal OBCA method \cite{OBCA} and other optimization-based uncertainty-aware methods to validate the effectiveness of the proposed method. In particular, we choose  RCA \cite{RCA}, LCC \cite{zhu2019chance}, ECC \cite{Castillo2020}, Risk Contours (RC) \cite{jasour2023convex}, and Distributionally Robust Chance Constraints (DRCC) \cite{ryu2024integrating} as baselines. 
These methods all approximate the shape of the vehicle and the obstacles with ellipses. To make the comparison fair, we test our method and the baselines under the same receding horizon setting, weighting matrices, risk level and noise distributions. 
The trajectories of a typical narrow corridor navigation scenario are shown in
Fig.~\ref{fig: The simulation results in the narrow corridor scenario.}.
As illustrated in Fig.~\ref{fig: Navigation trajectory of U-OBCA.} and \ref{fig: Navigation trajectory of OBCA.}, U-OBCA enables the smart wheelchair to safely navigate through the narrow corridor while avoiding the static and moving obstacles. However, OBCA chooses a more aggressive trajectory when interacting with the obstacles, however the trajectory is too close to the obstacles, which results in safety issues. Moreover, the trajectories in Fig.\ref{fig: Trajectory of LCC.} - \ref{fig: Trajectory of RC.} are more conservative due to the elliptical approximation. They stuck for a few seconds in front of the static obstacles to wait for the yellow moving obstacle passing by. Hence, they reach the navigation goal slower than U-OBCA.

Furthermore, to quantitatively compare the proposed U-OBCA with the baselines, we further generate 100 random corridor scenarios, where the number, starting and goal positions, size and shape of the obstacles are randomly generated. In these scenarios, we compare the success rate, mean finishing time, minimum distance to the obstacles as well as the mean computation time of these methods. The results are summarized in Table.\ref{table: Comparison of different methods in random corridor scenarios}. In particular, the finishing time and the mean cost value only take the successful navigation trials into account. 
From Table.\ref{table: Comparison of different methods in random corridor scenarios}, we conclude that U-OBCA has the highest navigation success rate since it considers the uncertainties in a more precise manner, while OBCA has the lowest success rate as it does not consider the uncertainties. Moreover, our method improves the navigation efficiency as it has shorter finishing time and lower cost value than other uncertainty-aware methods. The minimum distance to obstacles and the mean cost value of U-OBCA are also smaller than that of the other uncertainty-aware methods, which indicates that it avoids over-conservatism yet still attains a 97\% success rate. 
The trade-off of this method is that it needs more computation time compared to others. This is not surprising as we introduce the dual variables into the optimization and the expression of the collision avoidance constraints in U-OBCA is more complex than others.  Nevertheless, the average computation time of U-OBCA is about $0.11$s, which is promising to be implemented in a real-time re-planning framework.

\section{Experiments}

In this section, we present real-world experimental results in two representative scenarios—\emph{narrow corridor navigation} and \emph{parking}—to demonstrate the effectiveness of the proposed motion planning algorithm.

\subsection{Experimental Platform Setup}

All experiments are conducted on a smart wheelchair developed by our research group. The experimental platform is equipped with the following hardware components (see Fig.~\ref{fig:wheelchair}):

\begin{figure}
    \centering
    \includegraphics[width=1\linewidth]{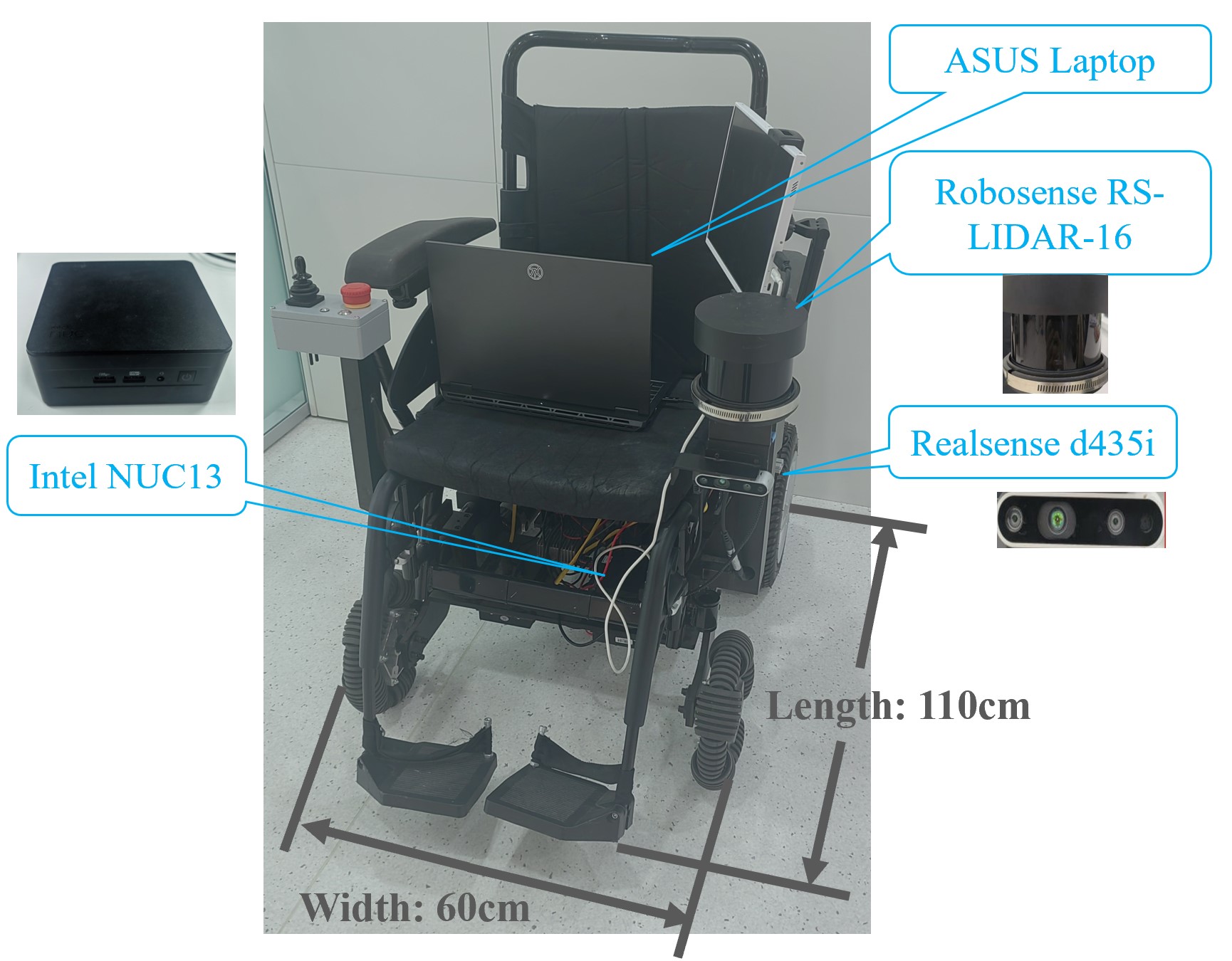}
    \caption{The smart wheelchair platform.}
    \label{fig:wheelchair}
\end{figure}

\begin{enumerate}
    \item \textbf{Main computation unit}: An Intel NUC~13 equipped with an i7-1370P CPU, 64~GB RAM, and a maximum clock frequency of 5.2~GHz.
    
    \item \textbf{Sensor data processing unit}:  
    Since the experiments involve dynamic obstacle avoidance (e.g., bicycles and pedestrians), an additional computing unit is deployed to execute the moving obstacle detection and trajectory prediction modules. More specifically, an ASUS TUF Gaming A16 laptop with an AMD Ryzen~7~H260 CPU, 16~GB RAM, a maximum clock frequency of 5.6~GHz, and an RTX~5060 GPU is used.  
    This unit is connected to the main computation unit via a Thunderbolt~4 interface.
    
    \item \textbf{RGB-D camera}: An Intel RealSense D435i with a resolution of $1920 \times 1080$, a frame rate of 30~FPS, and a Field of View (FoV) of $69.4^\circ \times 42.5^\circ$. The camera is also equipped with an onboard Bosch BMI055 IMU operating at 200~Hz.
    
    \item \textbf{LiDAR}: A Robosense RS-LiDAR-16 featuring a $360^\circ$ horizontal FoV, a $\pm 15^\circ$ vertical FoV, a frame rate of 20~Hz, and an angular resolution of $0.4^\circ \times 2.0^\circ$.
\end{enumerate}

The proposed algorithm is implemented in C++ using the Robot Operating System (ROS). The overall software architecture of the experimental platform is illustrated in Fig.~\ref{fig:software}.

\begin{figure}[!htpb]
    \centering
    \includegraphics[width=1\hsize]{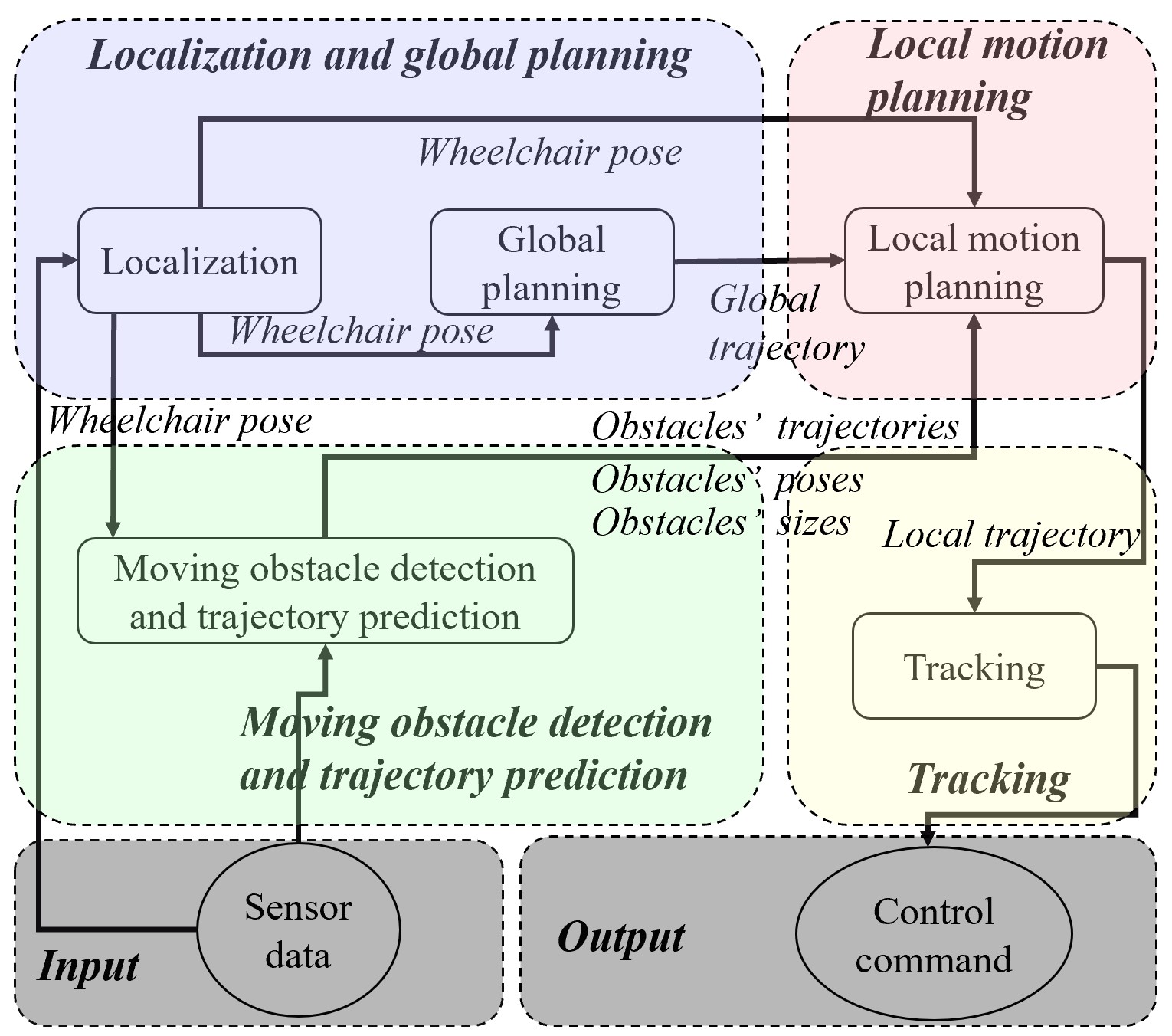}
    \caption{Software architecture of the experimental platform.}
    \label{fig:software}
\end{figure}

The system consists of the following software modules:

\begin{enumerate}
    \item \textbf{Localization and global planning module}:  
    This module runs on the main computation unit. It subscribes to sensor data from the LiDAR, RGB-D camera, IMU, and odometer to localize the wheelchair, and generates a global path using the A$^\star$ algorithm. The localization algorithm follows the method proposed in~\cite{localization}.  
    The estimated wheelchair's pose and the planned global path are published for downstream modules.

    \item \textbf{Moving obstacle detection and trajectory prediction module}:  
    This module subscribes to sensor data and detects dynamic obstacles (i.e., pedestrians and bicycles) using the algorithm proposed in~\cite{predicttrajectory}. Since this algorithm is GPU-intensive, it is deployed on the sensor data processing unit.  
    The module offers the obstacles' poses and predicts their trajectories in ego-vehicle frame over a horizon of 12~steps with a time interval of 0.2~s. Example detection and prediction results are shown in Fig.~\ref{fig:detection1.}. The module publishes obstacle information at a frequency of 7~Hz, including their sizes, poses, motion states, and predicted trajectories. Pedestrians are modeled as circular obstacles, while bicycles are represented as rectangular polygonal obstacles.

    \begin{figure}
        \centering
        \includegraphics[width=\hsize]{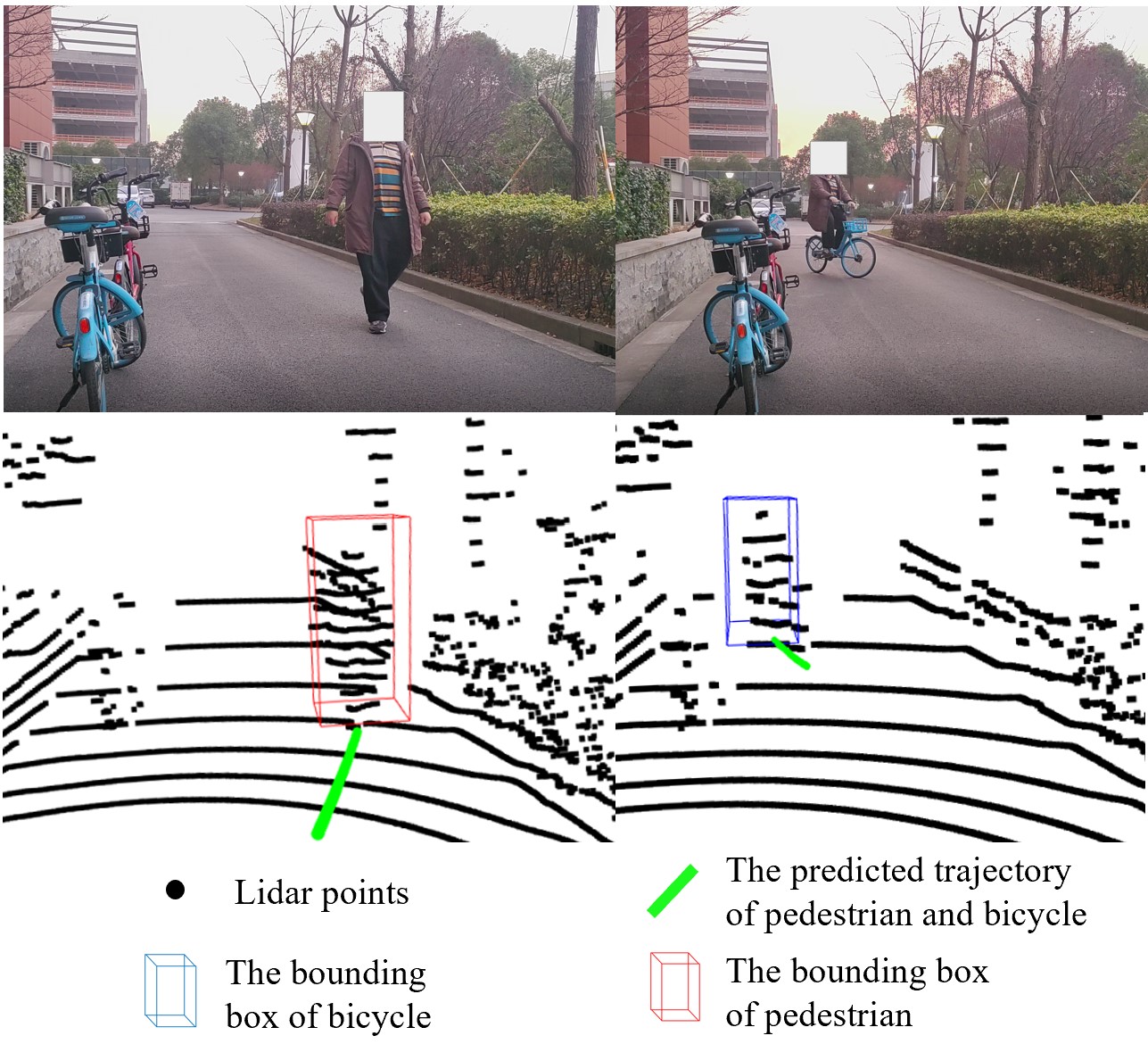}
        \caption{Detection and trajectory prediction results using the algorithm in~\cite{predicttrajectory}.}
        \label{fig:detection1.}
    \end{figure}

    \item \textbf{Local motion planning module}:  
    This module runs on the main computation unit and corresponds to the proposed U-OBCA planner. It subscribes to the global path, predicted trajectories of moving obstacles, and the current wheelchair pose, and generates a locally feasible trajectory in real time.

    \item \textbf{Tracking module}:  
    This module also runs on the main computation unit. It subscribes to the wheelchair pose and the planned local trajectory, and computes the corresponding linear and angular velocity commands.  
    The controller described in~\cite[Eq.~(3)]{gustavi2008stable} is employed to track the planned trajectory.
\end{enumerate}

\subsection{Estimating the Statistics of the Uncertainties}
\label{sec:uncertainty calibration}

The safety constraints derived in Theorem~\ref{theorem: round sufficient condition} and Theorem~\ref{theorem: polygon obstacle deterministic nonlinear constraints} rely on the statistical properties of the underlying random variables. In particular, the expectations and covariance matrices of the uncertainties are required to construct deterministic approximations of the chance constraints. 
Since the smart wheelchair operates on dry terrain at relatively low velocities, wheel slippage is negligible. Consequently, environmental disturbances are not explicitly considered in the uncertainty modeling. 
However, neither the moving obstacle detection and trajectory prediction module~\cite{predicttrajectory} nor the localization and global planning module~\cite{localization} provides explicit statistical information on these uncertainties. Therefore, we estimate the following quantities offline:
\begin{enumerate}
    \item the localization error of the smart wheelchair,
    \item the pose perception errors of obstacles,
    \item the predicted pose errors of moving obstacles.
\end{enumerate}

\subsubsection{Smart Wheelchair's Localization Error}
\label{sec:localization_calibration}

To estimate the expectation and variance of the smart wheelchair localization error, we temporarily equip the wheelchair with a high-precision RTK-GPS module to obtain ground-truth poses. The ground-truth and estimated poses are synchronized according to their timestamps, resulting in 200 localization error samples. The expectations and variances of the $x$-, $y$-, and $\theta$-components are then computed. The results are summarized in Table~\ref{table: Varience of wheelchair}.

From the table, the expectation of the localization error is observed to be close to zero. Therefore, consistent with Assumption~\ref{ass: zero_mean_and_known_covariance}, we set $\mathbb{E}([\bm{w}_{x,k}^v, \bm{w}_{y,k}^v, \bm{w}_{\theta,k}^v]^T)=\bm{0}_{3\times1}$. Since U-OBCA incorporates Wasserstein ambiguity sets, this zero-mean approximation does not compromise robustness.
Consequently, the covariance of the smart wheelchair localization error is set as $\cov([\bm{w}_{x,k}^v, \bm{w}_{y,k}^v, \bm{w}_{\theta,k}^v]^T)=\diag(7.28\times10^{-4}~\mathrm{m}^2,\,3.17\times10^{-4}~\mathrm{m}^2,\,5.89\times10^{-2}~\mathrm{deg}^2)$, $k\in\{1,\dots,N\}$.

\begin{table}[!htbp]
\caption{Estimated expectation and variance of the smart wheelchair's localization error}
\vspace{-10pt}
\label{table: Varience of wheelchair}
\rowcolors{2}{white}{gray!30}
\begin{center}
\begin{tabular}{c|c|c|c|c}
\rowcolor{blue!10}
& \textbf{Expectation} & \textbf{Variance} \\
\hline
$x$    & $-5.27\times 10^{-3}~\mathrm{m}$ & $7.28\times 10^{-4}~\mathrm{m}^2$\\
$y$    & $-3.22\times 10^{-3}~\mathrm{m}$ & $3.17\times 10^{-4}~\mathrm{m}^2$\\
$\theta$  & $3.70\times 10^{-2}~\mathrm{deg}$ & $5.89\times 10^{-2}~\mathrm{deg}^2$
\end{tabular}
\end{center}
\end{table}

\subsubsection{Pose Perception Error of Static Obstacles}
\label{sec:obstalce_calibration}

To identify the pose perception error of obstacles, we first fix the smart wheelchair at Point~A (illustrated in Fig.~\ref{fig:draw1.}), with its orientation aligned with the bicycle lane, corresponding to the experimental scenario shown in Fig.~\ref{fig:real1.}. The obstacles are then placed at eight predefined poses. At each pose, the obstacle poses are measured by the perception module, and the corresponding noise variances are estimated.
More specifically, the obstacles are placed in front of the wheelchair at distances of $0.5$~m, $1.0$~m, $1.5$~m, and $2.0$~m (Poses~1--4), with their orientations facing toward the wheelchair. Poses~5--8 share the same positions as Poses~1--4, respectively, but with the obstacle orientation rotated clockwise by $90^\circ$.

Recall that in Sec.~\ref{sec: Problem Statement}, obstacle poses are represented in the world frame. Therefore, the error statistics of both obstacle poses and their predicted trajectories must be estimated in the world frame. However, the perception module provides obstacle poses and predicted trajectories in the ego-vehicle frame. Consequently, localization errors of the smart wheelchair propagate into the obstacle pose uncertainties when transformed to the world frame.
To account for this effect, we approximate the distribution of the localization error obtained in Sec.~\ref{sec:localization_calibration} using histograms and randomly sample perturbation vectors from it. These perturbations are injected into the coordinate transformation from the ego frame to the world frame, generating world-frame obstacle pose samples. The expectations and variances of the resulting estimation errors are then computed from these samples.
On the other hand, since ground-truth obstacle poses are difficult to obtain, we assume that the obstacle detection and trajectory prediction module~\cite{predicttrajectory} provides unbiased pose estimates. Under this assumption, the expectations of both the static obstacle pose perception errors and the predicted pose errors of moving obstacles are set to zero, and the estimation procedure focuses solely on the covariance matrices.
Moreover, as stated in Assumption~\ref{ass: zero_mean_and_known_covariance}, the components are assumed to be mutually independent. Therefore, it suffices to estimate the variances of each component independently. For each pose, 100 samples of obstacle poses are collected, and the resulting variances are summarized in Table~\ref{table: Varience of obstacles}.

\begin{table*}[!htbp]
\caption{Estimated variances of static obstacle pose perception error}
\vspace{-10pt}
\label{table: Varience of obstacles}
\rowcolors{2}{white}{gray!30}
\begin{center}
\begin{tabular}{c|c|c|c|c|c|c|c|c}
\rowcolor{blue!10}
\textbf{Variance} & \textbf{Pose 1} & \textbf{Pose 2} & \textbf{Pose 3} & \textbf{Pose 4} & \textbf{Pose 5} & \textbf{Pose 6} & \textbf{Pose 7} & \textbf{Pose 8} \\
\hline
Round obstacle $x$ ($\times 10^{-4}~\mathrm{m}^2$)   & 10.37 & 8.75 & 9.38 & \textbf{11.33} & 11.05 & 9.29 & 8.93 & 10.02 \\
Round obstacle $y$ ($\times 10^{-4}~\mathrm{m}^2$)   & \textbf{5.22} & 4.70 & 4.23 & 4.78 & 4.95 & 4.44 & 4.02 & 4.68 \\
Polygon obstacle $x$ ($\times 10^{-4}~\mathrm{m}^2$) & 15.85 & 13.21 & 11.03 & 13.68 & 14.45 & 13.89 & 14.58 & \textbf{16.67}\\
Polygon obstacle $y$ ($\times 10^{-4}~\mathrm{m}^2$) & 4.07 & 4.19 & 4.22 & 4.72 & 5.11 & 4.80 & 5.02 & \textbf{5.78} \\
Polygon obstacle $\theta$ ($\times 10^{-2}~\mathrm{deg}^2$) & 10.22 & 8.31 & 11.67 & 10.38 & \textbf{15.92} & 11.28 & 9.73 & 10.69
\end{tabular}
\end{center}
\end{table*}
To ensure safety during navigation, the covariance matrices are constructed by selecting, for each state component, the largest variance observed among the eight poses. Specifically, for round-shaped obstacles, we set $\cov([\bm{w}^{o_j}_{x}, \bm{w}^{o_j}_{y}]^T)=\diag(11.33\times10^{-4}~\mathrm{m}^2,\,5.22\times10^{-4}~\mathrm{m}^2)$. For polygon-shaped obstacles, we set $\cov([\bm{w}^{o_j}_{x}, \bm{w}^{o_j}_{y}, \bm{w}^{o_j}_{\theta}]^T)=\diag(16.67\times10^{-4}~\mathrm{m}^2,\,5.78\times10^{-4}~\mathrm{m}^2,\,15.92\times10^{-2}~\mathrm{deg}^2)$.

\subsubsection{Predicted Pose Error of Moving Obstacles}

The smart wheelchair is fixed as in Sec.~\ref{sec:obstalce_calibration}, while the moving obstacles (pedestrians and bicycles) repeatedly follow a predefined trajectory passing consecutively from Pose~4 to Pose~1. This process is repeated 20 times.
Each time an obstacle reaches a milestone pose, its predicted future poses are recorded. For each prediction step, the variance of the predicted pose errors in the world frame is computed using the same perturbation-based procedure described in Sec.~\ref{sec:obstalce_calibration}. The results are illustrated in Fig.~\ref{fig:rand_walk.}.

In particular, the variances of the $x$- and $y$-components exhibit approximately linear growth with respect to the prediction step. Therefore, the slope of the fitted line is adopted as the prediction error growth rate.
In contrast, as shown in Fig.~\ref{fig:rand_walk.}, the $\theta$-variance of the polygon-shaped obstacle does not exhibit significant growth as the prediction horizon increases. To ensure safety, the largest observed value is therefore adopted as the variance of the $\theta$-component.
\begin{figure}[!htpb]
    \centering
    \includegraphics[width=\hsize]{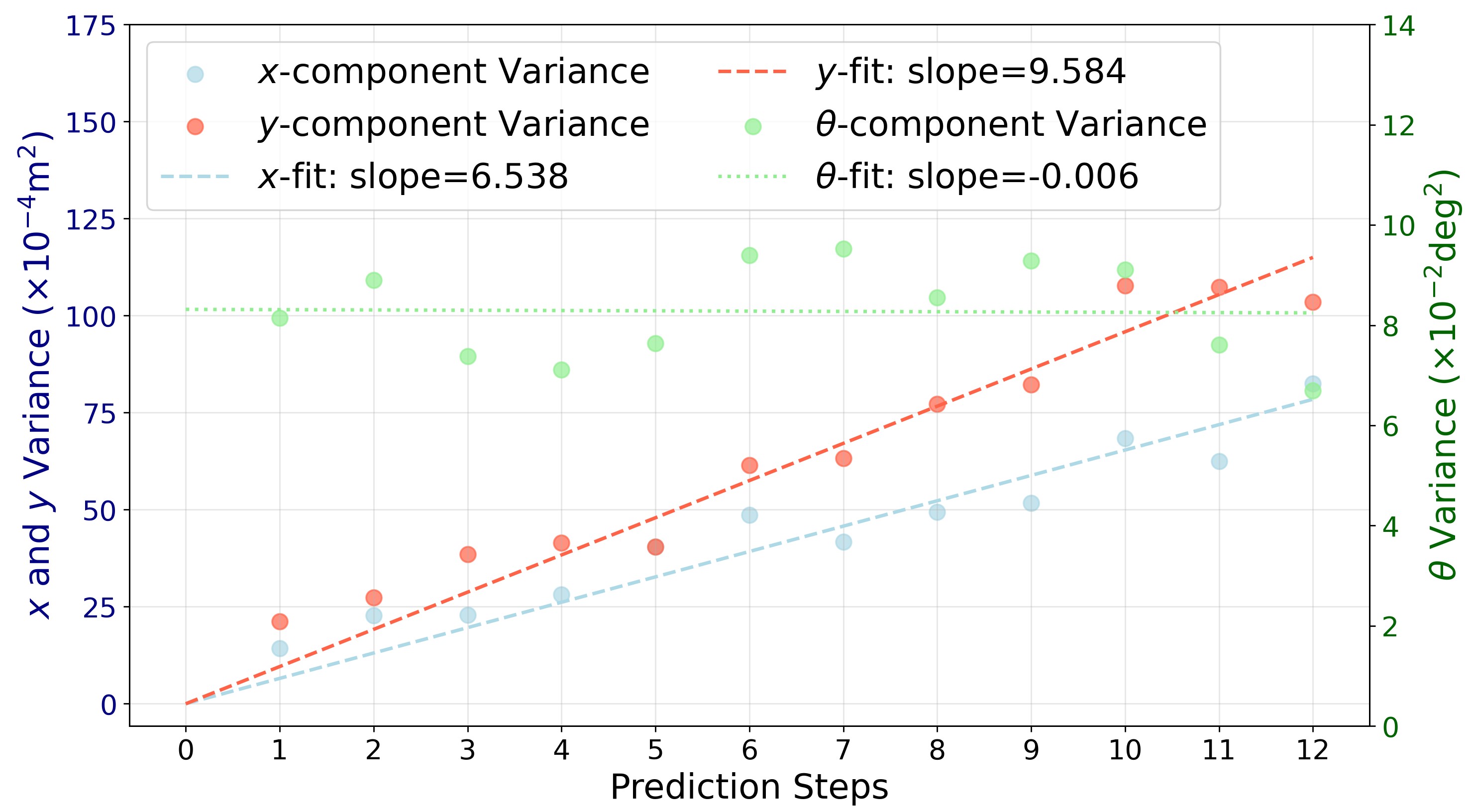}
    \caption{An example of predicted pose error for a polygon-shaped obstacle at pose~3.}
    \label{fig:rand_walk.}
\end{figure}
Consequently, the estimated variances at different poses are summarized in Table~\ref{table: Varience of prediction}.
\begin{table}[!htbp]
\caption{Estimated variances of moving obstacle predicted pose error}
\vspace{-10pt}
\label{table: Varience of prediction}
\rowcolors{2}{white}{gray!30}
\begin{center}
\begin{tabular}{c|c|c|c|c}
\rowcolor{blue!10}
\textbf{Variance} & \textbf{Pose 1} & \textbf{Pose 2} & \textbf{Pose 3} & \textbf{Pose 4} \\
\hline
$x$ (Round) ($\times k \times 10^{-4}~\mathrm{m}^2$) & \textbf{6.72} & 6.37 & 5.63 & 5.45 \\
$y$ (Round) ($\times k \times 10^{-4}~\mathrm{m}^2$) & 8.28 & \textbf{9.17} & 8.35 & 8.97 \\
$x$ (Polygon) ($\times k \times 10^{-4}~\mathrm{m}^2$) & 6.92 & 6.60 & 6.54 & \textbf{7.20} \\
$y$ (Polygon) ($\times k \times 10^{-4}~\mathrm{m}^2$) & \textbf{12.32} & 11.64 & 9.58 & 10.62 \\
$\theta$ (Polygon) ($\times 10^{-2}~\mathrm{deg}^2$) & \textbf{10.17} & 8.37 & 9.52 & 9.21
\end{tabular}
\end{center}
\end{table}

When constructing the covariance matrices of the predicted pose errors, the largest value across different poses is selected for each component. Specifically, for moving round obstacles, we set $\cov([\bm{w}^{o_j}_{x,k}, \bm{w}^{o_j}_{y,k}]^T)=\diag(6.72\times k\times10^{-4}~\mathrm{m}^2,\,9.17\times k\times10^{-4}~\mathrm{m}^2)$, $k\in\{1,\dots,N\}$. For moving polygon-shaped obstacles, we set $\cov([\bm{w}^{o_j}_{x,k}, \bm{w}^{o_j}_{y,k}, \bm{w}^{o_j}_{\theta,k}]^T)=\diag(7.20\times k\times10^{-4}~\mathrm{m}^2,\,12.32\times k\times10^{-4}~\mathrm{m}^2,\,10.17\times10^{-2}~\mathrm{deg}^2)$, $k\in\{1,\dots,N\}$, where $k$ denotes the prediction step.

\begin{figure}[!htbp]
\centering
\subfigure[Real world scenario of narrow corridor.]{\includegraphics[width=1\hsize]{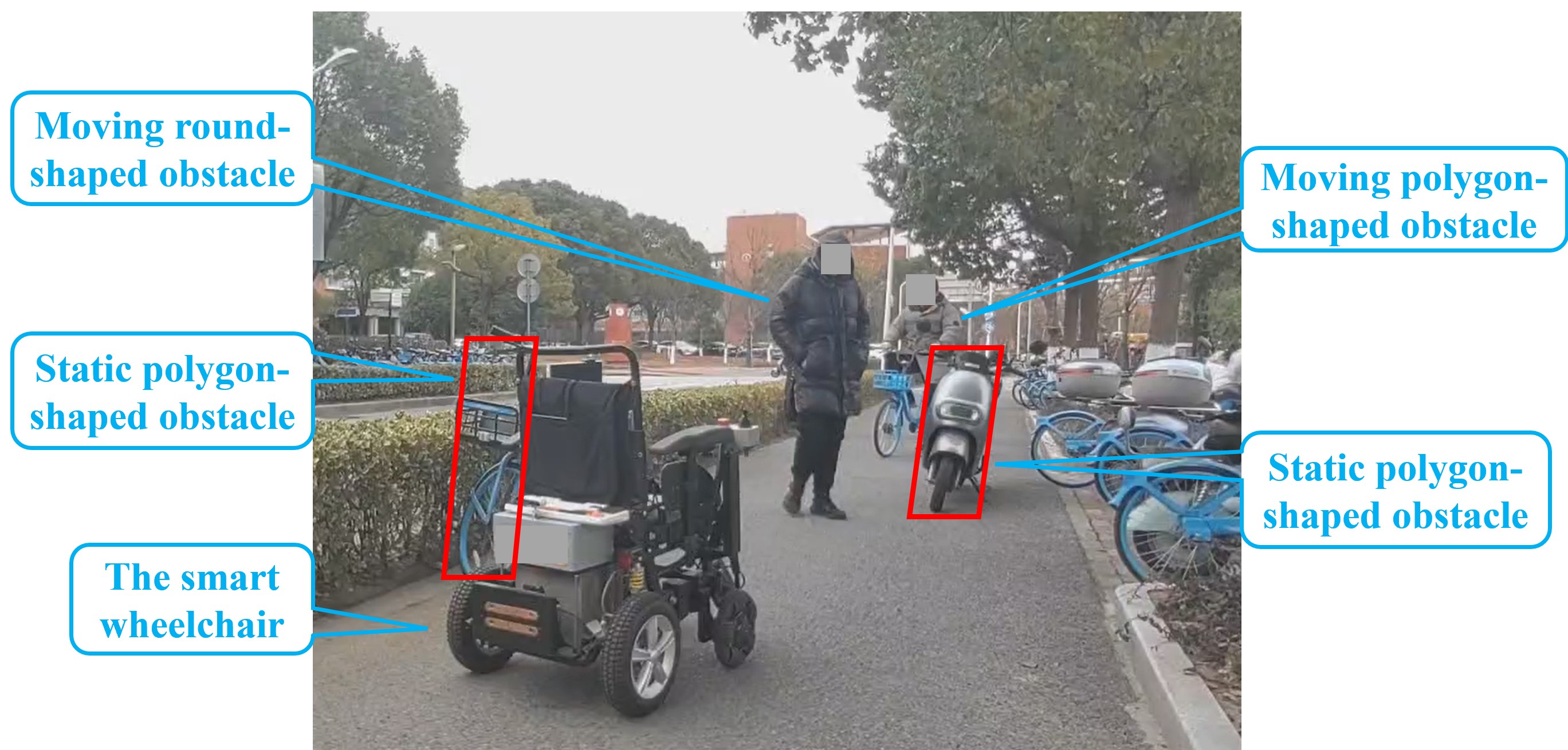}
\label{fig:real1.}}
\subfigure[The approximate trajectories and target points of bicycle, pedestrian and wheelchair.]
{\includegraphics[width=1\hsize]{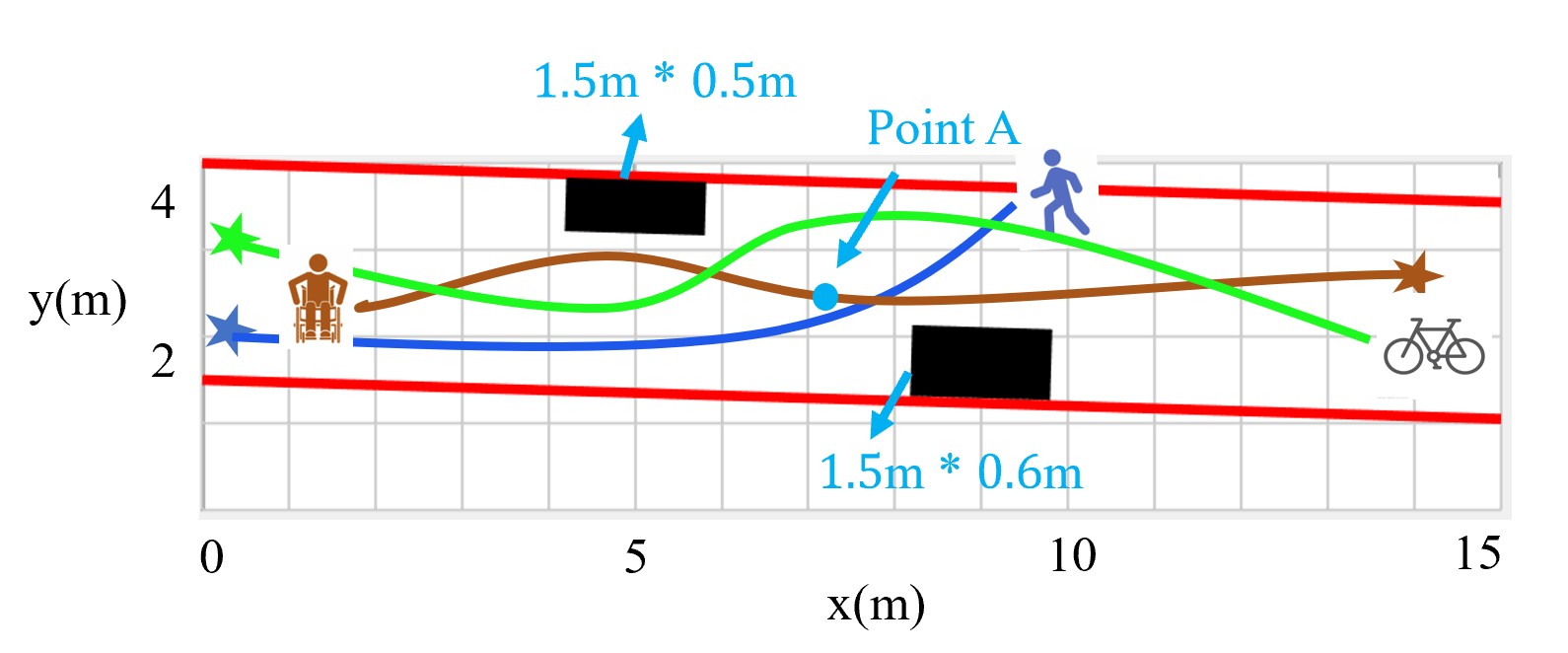}
\label{fig:draw1.}}
\caption{The experiment scenario of narrow corridor.}
\label{fig:scenario1.}
\end{figure}

\begin{figure}[!htbp]
\centering
\subfigure[The finishing time of different methods.]{\includegraphics[width=1\hsize]{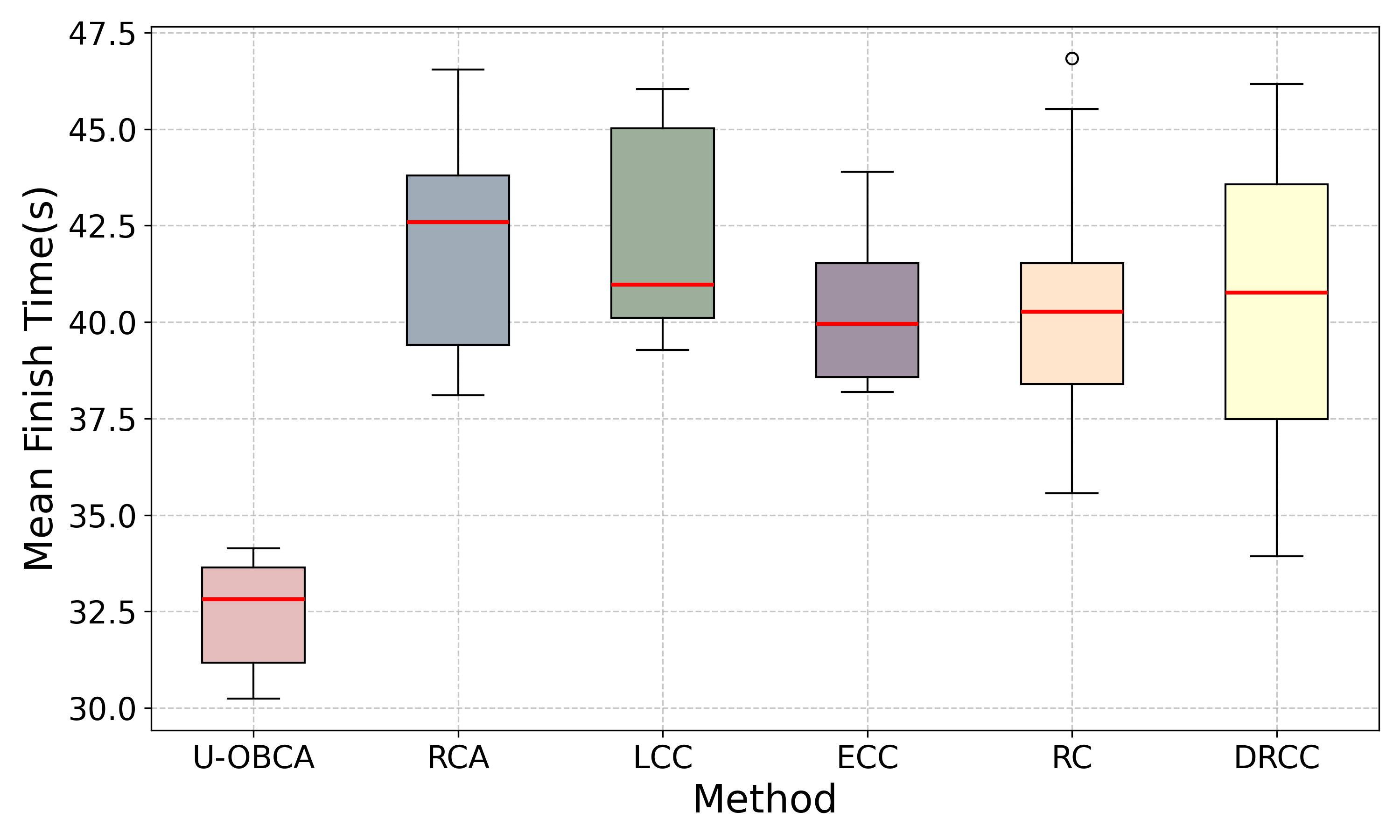}
\label{fig:compare1_FT.}}
\subfigure[The minium distance to obstacles of different methods.]
{\includegraphics[width=1\hsize]{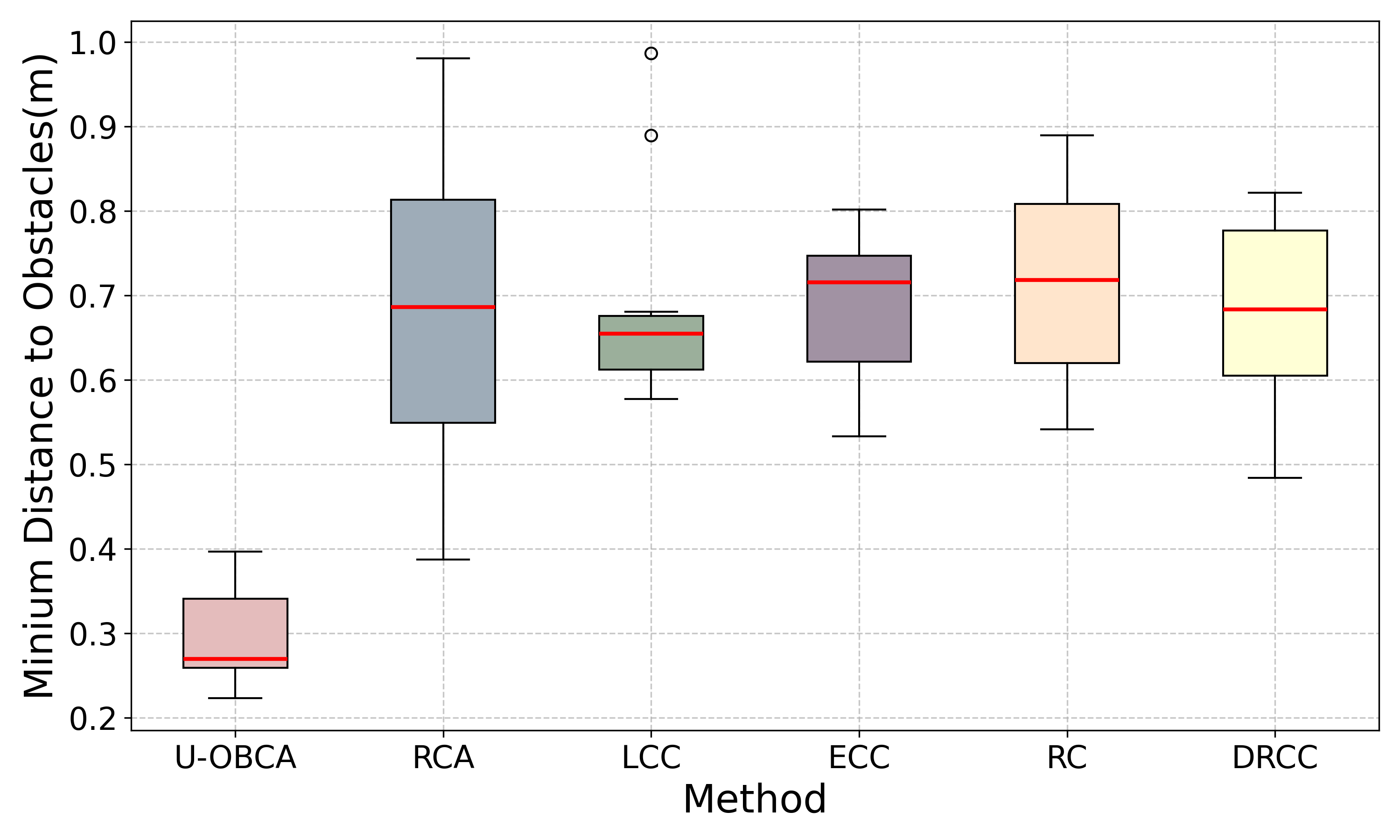}
\label{fig:compare1_MD.}}
\caption{Comparison of the navigation finishing time and the minium distance to obstacles between methods.}
\label{fig:comparison}
\end{figure}

\begin{figure}[!htbp]
\centering
\subfigure[Real-world scenario of the parking experiment.]{\includegraphics[width=1\hsize]{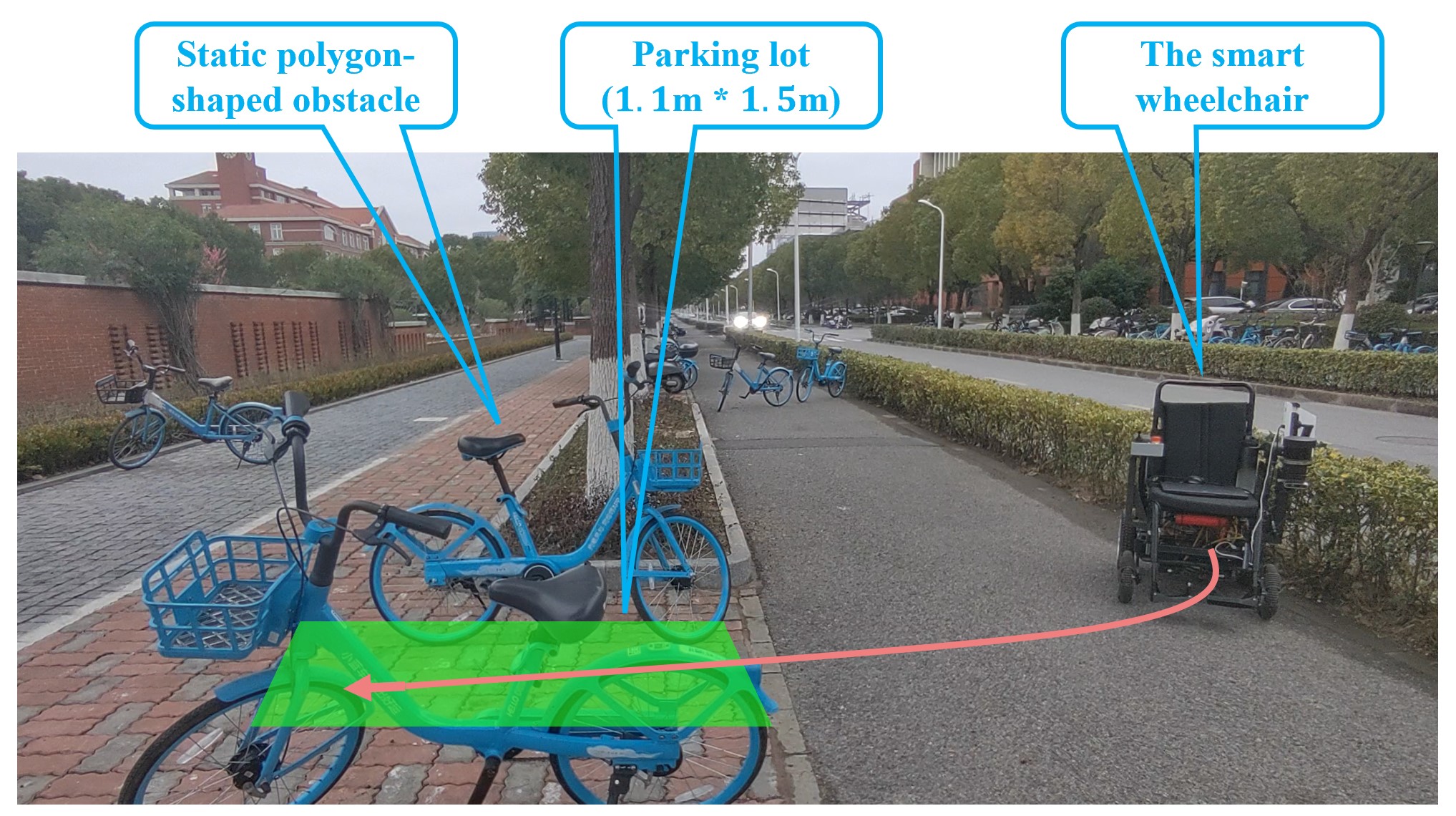}
\label{fig:real2.}}
\subfigure[The approximate trajectory and target point of wheelchair.]
{\includegraphics[width=1\hsize]{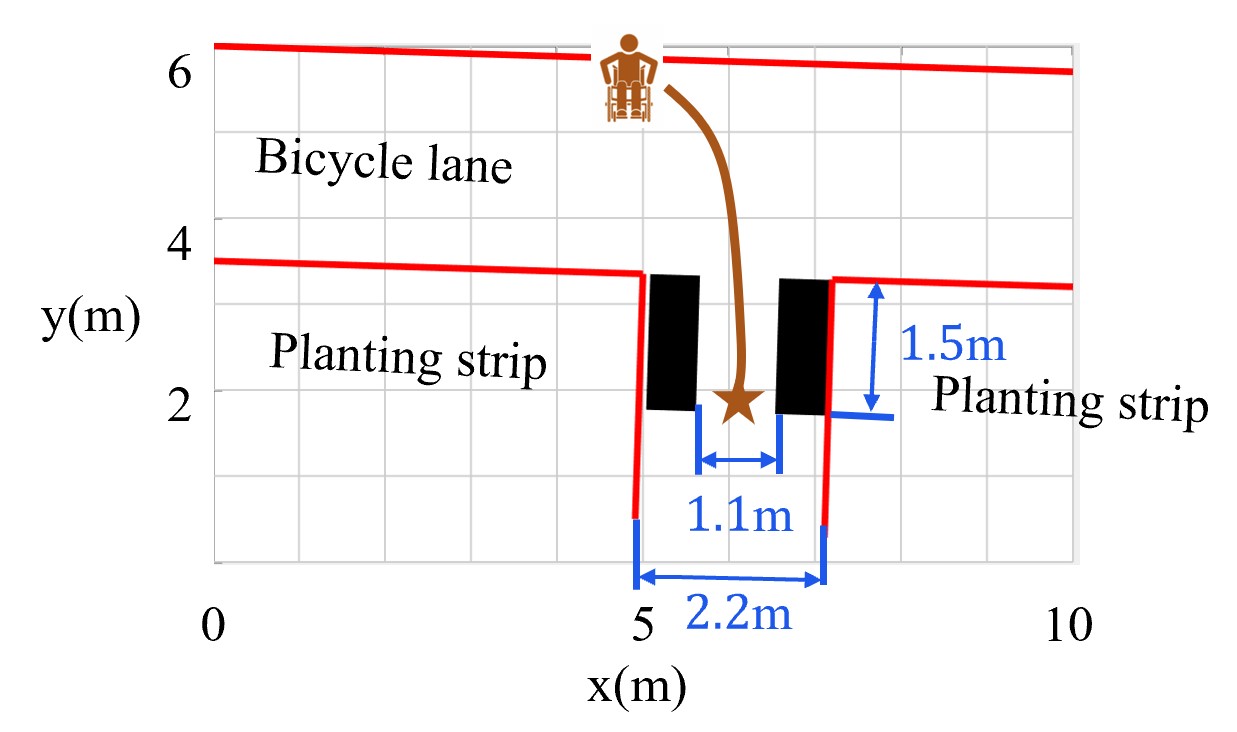}
\label{fig:draw2.}}
\subfigure[The front view of the parking lot.]{\includegraphics[width=0.8\hsize]{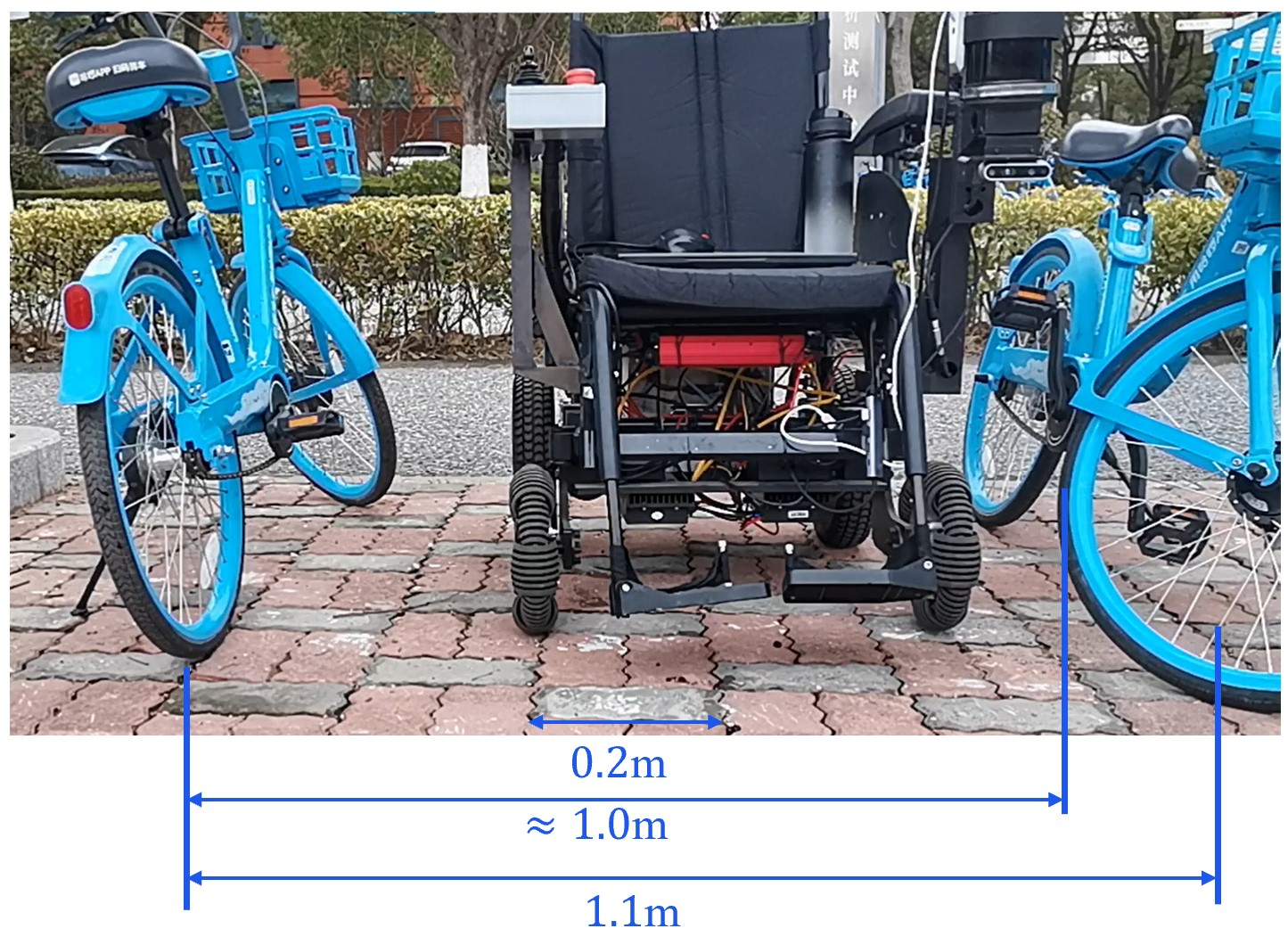}
\label{fig:real3.}}
\caption{The experiment scenario of the parking experiment.}
\label{fig:scenario2.}
\end{figure}

\begin{table*}[!htbp]
\caption{Comparison of different methods in real world narrow corridor navigation scenarios}
\vspace{-10pt}
\label{table: Comparison of different methods in real world narrow corridor scenarios}
\rowcolors{2}{white}{gray!30}
\begin{center}
\begin{tabular}{c|c|c|c|c|c}
\rowcolor{blue!10}
\textbf{Methods}  & \textbf{Finishing Time (s)} & \textbf{Mean Cost} & \textbf{Mean Computation Time (ms)} & \textbf{Mean Distance (m)} & \textbf{Min Distance (m)}   \\
\hline
U-OBCA                          & \textbf{32.44} & \textbf{267.08} & 557.98          & 1.37         & 0.30 \\
RCA \cite{RCA}                  & 42.23          & 324.32          & \textbf{427.75} & 1.73          & 0.69   \\
LCC \cite{zhu2019chance}        & 42.16          & 337.09          & 435.62          & 1.70          & 0.69  \\
ECC \cite{Castillo2020}         & 40.33          & 302.67          & 467.88          & 1.72          & 0.68  \\
RC \cite{jasour2023convex}      & 40.43          & 311.16          & 493.28          & 1.95 & 0.72 \\
DRCC \cite{ryu2024integrating}  & 40.35          & 296.08          & 478.02          & 1.83          & 0.67
\end{tabular}
\end{center}
\end{table*}

\subsection{Navigation in a Narrow Corridor}\label{sec:nav_corridor_exp}
\subsubsection{The Experiment Scenario}
The experimental field is a bicycle lane with a width of $2.5$ m, which serves as a narrow corridor. An electric scooter is parked on the road and is viewed as a static obstacle with a nominal size of $0.6$ m$\times 1.5$ m, and a bicycle is parked and viewed as another static obstacle with a nominal size of $0.5$ m $\times$ $1.5$ m; neither of these static obstacles is pre-built into the navigation map. In the experiment, the smart wheelchair is required to navigate toward the goal while avoiding collisions with two static obstacles, a pedestrian, and a moving bicycle. The experimental scenario is illustrated in Fig.~\ref{fig:scenario1.}.

\subsubsection{The Experiment Settings}
We set the sampling period $\Delta t = 0.2$s and choose the number of prediction steps as $N = 12$. Moreover, the risk levels $\alpha_1,\alpha_2,\alpha_3$ are set to $0.002, 0.002$ and $0.006$, respectively. We further set the Wasserstein radius $\theta_w = \theta_{w_1} = \theta_{w_2} = \theta_{w_3} = 0.001$ and adopt the estimated uncertainty statistics in Sec.\ref{sec:uncertainty calibration} in the experiments.
We also use IPOPT \cite{IPOPT} for the numerical solver, and we set the solver as ``mumps" \cite{mumps} in CasADi \cite{casadi}.
The state and control inputs of wheelchair are kept the same as those in Sec.~\ref{sec:sim_corridor_nav} (see \eqref{eq: nominal vehicle kinematics corridor scenario}).
Moreover, to achieve a better navigation performance, we introduce a global path planned by A$^\star$ into the cost function $J$ as a reference trajectory.
In particular, the reference trajectory shall have the same number of steps and the time interval as the predicted local trajectory. And consequently, the modified $J$ take the form 
\begin{equation*}
\label{eq: cost function real world}
\begin{aligned}
J(\bar{s}, u)=  \sum_{k=1}^{N}\|\bar{s}^v_k-\bar{s}^r_k\|_Q^2 + \sum_{k=0}^{N-1}\|u_k\|_R^2,
\end{aligned}
\end{equation*}
where $\bar{s}^r_k$ is the state of the $k$-th point in reference trajectory, $Q$ and $R$ are positive definite weighting matrices. Here $Q=\diag(0.5, 0.5, 0.7)$, and $R=\diag(0.1, 0.1)$.
To this end, the following steps are used to turn the global path planned by A$^\star$ algorithm into a reference trajectory for navigation:
\begin{enumerate}
    \item Orientation setting: For each path point, the orientation is set to point toward the position of the subsequent path point.
    \item Path length computing: For each path point, we compute its Euclidean distance to the next point, then we accumulate the distance to obtain the path length from the start point to every path point. 
    \item Path resampling: We then resample the path with a fixed interval distance of $0.1$m, and timestamp the first $N=12$ resampled points with $\Delta t = 0.2$s time interval. These points are then used as the reference trajectory $\bar{s}^r_k$ in $J$.
    \item Pose interpolation: Since some of the resampled trajectory points do not necessarily coincide with the global path points, such pose interpolation is still needed. In particular, for the trajectory points that locates between two global path points, we perform linear interpolation to obtain its position and Spherical Linear intERPolation (SLERP) \cite{SLERP} to obtain its orientation.
\end{enumerate}

\subsubsection{The Experiment Results}\label{sec:nav_corr_result}
We choose RCA \cite{RCA}, LCC \cite{zhu2019chance}, ECC \cite{Castillo2020}, RC \cite{jasour2023convex}, and DRCC \cite{ryu2024integrating} as the comparison baselines, just as it is in Sec.~\ref{sec:sim_corridor_nav}.
To guarantee a fair comparison, all of the above methods run under the same time-horizon length, weighting matrices, risk level and noise distributions. 
Moreover, we repeat the same experiment $10$ times for each method, and compare the Success Rate, Mean Finishing Time, Mean Cost Value, Minium and Mean Distance to Obstacles and Mean Computation Time.
All methods successfully accomplish the navigation task each time and the results are illustrated in Table.~\ref{table: Comparison of different methods in real world narrow corridor scenarios}.

From Table.~\ref{table: Comparison of different methods in real world narrow corridor scenarios}, we conclude that our proposed U-OBCA has the lowest minimum distance to obstacles yet still guarantees a full success rate. This indicates that it has the lowest safety over-conservatism. Therefore, it is not surprising that U-OBCA has the lowest finishing time and the lowest mean cost value. Nevertheless, U-OBCA has the highest computation time, which is the price to pay since we introduce dual variables into the problem. 
In addition, we notice that the computation time of real-world experiment is much higher than those in the simulations. This is because the localization and tracking module is necessary for real-world experiments and they both run on main computation unit, which cause extra burden (especially the localization module) and slow the planning speed.
Moreover, to further illustrate the performance advantage of our method, the box plots of the navigations' finishing time as well as the minimum distance to the obstacles across the methods are depicted in Fig.~\ref{fig:comparison}.

\subsection{Parking Experiment}
\subsubsection{The Experiment Scenario}
The experimental field is a $2.2$m-wide bicycle parking space between the planting strips, where two bicycles are already parked, leaving an $1.1$m-wide free space between them. 
In view of the bicycles' front wheel (as shown in Fig.~\ref{fig:real3.}), the minium width of the free space is about $1.0$m. In the experiment, two static rectangular obstacles are introduced to represent the parked bicycles, forming the boundaries of a $1.1$ m $\times$ $1.5$ m parking space. Each static rectangular obstacle has the size of $0.5$ m $\times$ $1.5$ m. 
The objective is to park our $0.6$ m $\times$ $1.1$ m smart wheelchair into the parking space while avoiding collisions with the static obstacles. 
The experimental scenario is shown in Fig.~\ref{fig:scenario2.}.

\subsubsection{The Experiment Settings}
All the experiment settings are the same as those in Sec.\ref{sec:nav_corridor_exp}.

\subsubsection{The Experiment Results}
We choose the same comparison baselines as those in Sec.~\ref{sec:nav_corridor_exp}. All the methods also run under the same receding horizon setting, weighting matrices, risk level and noise distributions so as to have a fair comparison.
Similar to Sec.~\ref{sec:nav_corridor_exp}, we repeat the experiment $10$ times for each methods. U-OBCA exhibits full success rate in the parking task, while none of the baselines successfully complete the task. 
This is reasonable because all other methods take elliptical approximation for the polygon-shaped static obstacles, which make the tiny parking space unreachable for the wheelchair. 
We hence only list the mean finishing time, minimum distance to obstacles and mean computation time of U-OBCA in Table.~\ref{table: Comparison of different methods in real world reverse parking scenarios} and do not depict the box plots of the finishing time and the minimum distance to obstacles as we do in Sec.\ref{sec:nav_corridor_exp}. 

\begin{table}[!htbp]
\caption{U-OBCA experiment results in real world parking scenarios}
\vspace{-10pt}
\label{table: Comparison of different methods in real world reverse parking scenarios}
\begin{center}
\begin{tabular}{c|c|c}
\rowcolor{blue!10}
 \textbf{Finishing Time (s)} & \textbf{Mean Computation Time (ms)} & \textbf{Min distance (m)}   \\
\hline
 7.96& 400.13& 0.17
\end{tabular}
\end{center}
\end{table}

\section{Conclusion}
In this paper, we extend the Optimization-Based Collision Avoidance (OBCA) method to Uncertainty-aware OBCA (U-OBCA) for vehicle trajectory planning in
uncertain narrow environments. More precisely, we first use the chance constraints to tightly bound the collision risks between the polygons. Moreover, we prove that the chance constraints can be further relaxed into deterministic nonlinear constraints, which can be efficiently solved. The numerical simulation and the real-world experiment results show the proposed method avoids the over-conservatism in trajectory planning and improves the navigation efficiency. 

\appendix

\subsection{Computation of the expectation and the covariance matrix in Theorem~\ref{theorem: round sufficient condition}}
\label{app: expectation and covariance in round condition}

To compute the expectation and covariance of $\bm{p}^j_k=[\delta \bm x^j_k,\delta \bm y^j_k]^T$, recall that $\bm{p}^j_k=\bar{A}^v R^{-1}(\bm\theta^v_k)\delta\bm t^j_k+\bar{b}^v$, we obtain
\begin{equation*}
\label{eq: detailed p^j_k}
\begin{aligned}
R^{-1}(\bm\theta^v_k)\delta\bm t^j_k
\!=\!
\begin{bmatrix}
\;\;\delta\bm x^j_k\cos\bm\theta^v_k\!+\!\delta\bm y^j_k\sin\bm\theta^v_k\\
-\delta\bm x^j_k\sin\bm\theta^v_k\!+\!\delta\bm y^j_k\cos\bm\theta^v_k
\end{bmatrix}
\!=:\! \begin{bmatrix}\bm{\eta}_{j,k,1}\\ \bm{\eta}_{j,k,2}\end{bmatrix}.
\end{aligned}
\end{equation*}
Hence,
\begin{equation}
\label{eq: E, cov p^j_k}
\begin{aligned}
\mE[\bm{p}^j_k]
&= \bar{A}^v\mE\Big[\begin{bmatrix}\bm{\eta}_{j,k,1}& \bm{\eta}_{j,k,2}\end{bmatrix}^T\Big]+\bar{b}^v,\\
\Sigma_{\bm{p}^j_k}
&= \bar{A}^v\ \Sigma\!\left(\begin{bmatrix}\bm{\eta}_{j,k,1}& \bm{\eta}_{j,k,2}\end{bmatrix}^T\right)\ \bar{A}^{v,T}.
\end{aligned}
\end{equation}
Note $\delta\bm x^j_k=\bm x_k^v-\bm x_k^{o_j}=\bar{x}_k^v+\bm w_k^v-\bar{x}_k^{o_j}-\bm w_k^{o_j}=\delta\bar{x}_k^j+\delta\bm w_{x,k}^j$. Similarly, 
$\delta\bm y^j_k=\delta\bar{y}^j_k+\delta\bm{w}_{y,k}^j$, and recall that
$\bm\theta^v_k=\bar{\theta}^v_k+\bm{w}_{\theta,k}^v$.
By Assumption \ref{ass: known_sin_cos}, it holds that $\mE[\sin(\bm w_{\theta,k}^v)]=0$, $\mE[\delta\bm w_{x,k}^j]=\mE[\bm w_{x,k}^v]-\mE[\bm w_{x,k}^{o_j}]=0$ and $\mE[\delta\bm w_{y,k}^j]=\mE[\bm w_{y,k}^v]-\mE[\bm w_{y,k}^{o_j}]=0$. 
Moreover, by Assumption \ref{ass: zero_mean_and_known_covariance}, the noises in positions and orientations are independent, it follows that
\begin{equation}
    \label{eq: E eta_j,k,1}
    \begin{aligned}
    &\mE[\bm{\eta}_{j,k,1}]= \mE[\delta\bm x^j_k\cos\bm\theta^v_k\!+\!\delta\bm y^j_k\sin\bm\theta^v_k] \\
    &= \mE[(\delta\bar{x}_k^j+\delta\bm w_{x,k}^)\cos\bm \theta_k^v)]+\mE[(\delta\bar{y}_k^j+\delta\bm w_{y,k}^j)\sin\bm \theta_k^v]\\
    &=\delta\bar{x}_k^j\mE[\cos\bm\theta_k^v]+\underbrace{\mE[\delta\bm w_{x,k}^j]}_{=0}\mE[\cos\bm\theta_k^v]\\
    &\quad +\delta\bar{y}_k^j\mE[\sin\bm\theta_k^v]+\underbrace{\mE[\delta\bm w_{y,k}^j]}_{=0}\mE[\sin\bm\theta_k^v]\\
    &=\delta\bar{x}_k^j\mE[\cos\bm\theta_k^v]+\delta\bar{y}_k^j\mE[\sin\bm\theta_k^v].
    \end{aligned}
\end{equation}
It further holds that
\begin{equation}
\label{eq: E eta_j,k,1 more}
    \begin{aligned}
   &\mE[\bm{\eta}_{j,k,1}]= \delta\bar{x}_k^j\mE[\cos(\bar{\theta}_{k}^v+\bm w_{\theta,k}^v)]+\delta\bar{y}_k^j\mE[\sin(\bar{\theta}_{k}^v+\bm w_{\theta,k}^v)]\nonumber\\
   &=\delta\bar{x}_k^j\cos\bar{\theta}_{k}^v\mE[\cos\bm w_{\theta,k}^v]-\delta\bar{x}_k^j\cos\bar{\theta}_{k}^v\underbrace{\mE[\sin\bm w_{\theta,k}^v]}_{=0}\nonumber\\
   &\quad +\delta\bar{y}_k^j\sin\bar{\theta}_{k}^v\mE[\cos\bm w_{\theta,k}^v]+\delta\bar{y}_k^j\cos\bar{\theta}_{k}^v\underbrace{\mE[\sin\bm w_{\theta,k}^v]}_{=0}\nonumber\\
    &=(\delta\bar{x}_k^j\cos\bar{\theta}_{k}^v+\delta\bar{y}_k^j\sin\bar{\theta}_{k}^v)\mE[\cos\bm w_{\theta,k}^v], 
    \end{aligned}
\end{equation}
and similarly, it holds that
\begin{equation}
\label{eq: E eta_j,k,2}
\begin{aligned}
\mE[\bm{\eta}_{j,k,2}]
\!=\! (-\delta\bar{x}^j_k \sin\bar\theta^v_k \!+\! \delta\bar{y}^j_k \cos\bar\theta^v_k )\mE[\cos\bm{w}_{\theta,k}^v].
\end{aligned}
\end{equation}
Now consider the $6$-dimensional random vector
\begin{equation}
\label{eq: w_s,j,k}
\bm{w}_{s,j,k}:=
\begin{bmatrix}
\cos\bm{w}_{\theta,k}^v\\
\sin\bm{w}_{\theta,k}^v\\
\delta\bm{w}_{x,k}^j\cos\bm{w}_{\theta,k}^v\\
\delta\bm{w}_{y,k}^j\cos\bm{w}_{\theta,k}^v\\
\delta\bm{w}_{x,k}^j\sin\bm{w}_{\theta,k}^v\\
\delta\bm{w}_{y,k}^j\sin\bm{w}_{\theta,k}^v
\end{bmatrix}
\end{equation}
whose covariance $\Sigma_{\bm{w}_{s,j,k}}:=\mathrm{Cov}(\bm{w}_{s,j,k})$ is known a priori since it can be computed off-line from the definition by Assumption~\ref{ass: known_sin_cos}.
On the other hand, note that $\bm{\eta}_{j,k,1}$ and $\bm{\eta}_{j,k,2}$ can be written:
\begin{equation}
\label{eq: rewrite eta}
\begin{aligned}   
    &\bm{\eta}_{j,k,1}=V_{\bm{\eta}_{j,k,1}}^T \bm{w}_{s,j,k},\\
    & \bm{\eta}_{j,k,2}=V_{\bm{\eta}_{j,k,2}}^T\bm{w}_{s,j,k},
\end{aligned}
\end{equation}
where
\begin{equation}
\label{eq: V eta}
\begin{aligned}
V_{\bm{\eta}_{j,k,1}}
=\big[&\delta\bar{x}^j_k\cos\bar\theta^v_k+\delta\bar{y}^j_k\sin\bar\theta^v_k,\\
-&\delta\bar{x}^j_k\sin\bar\theta^v_k+\delta\bar{y}^j_k\cos\bar\theta^v_k,\\
&\cos\bar\theta^v_k,\ \sin\bar{\theta}^v_k,-\sin\bar{\theta}^v_k,\ \cos\bar{\theta}^v_k \big]^T,\\
V_{\bm{\eta}_{j,k,2}}
=\big[&-\delta\bar{x}^j_k\sin\bar\theta^v_k+\delta\bar{y}^j_k\cos\bar{\theta}^v_k,\\
&-\delta\bar{x}^j_k\cos\bar{\theta}^v_k -\delta\bar{y}^j_k\sin\bar{\theta}^v_k,\\
&-\sin\bar{\theta}^v_k,\ \cos\bar{\theta}^v_k,\ -\cos\bar{\theta}^v_k,\ -\sin\bar\theta^v_k\ \big]^T.
\end{aligned}
\end{equation}
Hence
\begin{equation}
\label{eq: sigma eta}
\begin{aligned}
\sigma^2(\bm{\eta}_{j,k,1}) &= V_{\bm{\eta}_{j,k,1}}^{\!T}\Sigma_{\bm{w}_{s,j,k}}V_{\bm{\eta}_{j,k,1}},\\
\sigma^2(\bm{\eta}_{j,k,2}) &= V_{\bm{\eta}_{j,k,2}}^{\!T}\Sigma_{\bm{w}_{s,j,k}}V_{\bm{\eta}_{j,k,2}},\\
\mathrm{Cov}(\bm{\eta}_{j,k,1},\bm{\eta}_{j,k,2})
&= V_{\bm{\eta}_{j,k,1}}^{\!T}\Sigma_{\bm{w}_{s,j,k}}V_{\bm{\eta}_{j,k,2}},
\end{aligned}
\end{equation}
and finally substitute \eqref{eq: E eta_j,k,1}–\eqref{eq: sigma eta} into \eqref{eq: E, cov p^j_k}.

\subsection{Computation of the expectation and the covariance matrix in Theorem~\ref{theorem: polygon obstacle deterministic nonlinear constraints}} 
\label{{app: expectation and covariance in polygon condition}}
In this section, we compute the expectation and covariance matrix in Theorem~\ref{theorem: polygon obstacle deterministic nonlinear constraints} in detail. Similar to Appendix~\ref{app: expectation and covariance in round condition}, we first decompose the variables in to the expression of the vehicle's and the obstacles' poses $(\bm{x}^v_k, \bm{y}^v_k, \bm{\theta}^v_k)$,  $(\bm{x}^{o_j}_k\bm{y}^{o_j}_k, \bm{\theta}^{o_j}_k)$, and then compute the expectations and the covariances based on the expectations and the covariances of the noises' functions. Here we only give the final expression for brevity. More precisely, 
\begin{equation*}
\label{eq: E cov q^j1_k, q^j2_k}
\begin{aligned}
\mE[\bm{q}^{j,1}_k] &= \bar{A}^{o_j}\mE\Big[\begin{bmatrix}
\quad \cos\delta\bm{\theta}^j_k & \sin\delta\bm{\theta}^j_k 
\end{bmatrix}^T \Big],\\
\mE[\bm{q}^{j,2}_k] &= \bar{A}^{o_j}\mE\Big[\begin{bmatrix}
-\sin\delta\bm{\theta}^j_k & \cos\delta\bm{\theta}^j_k 
\end{bmatrix}^T \Big],\\
\Sigma_{\bm{q}^{j,1}_k}
&= \bar{A}^{o_j}\ \Sigma\!\left(\begin{bmatrix}
\quad \cos\delta\bm{\theta}^j_k & \sin\delta\bm{\theta}^j_k 
\end{bmatrix}^T\right)\ \bar{A}^{o_j,T},\\
\Sigma_{\bm{q}^{j,2}_k}
&= \bar{A}^{o_j}\ \Sigma\!\left(\begin{bmatrix}
-\sin\delta\bm{\theta}^j_k & \cos\delta\bm{\theta}^j_k 
\end{bmatrix}^T\right)\ \bar{A}^{o_j,T},
\end{aligned}
\end{equation*}
where similarly to Appendix \ref{app: expectation and covariance in round condition}, recall that $\delta\bm\theta^j_k=\delta\bar\theta^j_k+\delta\bm w^j_{\theta,k}$, we have
\begin{equation*}
\label{eq: E cov cos delta theta sin delta theta}
\begin{aligned}
&\mE[\cos\delta\bm{\theta}^j_k] \approx \cos\delta\bar{\theta}^j_k\mE[\cos\delta\bm{w}^j_{\theta, k}],\\
&\mE[\sin\delta\bm{\theta}^j_k] \approx \sin\delta\bar{\theta}^j_k\mE[\cos\delta\bm{w}^j_{\theta, k}],\\
&\sigma^2(\cos\delta\bm{\theta}^j_k) \!=\! [\cos\delta\bar{\theta}^j_k, \!-\!\sin\delta\bar{\theta}^j_k
]\Sigma_{\delta\bm{w}^j_{\theta, k}}[\cos\delta\bar{\theta}^j_k, \!-\!\sin\delta\bar{\theta}^j_k
]^T, \\
&\sigma^2(\sin\delta\bm{\theta}^j_k) \!=\! [\sin\delta\bar{\theta}^j_k, \;\;\cos\delta\bar{\theta}^j_k
]\Sigma_{\delta\bm{w}^j_{\theta, k}}[\sin\delta\bar{\theta}^j_k, \;\;\cos\delta\bar{\theta}^j_k
]^T, \\
&\mathrm{Cov}(\cos\delta\bm{\theta}^j_k, \sin\delta\bm{\theta}^j_k) \\
&\qquad \qquad \;\;\; \!=\! [\cos\delta\bar{\theta}^j_k, \!-\!\sin\delta\bar{\theta}^j_k
]\Sigma_{\delta\bm{w}^j_{\theta, k}}[\sin\delta\bar{\theta}^j_k, \;\;\cos\delta\bar{\theta}^j_k
]^T, 
\end{aligned}
\end{equation*}
where $\Sigma_{\delta\bm{w}^j_{\theta, k}} \in \mR^{2\times2}$ is the covariance matrix of the vector $\begin{bmatrix} \cos\delta\bm{w}^j_{\theta, k} & \sin\delta\bm{w}^j_{\theta, k}
\end{bmatrix}^T$. 

Next, we compute the expectation of $\bm{r}^j_k$ and $\bm{b}^v_k$. Recall that $\bm{b}^v_k=\bar{A}^v R^{-1}(\bm\theta^v_k)\bm t^v_k+\bar{b}^v$, we obtain
\begin{equation*}
\label{eq: detailed b_v^k}
\begin{aligned}
R^{-1}(\bm\theta^v_k)\bm t^v_k
\!=\!
\begin{bmatrix}
\;\;\bm x^v_k\cos\bm\theta^v_k\!+\!\bm y^v_k\sin\bm\theta^v_k\\
-\bm x^v_k\sin\bm\theta^v_k\!+\!\bm y^v_k\cos\bm\theta^v_k
\end{bmatrix}
\!=:\! \begin{bmatrix}\bm{\zeta}_{v,k,1}\\ \bm{\zeta}_{v,k,2}\end{bmatrix}.
\end{aligned}
\end{equation*}
Hence
\begin{equation*}
\label{eq: E b_v}
\begin{aligned}
\mE[\bm{b}^v_k] & = \mE\big[\bar{A}^vR^{-1}(\bm{\theta}^v_k) \bm{t}^v_k + \bar{b}^v\big]\\
& = \bar{A}^v\mE\big[R^{-1}(\bm{\theta}^v_k) \bm{t}^v_k \big] + \bar{b}^v\\
& = \bar{A}^v\mE\Big[\begin{bmatrix}\bm{\zeta}_{v,k,1}& \bm{\zeta}_{v,k,2}\end{bmatrix}^T\Big]+\bar{b}^v.
\end{aligned}
\end{equation*}
Then, the computation of $\mE[\bm{\zeta}_{v,k,1}]$, $\mE[\bm{\zeta}_{v,k,2}]$ is similar to~\eqref{eq: E eta_j,k,1}-\eqref{eq: E eta_j,k,2}, here we directly give the final expression as follows
\begin{equation*}
\label{eq: E, zeta_j,k}
\begin{aligned}
&\mE[\bm{\zeta}_{v,k,1}]
\!=\! \;\; \bar{x}^v_k \cos\bar\theta^v_k \mE[\cos\bm{w}_{\theta,k}^v] \!+\! \bar{y}^v_k \sin\bar\theta^v_k \mE[\cos\bm{w}_{\theta,k}^v].\\
&\mE[\bm{\zeta}_{v,k,2}]
\!=\! -\bar{x}^v_k \sin\bar\theta^v_k \mE[\cos\bm{w}_{\theta,k}^v] \!+\! \bar{y}^v_k \cos\bar\theta^v_k \mE[\cos\bm{w}_{\theta,k}^v].\\
\end{aligned}
\end{equation*}
To compute $\mE[\bm{r}^j_k]$, we first denote
\begin{equation*}
    \label{eq: kappa}
    \begin{aligned}
    & \bm{\kappa}_{j,k,1} \!:=\! \; \delta\bm{x}^j_k \cos\bm{\theta}^{o_j}_k \!+\! \delta\bm{y}^j_k\sin\bm{\theta}^{o_j}_k \!+\! \bar{b}^v_1\cos\delta\bm{\theta}^j_k \!-\! \bar{b}^v_2\sin\delta\bm{\theta}^j_k,\\ 
    &\bm{\kappa}_{j,k,2} \!:=\! \!-\!\delta \bm{x}^j_k \sin\bm{\theta}^{o_j}_k \!+\! \delta \bm{y}^j_k\cos\bm{\theta}^{o_j}_k \!+\! \bar{b}^v_1\sin\delta\bm{\theta}^j_k \!+\! \bar{b}^v_2\cos\delta\bm{\theta}^j_k,
    \end{aligned}
\end{equation*}
where  $\bar{b}^v_1$, $\bar{b}^v_2$ represent the $1^{\rm st}, 2^{\rm nd}$ element in $\bar{b}^v$, respectively. Then, recall that $\bm{r}^j_k = \bar{A}^{o_j} R_g(\delta \bm{\theta}^j_k)\bm{b}^v_k \!-\! \bm{b}^{o_j}_k$, $R_g(\delta \bm{\theta}^j_k) = \begin{bmatrix}
R(\delta \bm{\theta}^j_k) &  \bm{0}_{2 \times 2}
\end{bmatrix}$ and $\bar{A}^v = \begin{bmatrix}
1 & 0 & -1 & 0\\ 0 & 1 & 0 & -1\\
\end{bmatrix}^T = \begin{bmatrix}
\mathbf{I}_{2 \times 2} \\
-\mathbf{I}_{2 \times 2} \\
\end{bmatrix}$, we obtain
\begin{equation*}
\label{eq: r^j_k detailed}
\begin{aligned}
\bm{r}^j_k
=& \bar{A}^{o_j}
\begin{bmatrix} R(\delta \bm{\theta}^j_k) & \bm{0}_{2\times2} \end{bmatrix}
\Big(
\begin{bmatrix}\mathbf{I}_{2\times2}\\ -\mathbf{I}_{2\times2}\end{bmatrix}
R^{-1}(\bm{\theta}^v_k)\bm{t}^v_k + \bar{b}^v
\Big)\\
&-\Big(\bar{A}^{o_j} R^{-1}(\bm{\theta}^{o_j}_k)\bm{t}^{o_j}_k+\bar{b}^{o_j}\Big)\\
=& \bar{A}^{o_j}
R(\delta \bm{\theta}^j_k) R^{-1}(\bm{\theta}^v_k)\bm{t}^v_k - \bar{A}^{o_j} R^{-1}(\bm{\theta}^{o_j}_k)\bm{t}^{o_j}_k\\
&+\bar{A}^{o_j} R_g(\delta \bm{\theta}^j_k)\bar{b}^v
-\bar{b}^{o_j}\\
=& \bar{A}^{o_j} R^{-1}(\bm{\theta}^{o_j}_k)\big(\bm{t}^v_k-\bm{t}^{o_j}_k\big)
+\bar{A}^{o_j} R_g(\delta \bm{\theta}^j_k)\bar{b}^v
-\bar{b}^{o_j}\\
=& \bar{A}^{o_j} \begin{bmatrix}
\bm{\kappa}_{j,k,1}&
\bm{\kappa}_{j,k,2}
\end{bmatrix}^T-\bar{b}^{o_j}.
\end{aligned}
\end{equation*}
Hence
\begin{equation*}
\label{E r_j}
\begin{aligned}
\mE[\bm{r}^j_k] &= \mE\Big[\bar{A}^{o_j} \begin{bmatrix}
\bm{\kappa}_{j,k,1}&
\bm{\kappa}_{j,k,2}
\end{bmatrix}^T-\bar{b}^{o_j}\Big]\\
&= \bar{A}^{o_j}\mE\Big[\begin{bmatrix}
\bm{\kappa}_{j,k,1}&
\bm{\kappa}_{j,k,2}
\end{bmatrix}^T\Big] -\bar{b}^{o_j},
\end{aligned}
\end{equation*}
and
\begin{equation*}
\label{eq: E kappa}
\begin{aligned}
\mE[\bm{\kappa}_{j,k,1}] &=  (\;\delta\bar{x}^j_k\cos\bar{\theta}^{o_j}_k \!+\! \delta\bar{y}^j_k \sin\bar{\theta}^{o_j}_k)\mE[\cos\bm{w}_{\bm{\theta},k}^{o_j}],\\
& + (\bar{b}^v_1\cos\delta\bar{\theta}^j_k - \bar{b}^v_2\sin\delta\bar{\theta}^j_k)\mE[\cos\delta\bm{w}^j_{\theta, k}],\\
\mE[\bm{\kappa}_{j,k,2}] &= (\!-\!\delta\bar{x}^j_k\sin\bar{\theta}^{o_j}_k \!+\! \delta\bar{y}^j_k \cos\bar{\theta}^{o_j}_k)\mE[\cos\bm{w}_{\bm{\theta},k}^{o_j}],\!\\
& + (\bar{b}^v_1\sin\delta\bar{\theta}^j_k + \bar{b}^v_2\cos\delta\bar{\theta}^j_k)\mE[\cos\delta\bm{w}^j_{\theta, k}].\\
\end{aligned}
\end{equation*}

To compute the covariance matrix $\Sigma_{\bm{\psi}_{j,k,3}}$, recall that $\bm{\psi}_{j,k,3} := \bm{r}^{j,T}_k\lambda^j_k -\bm{b}^{v,T}_k\xi^j_k$, we obtain
\begin{equation*}
\label{eq: Sigma psi_jk3}
    \Sigma_{\bm{\psi}_{j,k,3}} = \begin{bmatrix}-\xi^{j,T}_k & \lambda^{j,T}_k \end{bmatrix}  \Sigma_{[\bm{b}^v_k; \bm{r}^j_k]} \begin{bmatrix}-\xi^j_k \\ \;\; \lambda^j_k \end{bmatrix}.
\end{equation*}
Next, we compute $\Sigma_{[\bm{b}^v_k; \bm{r}^j_k]}$.
More precisely, recall that $\bm{b}^v_k=\bar{A}^v \begin{bmatrix} \bm{\zeta}_{v,k,1}& \bm{\zeta}_{v,k,2} \end{bmatrix}^T+\bar{b}^v$ and $\bm{r}^j_k = \bar{A}^{o_j} \begin{bmatrix}
\bm{\kappa}_{j,k,1}&
\bm{\kappa}_{j,k,2}
\end{bmatrix}^T-\bar{b}^{o_j}$, hence
\begin{equation*}
    \label{eq: cov}
    \begin{aligned}
    \Sigma_{[\bm{b}^v_k; \bm{r}^j_k]} &= \begin{bmatrix} \Sigma_{\bm b^v_k} & \Sigma_{\bm b^v_k,\bm r^j_k}\\ \Sigma_{\bm r^j_k,\bm b^v_k} & \Sigma_{\bm r^j_k} \end{bmatrix} \\
    &= \begin{bmatrix} \bar{A}^v\Sigma_{\bm\zeta_{v,k}}\bar{A}^{v,T} & \bar{A}^v\Sigma_{\bm\zeta_{v,k},\bm\kappa_{j,k}}\bar{A}^{o_j,T}\\ \bar{A}^{o_j}\Sigma_{\bm\kappa_{j,k},\bm\zeta_{v,k}}\bar{A}^{v,T} & \bar{A}^{o_j}\Sigma_{\bm\kappa_{j,k}}\bar{A}^{o_j,T}\end{bmatrix},
    \end{aligned}
\end{equation*}
where $\bm\zeta_{v,k}, \bm\kappa_{j,k}$ are short-hand notation for $\begin{bmatrix} \bm{\zeta}_{v,k,1} & \bm{\zeta}_{v,k,2} \end{bmatrix}^T$, $\begin{bmatrix} \bm{\kappa}_{j,k,1} & \bm{\kappa}_{j,k,2} \end{bmatrix}^T$, respectively. Moreover,
\begin{equation*}
    \label{eq: cov zeta kappa detailed}
    \begin{aligned}
    \Sigma_{\bm\zeta_{v,k}} &= \begin{bmatrix} \sigma^2(\bm{\zeta}_{v,k,1}) & \mathrm{Cov}(\bm{\zeta}_{v,k,1},\bm{\zeta}_{v,k,2})\\
    \mathrm{Cov}(\bm{\zeta}_{v,k,2},\bm{\zeta}_{v,k,1}) & \sigma^2(\bm{\zeta}_{v,k,2}) \end{bmatrix}, \\
     \Sigma_{\bm\kappa_{j,k}} &= \begin{bmatrix} \sigma^2(\bm{\kappa}_{j,k,1}) & \mathrm{Cov}(\bm{\kappa}_{j,k,1},\bm{\kappa}_{j,k,2}) \\ 
    \mathrm{Cov}(\bm{\kappa}_{j,k,2},\bm{\kappa}_{j,k,1}) & \sigma^2(\bm{\kappa}_{j,k,2}) \end{bmatrix},\\
    \Sigma_{\bm\zeta_{v,k},\bm\kappa_{j,k}} &= \Sigma_{\bm\kappa_{j,k},\bm\zeta_{v,k}}^T \\
    &=\begin{bmatrix} \mathrm{Cov}(\bm{\zeta}_{v,k,1},\bm{\kappa}_{j,k,1}) & \mathrm{Cov}(\bm{\zeta}_{v,k,1},\bm{\kappa}_{j,k,2})\\ \mathrm{Cov}(\bm{\zeta}_{v,k,2},\bm{\kappa}_{j,k,1}) & \mathrm{Cov}(\bm{\zeta}_{v,k,2},\bm{\kappa}_{j,k,2}) \end{bmatrix}.\\
    \end{aligned}
\end{equation*}
Then, similar to~\eqref{eq: w_s,j,k}, consider the $14$-dimensional random vector
\begin{equation*}
\label{eq: w_psi,j,k}
\begin{aligned}
\bm{W}_{\psi,j,k}&:=
\begin{bmatrix}
\cos\bm{w}^{o_j}_{\theta,k} & \sin\bm{w}^{o_j}_{\theta,k}\\
\delta\bm w^j_{x,k}\cos\bm{w}^{o_j}_{\theta,k} & \delta\bm w^j_{x,k}\sin\bm{w}^{o_j}_{\theta,k}\\
\delta\bm w^j_{y,k}\cos\bm{w}^{o_j}_{\theta,k} & \delta\bm w^j_{y,k}\sin\bm{w}^{o_j}_{\theta,k}\\
\cos\delta\bm w^j_{\theta,k} & \sin\delta\bm w^j_{\theta,k}\\
\cos\bm{w}^{v}_{\theta,k}   & \sin\bm{w}^{v}_{\theta,k}\\
\bm w^v_{x,k}\cos\bm{w}^{v}_{\theta,k} & \bm w^v_{x,k}\sin\bm{w}^{v}_{\theta,k}\\
\bm w^v_{y,k}\cos\bm{w}^{v}_{\theta,k} & \bm w^v_{y,k}\sin\bm{w}^{v}_{\theta,k}
\end{bmatrix} \in \mR^{7\times2}, \\
\bm{w}_{\psi,j,k} &:= \mathrm{vec}(\bm{W}_{\psi,j,k}^T) \in \mR^{14\times1},
\end{aligned}
\end{equation*}
whose covariance $\Sigma_{\bm{w}_{\psi,j,k}}:=\mathrm{Cov}(\bm{w}_{\psi,j,k})$ is known a priori since it can be computed off-line by Assumption~\ref{ass: known_sin_cos}.
On the other hand, note that the first $8$ dimensions are related to $\bm{\kappa}_{j,k,1}$ and $\bm{\kappa}_{j,k,2}$, and the last $6$ dimensions are related to $\bm{\zeta}_{v,k,1}$ and $\bm{\zeta}_{v,k,2}$, and we can rewrite them as

\begin{equation*}
\label{eq: rewrite zeta kappa}
\begin{aligned}   
    &\bm{\zeta}_{v,k,1}=V_{\bm{\zeta}_{v,k,1}}^T \bm{w}_{\psi,j,k}, \quad \bm{\zeta}_{v,k,2}=V_{\bm{\zeta}_{v,k,2}}^T\bm{w}_{\psi,j,k},\\
    &\bm{\kappa}_{j,k,1}=V_{\bm{\kappa}_{j,k,1}}^T \bm{w}_{\psi,j,k}, \quad \bm{\kappa}_{j,k,2}=V_{\bm{\kappa}_{j,k,2}}^T\bm{w}_{\psi,j,k}.
\end{aligned}
\end{equation*}
Then, the covariance of the vector $[\bm\zeta_{v,k}, \bm\kappa_{j,k}]$ shall take the form
\begin{equation*}
\label{eq: cov zeta kappa}
\begin{aligned}
\sigma^2(\bm{\zeta}_{v,k,1}) &= V_{\bm{\zeta}_{v,k,1}}^T \Sigma_{\bm{w}_{\psi,j,k}} V_{\bm{\zeta}_{v,k,1}}, \\\sigma^2(\bm{\zeta}_{v,k,2}) &= V_{\bm{\zeta}_{v,k,2}}^T \Sigma_{\bm{w}_{\psi,j,k}} V_{\bm{\zeta}_{v,k,2}},\\
\sigma^2(\bm{\kappa}_{j,k,1}) &= V_{\bm{\kappa}_{j,k,1}}^T \Sigma_{\bm{w}_{\psi,j,k}} V_{\bm{\kappa}_{j,k,1}}, \\
\sigma^2(\bm{\kappa}_{j,k,2}) &= V_{\bm{\kappa}_{j,k,2}}^T \Sigma_{\bm{w}_{\psi,j,k}} V_{\bm{\kappa}_{j,k,2}},\\
\sigma(\bm{\zeta}_{v,k,1}, \bm{\zeta}_{v,k,2}) &= V_{\bm{\zeta}_{v,k,1}}^T \Sigma_{\bm{w}_{\psi,j,k}} V_{\bm{\zeta}_{v,k,2}},\\
\sigma(\bm{\kappa}_{j,k,1}, \bm{\kappa}_{j,k,2}) &= V_{\bm{\kappa}_{j,k,1}}^T \Sigma_{\bm{w}_{\psi,j,k}} V_{\bm{\kappa}_{j,k,2}},\\
\sigma(\bm{\zeta}_{v,k,1}, \bm{\kappa}_{j,k,1}) &= V_{\bm{\zeta}_{v,k,1}}^T \Sigma_{\bm{w}_{\psi,j,k}} V_{\bm{\kappa}_{j,k,1}},\\
\sigma(\bm{\zeta}_{v,k,1}, \bm{\kappa}_{j,k,2}) &= V_{\bm{\zeta}_{v,k,1}}^T \Sigma_{\bm{w}_{\psi,j,k}} V_{\bm{\kappa}_{j,k,2}},\\
\sigma(\bm{\zeta}_{v,k,2}, \bm{\kappa}_{j,k,1}) &= V_{\bm{\zeta}_{v,k,2}}^T \Sigma_{\bm{w}_{\psi,j,k}} V_{\bm{\kappa}_{j,k,1}},\\
\sigma(\bm{\zeta}_{v,k,2}, \bm{\kappa}_{j,k,2}) &= V_{\bm{\zeta}_{v,k,2}}^T \Sigma_{\bm{w}_{\psi,j,k}} V_{\bm{\kappa}_{j,k,2}},\\
\end{aligned}
\end{equation*}
and 
\begin{equation*}
\label{eq: V_kappa}
\begin{aligned}
V_{\bm{\zeta}_{v,k,1}} = [&\bm{0}_{1 \times 8}, \bar{x}^v_k\cos\bar{\theta}^v_k + \bar{y}^v_k\sin\bar{\theta}^v_k, \\
&-\bar{x}^v_k\sin\bar{\theta}^v_k + \bar{y}^v_k\cos\bar{\theta}^v_k,\\
&\cos\bar{\theta}^v_k,  \sin\bar{\theta}^v_k, \!-\! \sin\bar{\theta}^v_k,  \cos\bar{\theta}^v_k, ]^T,\\
V_{\bm{\zeta}_{v,k,2}} = [&\bm{0}_{1 \times 8}, -\bar{x}^v_k\sin\bar{\theta}^v_k + \bar{y}^v_k\cos\bar{\theta}^v_k, \\
&-\bar{x}^v_k\cos\bar{\theta}^v_k - \bar{y}^v_k\sin\bar{\theta}^v_k, \\
&-\sin\bar{\theta}^v_k,  \cos\bar{\theta}^v_k, \!-\! \cos\bar{\theta}^v_k, \!-\! \sin\bar{\theta}^v_k]^T,\\
\end{aligned}
\end{equation*}

\begin{equation*}
\label{V_s_j3 V_t_j3}
\begin{aligned}
V_{\bm{\kappa}_{j,k,1}} =[\;\;&\delta\bar{x}^j_k\cos\bar{\theta}^{o_j}_k + \delta\bar{y}^j_k\sin\bar{\theta}^{o_j}_k, \\
-&\delta\bar{x}^j_k\sin\bar{\theta}^{o_j}_k + \delta\bar{y}^j_k\cos\bar{\theta}^{o_j}_k, \\
&\cos\bar{\theta}^{o_j}_k, \sin\bar{\theta}^{o_j}_k, - \sin\bar{\theta}^{o_j}_k,  \cos\bar{\theta}^{o_j}_k,\\
&\bar{b}^v_1\cos\delta\bar{\theta}^j_k - \bar{b}^v_2\sin\delta\bar{\theta}^j_k, \\
-&\bar{b}^v_1\sin\delta\bar{\theta}^j_k - \bar{b}^v_2\cos\delta\bar{\theta}^j_k, \bm{0}_{1 \times 6}]^T,\\
V_{\bm{\kappa}_{j,k,2}} = [-&\delta\bar{x}^j_k\sin\bar{\theta}^{o_j}_k + \delta\bar{y}^j_k\cos\bar{\theta}^{o_j}_k, \\
-&\delta\bar{x}^j_k\cos\bar{\theta}^{o_j}_k - \delta\bar{y}^j_k\sin\bar{\theta}^{o_j}_k,\\
-&\sin\bar{\theta}^{o_j}_k,  \cos\bar{\theta}^{o_j}_k, - \cos\bar{\theta}^{o_j}_k, - \sin\bar{\theta}^{o_j}_k, \\
&\bar{b}^v_1\sin\delta\bar{\theta}^j_k + \bar{b}^v_2\cos\delta\bar{\theta}^j_k, \\
&\bar{b}^v_1\cos\delta\bar{\theta}^{o_j}_k - \bar{b}^v_2\sin\delta\bar{\theta}^j_k, \bm{0}_{1 \times 6}]^T.\\
\end{aligned}
\end{equation*}
In practice, we compute the covariance matrix by replacing the nominal state value with the initial guess of the nominal state, this will accelerate the computation when solving the optimization problem.
\begingroup
\sloppy\emergencystretch=3em
\bibliographystyle{IEEEtran}
\bibliography{reference}
\endgroup

\end{document}